\documentclass[openany]{now2} 

\usepackage{amsfonts}
\usepackage{amsmath}
\usepackage{amssymb}
\usepackage{type1cm} 
\usepackage{algorithmic}
\usepackage{algorithm}
\usepackage{textcomp}

\usepackage{amsmath,amsfonts}
\graphicspath{{./figs/}}
\usepackage{fixltx2e}
\usepackage{array}
\usepackage{wrapfig} 

\usepackage{upgreek}
\usepackage{hyperref}
\usepackage{setspace}

\usepackage{booktabs}
\usepackage{multirow}
\usepackage{longtable}
\usepackage[font=singlespacing, labelfont=bf]{caption}

\usepackage{amssymb,amsfonts,amsmath,dsfont}
\usepackage{times}
\usepackage{color}
\usepackage{algorithmic}
\usepackage{algorithm}

\title{Graphical Models: An Extension to Random Graphs, Trees, and Other Objects}

\author{
Neil Hallonquist \\
Johns Hopkins University \\ 
neil.hallonquist@yahoo.com \\
}

\begin{document}

\copyrightowner{N.~Hallonquist}
 \volume{1}
 \issue{3}
 \pubyear{2017}
 \copyrightyear{2017}
 \isbn{978-0521833783}
 \doi{1234567890}
 \firstpage{1}
 \lastpage{112}

\frontmatter  

\maketitle

\tableofcontents

\mainmatter

\begin{abstract}

In this work, we consider an extension of graphical models to random graphs, trees, and other objects. To do this, many fundamental concepts for multivariate random variables (e.g., marginal variables, Gibbs distribution, Markov properties) must be extended to other mathematical objects; it turns out that this extension is possible, as we will discuss, if we have a consistent, complete system of projections on a given object. Each projection defines a marginal random variable, allowing one to specify independence assumptions between them. Furthermore, these independencies can be specified in terms of a small subset of these marginal variables (which we call the atomic variables), allowing the compact representation of independencies by a directed graph. Projections also define factors, functions on the projected object space, and hence a projection family defines a set of possible factorizations for a distribution; these can be compactly represented by an undirected graph. 

The invariances used in graphical models are essential for learning distributions, not just on multivariate random variables, but also on other objects. When they are applied to random graphs and random trees, the result is a general class of models that is applicable to a broad range of problems, including those in which the graphs and trees have complicated edge structures. These models need not be conditioned on a fixed number of vertices, as is often the case in the literature for random graphs, and can be used for problems in which attributes are associated with vertices and edges. For graphs, applications include the modeling of molecules, neural networks, and relational real-world scenes; for trees, applications include the modeling of infectious diseases and their spread, cell fusion, the structure of language, and the structure of objects in visual scenes. Many classic models can be seen to be particular instances of this framework.

\end{abstract}

\chapter{Introduction}
\label{c-intro} 



In problems involving the statistical modeling of a collection of random variables (i.e., a multivariate random variable), the use of invariance assumptions is often critical for practical learning and inference. A graphical model is a framework for such problems based on conditional independence, a fundamental invariance for these variables; this framework has found wide-spread use because independence occurs naturally in many problems, and is often specifiable by practitioners. Furthermore, independence assumptions can be made at varying degrees (for many invariances, this is not the case), thus creating a range of model complexities, and allowing practitioners to adjust models to a given problem. 


In this work, we consider an extension of graphical models from multivariate random variables to other random objects such as random graphs and trees. 
To do this, core concepts from graphical models must be abstracted, forming a more general formulation; in this formulation, graphical models can be applied to any object that has, loosely speaking, a structure allowing a hierarchical family of projections on it. Each projection in this family defines a marginal random variable, allowing one to specify independence assumptions between them, and further, allowing a graph to represent these independencies (where vertices correspond to atomic variables). This projection family also defines, for distributions, a family of factors, allowing one to specify general factorizations, and further, also represent them compactly by a graph. A projection family must satisfy certain basic properties in order for the corresponding variables to be consistent with each other. 

In the first part of this work, we examine models for random graphs, the problem that originally motivated this investigation. Applying graphical models to them results in a general framework, applicable to problems in which graphs have complicated edge structures. These models need not be conditioned on a fixed number of vertices, as is often the case in the literature, and can be used for problems in which graphs have attributes associated with their vertices and edges. The focus of this work is on problems in which the number of vertices can vary. Some examples of graphs that these models are applicable to are shown in Figures \ref{fig:chemo1} and \ref{fig:mouse1}. This work makes no contribution to the traditional setting of random graphs in which the vertex set is fixed; the formulation presented here is unnecessary in that setting.

After investigating graphical models for graphs, we consider their application to trees, a special type of graph used in many real-world problems. As with graphs, this results in models applicable to a broad range of problems, including those in which trees have complex structures and attributes. In the approach taken in most of the literature, probabilities are placed on trees based on how a tree is incrementally constructed (e.g., from a branching process or grammar). Using graphical models, this approach may be extended, allowing distributions to be defined based on how trees are deconstructed into parts. The benefit of this graphical model approach is that one can make well-defined distributions that have complex dependencies; in contrast, it is often intractable to define distributions over, for example, context-sensitive grammars.

In the last part of this work, we define some consistency and completeness conditions for projection families. These conditions on projections ensure the consistency of their corresponding random variables (i.e., they form a family of marginal variables), which in turn, allows graphical models to be directly defined in terms of projection families. In this formulation, graphical models may be loosely thought of as a modeling framework based on independence assumptions between the parts of an object, given the object is compositional. 
An object is compositional if: (a) it is composed of parts, which in turn, are themselves composed of parts, etc.; and (b) a part can be a member of multiple larger parts. 
Objects such as vectors, graphs, and trees, are compositional; in more applied settings, objects such as words and sentences, people, and real-world scenes, are compositional as well. Graphical models are naturally suited to the modeling of these objects.

\section{Random Graphs}

A graph is a mathematical object that is able to encode relational information, and can be used to represent many entities in the world such as molecules, neural networks, and real-world scenes. An (undirected) graph is composed of a finite set of objects called vertices, and for each pair of vertices, specifies a binary value. If this binary value is positive, there is said to be an edge between that pair of vertices. In most applications, graphs have attributes associated with their vertices and edges; we will refer to attributed graphs simply as graphs in this work. (We make more formal definitions in Section \ref{sec:graph_modeling_framework}.) A random graph is a random variable that maps into a set of graphs. In this section, we give a brief overview of random graph models in the literature, and discuss some of their shortcomings, motivating our work.

\subsection{Literature}

The most commonly studied random graph model is the Erd\H os-R\'enyi model (\citep{erdds1959random}, \citep{gilbert1959random}). This is a model for conditional distributions in which, for a given set of vertices, a distribution is placed over the possible edges. It makes the invariance assumption that, for any two vertices, the probability of an edge between them is independent of the other edges in the graph, and further, this probability is the same for all edges. This classic model, due to its simplicity, is conducive to mathematical analysis; its asymptotic behavior (i.e, its behavior as the number of vertices becomes large) has been researched extensively (\citep{bollobas1998random}, \citep{janson2011random}).

There are many ways in which the Erd\H os-R\'enyi model can be extended. One such extension is the \textit{stochastic blockmodel} \citep{holland1983stochastic}. This model is for conditional distributions over the edges, given vertices, where each vertex has a label (e.g., a color) associated with it. Similar to the Erd\H os-R\'enyi model, for any two vertices, the probability of an edge between them is independent of the other edges in the graph; unlike the the Erd\H os-R\'enyi model, this probability depends on the labels of those two vertices.

An extension of the stochastic blockmodel is the \textit{mixed membership stochastic blockmodel} \citep{airoldi2009mixed}. In this model, instead of associating each vertex with a fixed label, each vertex is associated with a probability vector over the possible labels. Given a set of vertices (and their label probability vectors), a set of edges can be sampled as follows: for each pair of vertices, first sample their respective labels, then sample from a Bernoulli distribution that depends on these labels. 
Another extension of the stochastic blockmodel is the \textit{latent space model} \citep{hoff2002latent}, where instead of associating vertices with labels from a finite set, they are instead associated with positions in a Euclidean space; given the position of two vertices, the probability of an edge between them only depends on their distance.

A general class of random graph models, of which the above models fall within, is the exponential family (\citep{holland1981exponential}, \citep{robins2011exponential}, \citep{snijders2006new}).
A well-known example is the Frank and Strauss model \citep{frank1986markov}, also a model for conditional distributions, specifying the probability of having some set of edges, given vertices. Since the randomness is only over the edges, a graphical model can be applied in which there is a random variable for each pair of vertices, specifying the presence or absence of an edge. These random variables are conditionally independent, in this model, if they do not share a common vertex. 

\subsection{Other Literature}


In this section, we review models from outside the mainstream random graph community that were designed for graphs that vary in size and have complicated attributes. One of the first such models was developed by Ulf Grenander under the name \textit{pattern theory} (\citep{grenander2007pattern}, \citep{grenander1997geometries}, \citep{grenander2012pattern}). This work was motivated by the desire to formalize the concept of a pattern within a mathematical framework. A large collection of natural and man-made patterns is shown in \citep{grenander1996elements}. Examples range from textures to leaf shapes to human language. In each of these examples, every particular instance of the given pattern can be represented by a graph. These instances have natural variations, and so the mathematical framework for describing these patterns is probabilistic, i.e. a random graph model. 
The model developed was based on applying Markov random fields to graphs. Learning and inference are often difficult in this model, limiting its practical use. 




Later, random graph models were developed within the field of \textit{relational statistical learning}.
In particular, techniques such as Probabilistic Relational Models \citep{getoor2001learning}, Relational Markov Networks \citep{taskar2002discriminative}, and Probabilistic Entity-Relationship Models \citep{heckerman2007probabilistic}, were specifically designed for modeling entities that are representable as graphs. 
These models specify conditional distributions, applying graphical models in which: (1) for each vertex, there is a random variable representing its attributes; and (2) for each pair of vertices, there is a random variable representing their edge attributes. (This is an approach similar to the one taken in the Frank and Strauss model).

\begin{figure} 
  \begin{minipage}[b]{0.5\linewidth}
    \includegraphics[scale=0.24]{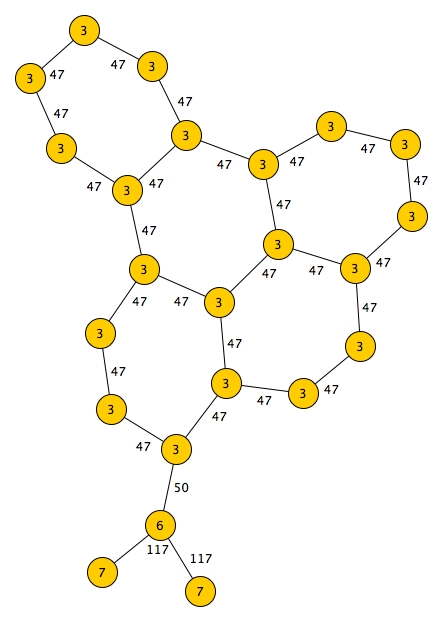} 
  \end{minipage} 
  \begin{minipage}[b]{0.5\linewidth}
    \includegraphics[scale=0.24]{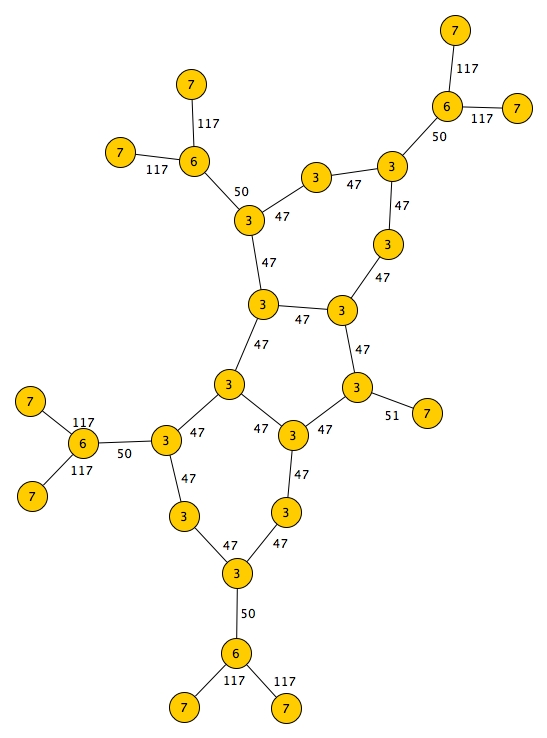} 
  \end{minipage} 
    \begin{minipage}[b]{0.5\linewidth}
    \includegraphics[scale=0.24]{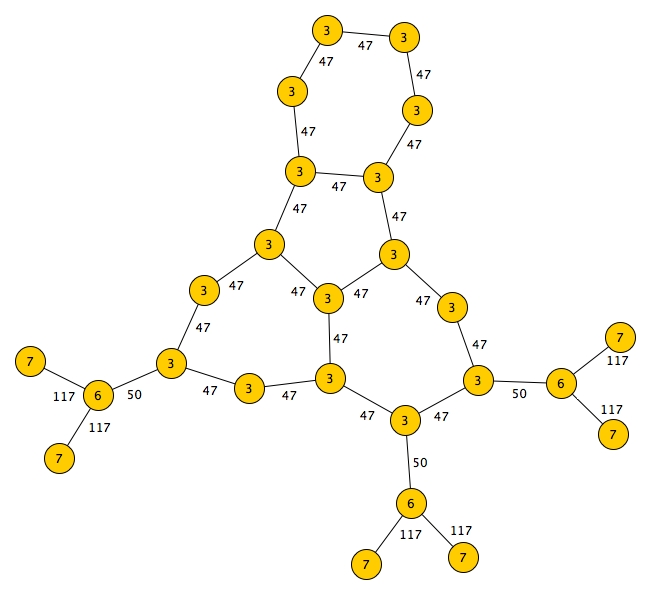} 
  \end{minipage} 
  \begin{minipage}[b]{0.5\linewidth}
    \includegraphics[scale=0.24]{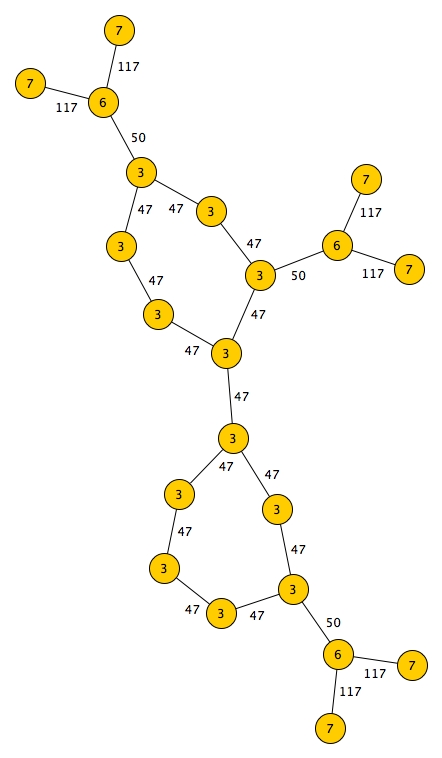} 
  \end{minipage} 
      \begin{minipage}[b]{0.5\linewidth}
    \includegraphics[scale=0.24]{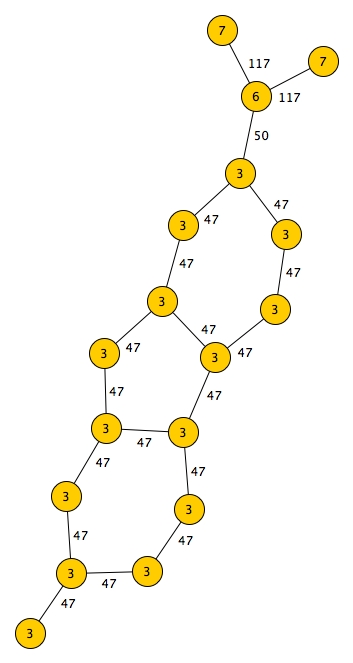} 
  \end{minipage} 
  \begin{minipage}[b]{0.5\linewidth}
    \includegraphics[scale=0.24]{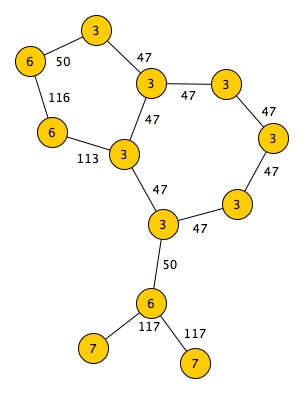} 
  \end{minipage} 
  \caption{Examples of molecule graphs in the MUTAG dataset \citep{shervashidze2011weisfeiler}.}
  \label{fig:chemo1}
\end{figure}

\begin{figure}
  \centering
    \includegraphics[scale=0.24]{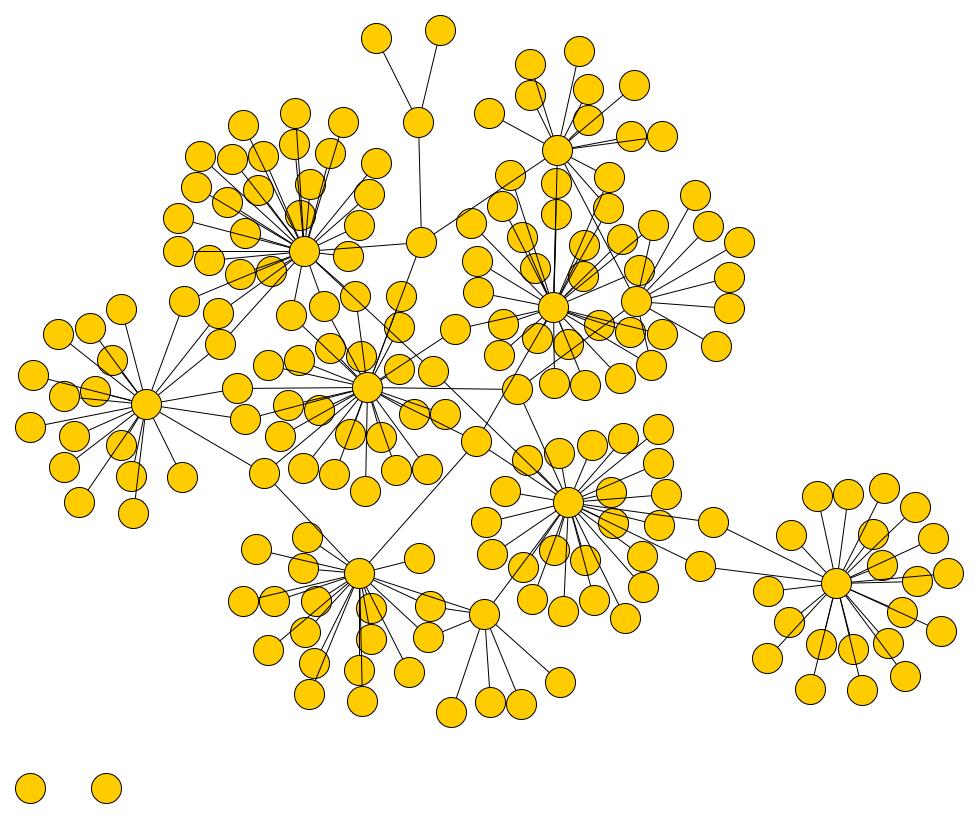}
  \caption{A mouse visual cortex \citep{bock2011network}. }
  \label{fig:mouse1}
\end{figure}

\subsection{Issues}

Suppose we want to learn a distribution over some graph space. This distribution cannot be directly modeled with graphical models because these were designed for multivariate random variables (with a fixed number of components). To avoid this issue, most random graph models in the literature transform the problem into one in which graphical models can be applied. This is done by only modeling a selected set of conditional distributions, for example, the set of distributions in which each is conditioned on some number of vertices. Aside from the fact that many applications simply require full distributions, problems with this approach include: (1) there are complicated consistency issues; a distribution may not exist that could produce a given set of conditional distributions; and (2) this partial modeling, loosely speaking, cannot capture important structures in distributions (e.g., there may be invariances within a full distribution that are difficult to encode within conditional distributions). 
To correct these issues, graphical models (for multivariate random variables) cannot be used for this problem; we need statistical models specifically designed for general graph spaces. 
Suppose we have a graph space $\mathcal{G}$ in which graphs may differ in their order (i.e., graphs in this space may vary in their number of vertices); 
in this work, we want to develop distributions over this type of space.

In addition, we want models that are applicable to problems in which: (a) graphs have complex edge structures; and (b) graphs have attributes associated to their vertices and edges. 
To handle these problems, expressive models are necessary (i.e., models containing a large set of distributions). To make learning feasible in these models, it becomes imperative to specify structure in them as well. 

\subsection{Structure}

To specify structure in random graph distributions, we look to the standard methods used in multivariate random variables for insight. Suppose we have a random variable $\bold{X}$ taking values in $\mathcal{X}=\{0,1\}^n$. In general, its distribution has $2^{n}-1$ parameters that need to be specified. If the value of $n$ is not small, learning this number of parameters, in most real-world problems, is infeasible; hence, the need to control complexity. 
This has led to the wide-spread use of graphical models (\citep{lauritzen1996graphical}, \citep{jordan2004graphical}, \citep{wainwright2008graphical}), a framework that uses factorization to simplify distributions. In this framework, joint distributions are specified by simpler functions, and more specifically, the probability of any $X \in \mathcal{X}$ is uniquely determined by a product of functions over subsets of its components. 

Now, suppose we have a random graph $\bold{G}$ taking values in some finite graph space $\mathcal{G}$. In general, its distribution has $|\mathcal{G}|-1$ parameters that need to be specified, and again, clearly there is a need to control complexity. Similar to the above graphical models, we can simplify distributions through the use of factorization: the probability of any graph $G \in \mathcal{G}$ can be uniquely determined by a product of simpler functions over its subgraphs. Thus, we can create a general framework for random graphs analogous to that of graphical models. 
Indeed, just as graphs can be used to represent the factorization in graphical models, graphs can also be used to represent the factorizations in random graphs. These ideas are explored in Section \ref{sec:graph_modeling_framework}.

\section{Random Trees}

A tree is a special type of graph used in many real-world problems. Like graphs, random tree models range from simplistic ones, amenable to asymptotic analysis, to complex ones, more suited to problem domains with smaller, finite trees. We now briefly review models in the literature.




\subsection{Literature}

A classic random tree model is the \textit{Galton-Watson model} (\citep{watson1875probability}, \citep{le2005random}, \citep{drmota2009random}, \citep{haas2012scaling}), where trees are incrementally constructed: beginning with the root vertex, the number of its children is sampled according to some distribution; for each of these vertices, the number of its children is sampled according to the same distribution, and so on. The literature on these models is vast, most focusing on the probability of extinction and the behavior in the limit. These models are often used, for example, in the study of extinction \citep{haccou2005branching}, and the spread of infectious diseases \citep{britton2010stochastic}.

An extension of the Galton-Watson model is the \textit{multi-type Galton Watson model} (\citep{seneta1970population}, \citep{mode1971multitype}), in which each vertex now has a label from some finite set. As before, trees are incrementally constructed; for a given vertex, the number of its children and their labels is now sampled according to some conditional distribution (conditioned on the label of the parent).

In problems in which vertices have relatively complex labels, 
often a grammar is used to specify which trees are valid (i.e., used to define the tree space). 
These grammars produce trees by production rules, which may be thought of as functions that take a tree and specify a larger one; beginning with the empty tree, trees are incrementally built by the iterative application of these rules. In a \textit{context-free grammar} \citep{chomsky2002syntactic}, the production rules are functions that depend only on one (leaf) vertex in a given tree, and specify its children and their attributes. Distributions can be defined over trees in this grammar by associating a probability with each production rule. 

A \textit{context-sensitive grammar} is an extension of a context-free grammar in which production rules are functions that depend both on a given leaf vertex and certain vertices neighboring this leaf as well. It is well-known that the approach of associating probabilities to production rules does not extend to context-sensitive grammars (i.e., does not produce well-defined distributions in this case); 
to make distributions for these grammars, very high-order models are required. 

There are many applications of random trees with attributes: in linguistics, they are used to describe the structure of sentences and words in natural language \citep{chomsky2002syntactic}; in computer vision, they are used to describe the structure of objects in scenes (\citep{jin2006context}, \citep{zhu2007stochastic}); and in genetics, they are used in the study of structure in RNA molecules (\citep{cai2003stochastic}, \citep{dyrka2009stochastic}, \citep{dyrka2013probabilistic}).

In this work, we consider  
a graphical model approach for random trees; by decomposing a random tree into its marginal random (tree) variables, it becomes tractable to make well-defined tree distributions that are, loosely speaking, context-sensitive. Since trees are graphs, one could model them by applying the same framework that we develop for general graphs. 
However, it is beneficial to instead use models that are tuned to the defining properties of trees.

\section{Outline}

In Section 2, we examine the common compositional structure within multivariate random variables and random graphs, allowing graphical models to be applied to each. The main ideas for extending graphical models to other objects are outlined in this section. In Section 3, we explore the modeling of random trees with graphical models. In Section 4, we provide a formulation for general random objects, and in Section 5, we illustrate the application of these models with some examples, focusing on random graphs. Finally, we conclude with a discussion in Section 6.

\chapter{Random Graphs}
\label{sec:graph_modeling_framework}

In this section, we present a general class of models for random graphs which can be used for creating complex distributions (e.g., distributions that place significant mass on graphs with complicated edge structures). We begin by defining a canonical projection family based on projections that take graphs to their subgraphs. These projections define a consistent family of marginal random (graph) variables, allowing us to specify conditional independence assumptions between them, and in turn, apply Bayesian networks (over the marginal variables that are atomic). Next, we define, using these same graph projections, a Gibbs form for graph distributions, allowing us to specify general factorizations, and in turn, apply Markov random fields. Finally, we consider partially directed models (also known as chain graph models), a generalization of Markov random fields and Bayesian networks; these models are important for random graphs because, as we will discuss, they avoid certain drawbacks that these former models have for this problem, while maintaining their advantages. 



\section{Graphs}
\label{sec:graph_def}

Suppose we have a vertex space $\Lambda_V$ and a edge space $\Lambda_E$, and for simplicity, assume the vertex space is finite. We define a graph to be a couple of the form $G = (V,E)$, where $V$ is a set of vertices and $E$ is a function assigning an edge value to every pair of vertices:
\begin{align*}
& V \subseteq \Lambda_V \\
& E: V \times V \rightarrow \Lambda_E.
\end{align*}
Hence, every vertex is unique, i.e. no two vertices can share the same value in $\Lambda_V$. We assume the edge space $\Lambda_E$ contains a distinguished element that represents the `absence' of an edge (e.g. the value $0$). If a graph has no vertices, i.e., $|V|=0$, we will denote it by $\emptyset$ and refer to it as the empty graph. For simplicity, we assume there are no self loops. That is, there are no edges between a vertex and itself (i.e., $E(v,v) = 0$ for all $v \in V$). 

\begin{example}
Suppose we have a vertex space in which each vertex has a color and a location; let $\Lambda_V = \mathcal{C} \times \mathcal{L}$, where:
	\begin{align*}
	\mathcal{C} & = \{\text{`red'},\text{`blue'} \} \\
	\mathcal{L} & = \{1,\ldots,p\} \times \{1,\ldots,p\}.
	\end{align*}
	Here, $\mathcal{L}$ represents a location space, a two dimensional grid of size $p$. Let the edge space be $\Lambda_E = \{0,1,2\}$, where the value $0$ represents the absence of an edge. See Figure \ref{fig:ilai} for an example graph.
\end{example}

	\begin{figure}[H]
  \centering
    \includegraphics[scale=0.5]{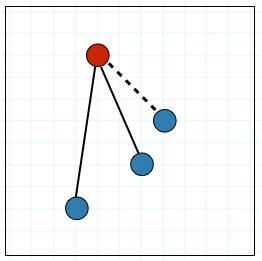}
  \caption{An example graph that uses the vertex and edge space described in Example 2.1. The edge values $1$ and $2$ are represented by lines and dotted lines, respectively. } 
  \label{fig:ilai}
\end{figure}

In most real-world applications, graphs have attributes associated with their vertices and edges; in this case, attributes can be incorporated into the vertex space $\Lambda_V$ and edge space $\Lambda_E$, or alternatively, graphs can be defined to be attributed. For the presentation of the random graph models, we will proceed in the simpler setting though, deferring attributed graphs and other variations to Section \ref{sec:graph_variations}. We consider more examples in Section \ref{sec:examples_sec}.


\section{Marginal Random Graphs}
\label{sec:chap2_discussion}

Suppose we have a graph space $\mathcal{G}$ that we want to define a distribution over. To do this, some basic probabilistic concepts need to be developed; in this section, we define, for a random graph, a family of marginal random (graph) variables. These marginal variables are defined using projections on the graph space, and hence, we require that random graphs take values in graph spaces that are projectable. 

Let's begin by defining an induced subgraph. For a graph $G=(V,E)$, let the subgraph \textit{induced} by a subset $V' \subseteq V$ of its vertices be the graph $G'=(V',E')$, where $E' = E|_{V' \times V'}$ is the restriction of $E$; we let $G' = G(V')$ denote the subgraph of $G$ induced by $V' \subseteq V$. For a given graph, its subgraphs may be thought of as its components or parts, and are fundamental to its statistical modeling. 

A graph space $\mathcal{G}$ is projectable if, for every graph existing in this space, its subgraphs also exist in it. For a graph $G$, let $S(G)$ denote the set of all its subgraphs:
\begin{align*}
S(G) = \{ G(V') \; | \; V' \subseteq V(G) \},
\end{align*}
where $V(G)$ is the set of vertices of graph $G$. This set contains, for example, subgraphs corresponding to individual vertices in $G$ (i.e. the subgraphs $G(V')$ where $V' \subseteq V(G)$ and $|V'|=1$), and the subgraphs corresponding to pairs of vertices in $G$ (i.e. the subgraphs $G(V')$ where $V' \subseteq V(G)$ and $|V'|=2$). 
Now, we may define a projectable graph space: 

\begin{definition}[Projectable Space]
    A graph space $\mathcal{G}$ is \textit{projectable} if:
	\begin{align*}
	G \in \mathcal{G} \implies G' \in \mathcal{G} \text{ for all } G' \in S(G).
	\end{align*}
\end{definition}

Henceforth, we assume that every graph space $\mathcal{G}$ is projectable. Now, we may define graph projections:

\begin{definition}[Canonical Graph Projection]
\label{def:canonical_graph_projections}
Let $V \subseteq \Lambda_V$ be a set of vertices. Define the projection $\pi_{V}: \mathcal{G} \to \mathcal{G}_V$, where $\mathcal{G}_V = \{G \in \mathcal{G} \; | \; V(G) \subseteq V\}$, as follows:
\begin{align*}
\pi_{V}(G) = G(V \cap V'),
\end{align*}
where $V' = V(G)$ is the set of vertices of the graph $G \in \mathcal{G}$.
\end{definition}


The projection $\pi_V$ maps graphs to their induced subgraphs based on the intersection of their vertices with the vertices $V$. 
That is, for a graph $G$, if there are no vertices in this intersection (i.e., $V \cap V' = \emptyset$, where $V' = V(G)$), then $G$ gets projected to the empty graph; if there is an intersection (i.e. $V \cap V' \neq \emptyset$, where $V' = V(G)$), then $G$ gets projected to its subgraph induced by the vertices in this intersection. 

This projection has the property that the image of a projectable graph space is also a projectable space. That is, if the domain $\mathcal{G}$ is projectable, then for each projection $\pi_V$, the codomain $\mathcal{G}_V \subseteq \mathcal{G}$ is also projectable. This property is useful because it allows us to define a consistent set of marginal random variables. Suppose we have a distribution $P$ over a countable graph space $\mathcal{G}$; then the distribution for a marginal random variable $\bold{G}_V$ taking values in $\mathcal{G}_V$ is defined as:
\begin{align*}
P_V^{\text{marg}}(G) = \sum \limits_{\substack{G' \in \mathcal{G} \\ \pi_V(G')=G}} P(G'), \quad G \in \mathcal{G}_V.
\end{align*}
It can be verified that this defines a valid probability distribution, i.e., that
\begin{align*}
\sum \limits_{G \in \mathcal{G}_V} P_V^{\text{marg}}(G) = 1,
\end{align*}
and further, that this set of distributions (i.e., the set $\{P_V^{\text{marg}}, V \subseteq \Lambda_V \}$) is consistent, i.e., for all $V_0, V_1 \subseteq \Lambda_V$ such that $V_0 \subseteq V_1$, we have that $P_{V_0}^{\text{marg}}$ and $P_{V_1}^{\text{marg}}$ are consistent:
\begin{align*}
P_{V_0}^{\text{marg}}(G) = P_{V_1}^{\text{marg}}(\{G' \in \mathcal{G}_{V_1} \; | \; G'(V_0)=G\}),
\end{align*}
for all $G \in \mathcal{G}_{V_0}$.

\section{Independence}
\label{sec:chap2_independence}

The marginal random variables for random graphs defined in the previous section allow us to use the standard definitions of independence and conditional independence for random variables. For convenience, we repeat the definition of independence here, using the notation for random graphs. 
Suppose we have a vertex space $\Lambda_V$ and an edge space $\Lambda_E$, and let $\mathcal{G}$ be a graph space with respect to them. 
Define independence as follows:

\begin{definition}[Independence]
Let $V_1, V_2 \subset \Lambda_V$. For a distribution $P$ over $\mathcal{G}$, we say that the marginal random variables $\bold{G}_{V_1}$ and $\bold{G}_{V_2}$ are \textit{independent} if 
\begin{align*}
P(G_{V_1}, G_{V_2}) & = P(\pi_{V_1}(\bold{G})=G_{V_1}, \pi_{V_2}(\bold{G})=G_{V_2}) \\
			& = P(\pi_{V_1}(\bold{G})=G_{V_1}) \cdot P( \pi_{V_2}(\bold{G})=G_{V_2}) \\
			& = P_{V_1}^{\text{marg}}(\bold{G}_{V_1}=G_{V_1})  \cdot P_{V_2}^{\text{marg}}(\bold{G}_{V_2}=G_{V_2})\\
			& = P_{V_1}^{\text{marg}}(G_{V_1})  \cdot P_{V_2}^{\text{marg}}(G_{V_2}),
\end{align*}
for all $G_{V_1} \in \mathcal{G}_{V_1}$ and $G_{V_2} \in \mathcal{G}_{V_2}$.
\end{definition}


Similarly, conditional independence for random graphs can defined using the standard definitions as well, which we do not repeat here. These definitions suggest methods for specifying structure in distributions:

\begin{example}[`Naive' Random Graphs]
To define a distribution $P$ over a graph space $\mathcal{G}$, a naive approach might be to assume, loosely speaking, that all marginal random variables are independent. Unlike for multivariate random variables though, due to the constraints imposed by the dependence of edges on vertices, conditional independence assumptions are also necessary. Let $\mathcal{V} = \mathcal{V}^{(1)} \cup \mathcal{V}^{(2)}$, where each $\mathcal{V}^{(i)} =  \{ V \subseteq \Lambda_V \; : \; |V| = i\}$, be the set of all singleton vertices as well as all pairs of vertices. We will make invariance assumptions with respect to the marginal random variables $\bold{G}_V, V \in \mathcal{V}$: assume independence between these variables if they do not have vertices in common, and further, assume conditionally independence between edge variables, given the vertex variables. More precisely, suppose:
\begin{enumerate}
\item For every $V_1,V_2 \in \mathcal{V}$ such that $V_1 \cap V_2 = \emptyset$, the random variable $\bold{G}_{V_1}$ is independent of the random variable $\bold{G}_{V_2}$. 
\item For every $V_1,V_2 \in \mathcal{V}$ such that $V_1 \cap V_2 \neq \emptyset$, the random variable $\bold{G}_{V_1}$ is conditionally independent of the random variable $\bold{G}_{V_2}$, given the variable $\bold{G}_{V_1 \cap V_2}$.
\end{enumerate}
Loosely speaking, this latter assumption makes edges incident on a common vertex conditionally independent, given that vertex. With these assumptions, the model has the form:
\begin{align}
 \label{eq:naive_model}
P(G) & = P(G_V, \; V \in \mathcal{V}) \notag \\
	& = P(G_V, \; V \in \mathcal{V}^{(1)}) \cdot P(G_V, \; V \in \mathcal{V}^{(2)} \; | \; G_V, \; V \in \mathcal{V}^{(1)}) \notag \\
	& = \prod \limits_{\substack{V \in \mathcal{V}^{(1)}} } P_{V}^{\text{marg}}( G_{V} ) \prod \limits_{\substack{V \in \mathcal{V}^{(2)}} } P_V^{\text{marg}}(G_V \; | \; G_{\{v\}}, v \in V ) \notag \\
	& = \prod \limits_{\substack{v \in \Lambda_V} } P_{\{v\}}^{\text{marg}}( G_{\{v\}} ) \prod \limits_{\substack{v,v' \in \Lambda_V \\ v \neq v'} } P_{\{v,v'\}}^{\text{marg}}( G_{\{v,v'\}} \; | \; G_{\{v\}}, G_{\{v'\}} ). 
\end{align}
Finally, we mention that this model may be further simplified by assuming the distribution is invariant to isomorphisms (see Section \ref{sec:chap2_isomorphisms}).
\end{example}

\section{Bayesian Networks }
\label{sec:Bayesian_Networks}

In graphical models, graphs are used to represent the structure within distributions; we will refer to these as structure graphs to avoid confusion. For a graph with a binary edge function, two vertices $v,u$ are said to have a \textit{directed} edge from $v$ to $u$ (denoted by $v \rightarrow u$) if $E(v,u)=1$ and $E(u,v)=0$, and are said to have an \textit{undirected} edge between them (denoted $v - u$) if $E(v,u) = E(u,v)=1$. The vertex $v$ is a \textit{parent} of vertex $u$ if $v \rightarrow u$, and vertices $v$ and $u$ are \textit{neighbors} if $v - u$. The set of parents of $v$ is denoted by $\text{pa}(v)$ and the set of neighbors by $\text{ne}(v)$. 
In this section, we consider Bayesian networks, a modeling framework based on conditional independence assumptions, specified in structure graphs with directed edges \citep{pearl1995probabilistic}.




\subsection{Structure Graphs}

Let's begin by considering Bayesian networks for multivariate random variables; suppose we have a random variable $\bold{X}$ taking values in $\mathcal{X} = \{0,1\}^n$, and a structure graph with vertices $\{1,\ldots,n\}$ and a binary edge function of the form $\mathcal{N}: \{1,\ldots,n\}^2 \to \{0,1\}$. Further, assume this structure graph has directed edges and is acyclic; a distribution $P$ over $\mathcal{X}$ is said to factor according to this structure graph if it can be written in the form: 
\begin{align*}
P(X) = \prod_{i = 1}^n P(X_{i} \; | \; X_{\text{pa}(i)} ), 
\end{align*}
where $X_A \equiv \pi_{A}(X)$ is the projection of $X$ onto its components in the set $A \subseteq \{1,\ldots,n\}$.

Now consider Bayesian networks for random graphs; suppose we have a random graph $\bold{G}$ taking values in a graph space $\mathcal{G}$, and a structure graph with vertices $\mathcal{V} = \mathcal{V}^{(1)} \cup \mathcal{V}^{(2)}$, where each $\mathcal{V}^{(i)} =  \{ V \subseteq \Lambda_V \; : \; |V| = i\}$, and a binary edge function of the form $\mathcal{N}: \mathcal{V}^2 \to \{0,1\}$. Further, assume the structure graph is directed and acyclic; a distribution $P$ over $\mathcal{G}$ factorizes according to this structure graph if it can be written in the form:
\begin{align*}
P(G) & = P(G_V, \; V \in \mathcal{V}) \\
	& = \prod_{V \in \mathcal{V}} P(G_{V} \; | \; G_{\text{pa}(V)} ),
\end{align*}
where, for $A \subseteq \mathcal{V}$, we have $G_A \equiv \{G_V, \; V \in A\}$, and where, recall $G_V \equiv \pi_V(G)$ is the projection of $G$ onto the vertices $V$. 

\begin{example}[`Naive' Random Graphs (cont.)]
We revisit example 2.2, now specifying the structure in terms of a structure graph. Define the neighborhood function $\mathcal{N}:\mathcal{V}^2 \to \{0,1\}$ as follows:
\begin{align*}
\mathcal{N}(V,V') = \left \{
\begin{array}{l}
    1, \text{ if }  V \subset V' \\ 
    0, \text{ otherwise }
\end{array}.
\right.
\end{align*}
In other words, this neighborhood function specifies a directed edge from each vertex variable $G_{\{v\}}$ to every edge variable of the form $G_{\{v,v'\}}$. Distributions that cohere with this structure graph can be written in the form of equation (\ref{eq:naive_model}). This model has the minimal complexity in the sense that the neighborhood function cannot have fewer non-zero values while still defining a valid structure graph (i.e., the structure graph specifies independence assumptions that are consistent in the sense that there exists a well-defined distribution that satisfies them). Hence, the reason for referring to this as the \textit{naive} model. 
\end{example}

\subsection{Atomic Variables}
\label{sec:atomics}




In the previous section, the main difference between the graphical model for multivariate random variables and for random graphs was in the marginal variables used in the structure graph in each case (i.e., the variables in which vertices in the structure graph correspond). In this section, we consider in more detail the subset of variables used, for a given random object, by graphical models in their structure graphs.

Suppose we have a random graph $\bold{G}$ taking values in a projectable graph space $\mathcal{G}$. The canonical set of projections on this graph space defines a set of marginal random variables, and a projection in this set such that, loosely speaking, no other projection further projects downward, defines an \textit{atomic} variable. 
Informally, a projection $\pi$ is atomic (with respect to a finite projection family) if: (a) there does not exist a projection in this family that projects to a subset of its image; or (b) if there are projections in this family that project to a subset of its image, then this set loses information (i.e., $\pi$ is not a function of these projections). We defer more formal definitions to Section \ref{sec:proj_families}. 
The second condition ensures that any object projected by the set of atomic projections can be reconstructed. We will call a marginal variable atomic if it corresponds to an atomic projection. 

For random graphs, the atomic projections have the form $\pi_{\{v\}}$ or $\pi_{\{v,v'\}}$ (i.e., loosely speaking, the projections to some vertex or edge), and the non-atomic projections have the form $\pi_V$ where $|V|>2$ (i.e., the projections to larger vertex sets). 
Hence, for a random graph $\bold{G}$, the atomic variables are $\{\pi_V(\bold{G}), \; V \in \mathcal{V}\}$, where $\mathcal{V} = \mathcal{V}^{(1)} \cup \mathcal{V}^{(2)}$; these variables can be used as a representation of the random graph, and graphical models specify structure in terms of them (i.e., the vertices in structure graphs correspond to these variables). 

\section{Gibbs Distribution}



In this section, we define a Gibbs form for random graphs based on a canonical factorization; this factorization is determined by the canonical projections, the projection family taking graphs to their subgraphs. For a graph $G$, let $S_k(G)$ denote the set of all subgraphs of $G$ of order $k$:
\begin{align*}
S_k(G) = \{ G' \in S(G) \; : \; |V(G')| = k \},
\end{align*}
where, recall $S(G)$ is the set of all of induced subgraphs of $G$, and where $V(G)$ denotes the vertices of graph $G$. Hence, the set $S_1(G)$ contains graphs having a single vertex, the set $S_2(G)$ contains graphs having two vertices, and so on. 

For this section, let the vertex space $\Lambda_V$ be countable, and for any graph, assume its vertex set $V \subset \Lambda_V$ is finite. We can define a Gibbs distribution for a countable graph space $\mathcal{G}$ as follows: \\

\begin{definition}[Gibbs Distribution]
A probability mass function (pmf) P over a countable graph space $\mathcal{G}$ is a \textit{Gibbs distribution} if it can be written in the form
\begin{align}
P(G) = \exp \left[ \psi_0 + \sum \limits_{G' \in S_1(G)} \psi_1(G') +  \sum \limits_{G' \in S_2(G)} \psi_2(G') + \ldots \right]
\label{eq:gibbs_eqn}
\end{align}
where $\psi_k: \mathcal{G}^{(k)} \rightarrow \mathbb{R} \cup \{-\infty\}$ is called the potential of order $k$, and $\mathcal{G}^{(k)}$ denotes the space of graphs of order $k$, i.e.:
\begin{align*}
\mathcal{G}^{(k)} = \{ (V,E) \in \mathcal{G} \; : \; |V| = k \}.
\end{align*}
\end{definition}





A graph space need not be countable (depending on $\Lambda_E$), but for ease of exposition, we assumed so here. We give some examples in which classic models are expressed in this form. 

\begin{example}[The Erd\H os-R\'enyi model]
\citep{erdds1959random} \citep{gilbert1959random}:
Let $\mathcal{G}$ be a standard graph space (i.e. the vertex space $\Lambda_V = \mathbb{N}$ is the set of natural numbers and the edge space $\Lambda_E = \{0,1 \}$.) The Erd\H os-R\'enyi model is a conditional distribution specifying the probability of edges $E$ given a finite set of vertices $V$. It makes the invariance assumption that, for any two vertices, the probability of an edge between them is independent of the other edges in the graph:
\begin{align*}
P(E|V) & = \exp \left[  \sum \limits_{G' \in S_2((V,E))} \psi_2(G')  \right] 
\end{align*}
where
\begin{align*}
\psi_2(G') = \begin{cases} 
	      \log (p), & \text{if } G' \text{ has an edge}  \\ 
	      \log (1-p), & \text{otherwise} 
	\end{cases}
\end{align*}
and $p \in [0,1]$. 
\end{example}

\begin{example}[The stochastic blockmodel]
\citep{holland1983stochastic}: Let $\mathcal{G}$ be a graph space where the vertex space is $\Lambda_V = \{1,\ldots,l\} \times \mathbb{N}$, where the first component corresponds to some label, and the edge space is $\Lambda_E = \{0,1 \}$. The stochastic blockmodel is also a conditional distribution specifying the probability of edges $E$ given a finite set of vertices $V$. It makes the invariance assumption that, for any two vertices, the probability of an edge between them depends on only the label of those two vertices:
\begin{align*}
P(E|V) & = \exp \left[  \sum \limits_{G' \in S_2((V,E))} \psi_2(G')  \right] 
\end{align*}
where
\begin{align*}
\psi_2(G') = \begin{cases} 
	      \log \; p(a_1,a_2), & \text{if } G' \text{ has an edge}  \\ 
	      \log (1-p(a_1,a_2)), & \text{otherwise} 
	\end{cases}
\end{align*}
where $a_1,a_2$ are the labels of the two vertices in $G' \in \mathcal{G}^{(2)}$ and $p: \{1,\ldots,l\}^2 \rightarrow [0,1]$ is a symmetric function.  
\end{example}

We now define a positivity condition for distributions; this will allow us to make a statement about the universality of the Gibbs representation. \\ 

\begin{definition}[Positivity Condition]
Let $P$ be a real function over a projectable graph space $\mathcal{G}$. The function $P$ is said to satisfy the \textit{positivity condition} if, for all $G \in \mathcal{G}$, we have:
\begin{align*}
P(G) > 0 \implies P(G')>0 \text{ for all } G' \in S(G). \\
\end{align*}
\end{definition}

\begin{theorem}[]
\label{theorem:universality}
If $P$ is a positive distribution over a projectable graph space $\mathcal{G}$, then $P$ can be written in Gibbs form. 
\end{theorem}

\begin{proof}
For a given graph $G \in \mathcal{G}$, define
\begin{align*}
\phi_G(W) = \log \frac{ P(G_W) }{ P(\emptyset) }
\end{align*}
where $W \subseteq V(G)$ and where $G_W = \pi_W(G)$. Using the Mobius formula, we can write
\begin{align*}
\phi_G(W) & = \sum_{W' \subseteq W} \psi_G(W') \\
 \psi_G(W) & = \sum_{W' \subseteq W} (-1)^{|W| - |W'|} \phi_G(W'),
\end{align*}
where the positivity condition is required for the validity of the second equation. Note that $\psi_G(W)$ only depends on $G_W$ (not on the rest of $G$), so it can be renamed $\psi(G_W)$; letting $W = V(G)$, we have:
\begin{align*}
P(G) = P(\emptyset) \exp \left [ \sum_{ W \subseteq V(G) } \psi(G_W) \right ].
\end{align*}
\end{proof}

This theorem shows that distributions can be expressed in such a way that the probability of a graph is a function of only its induced subgraphs; that is, statistical models need not include (more formally, set to zero) the value of potentials that involve vertices that are absent from a given input graph. Henceforth, we return to assuming vertex spaces are finite (since, in our formulation, graphical models are limited to finite projection families (see Section \ref{sec:General_Random_Objects})).




\section{Markov Random Fields}
\label{sec:compact_dists}

In the previous section, we defined a Gibbs distribution for random graphs, a universal representation (Theorem \ref{theorem:universality}) based on a general factorization. In this section, we consider Markov random fields, a graphical model that specifies structure in distributions based on these factorizations (\citep{kindermann1980markov}, \citep{geman1986markov}, \citep{clifford1990markov}). 

Consider Markov random fields for multivariate random variables: suppose we have a random variable $\bold{X}$ taking values in $\mathcal{X} = \mathcal{X}_1 \times \cdots \times \mathcal{X}_n$, where each $\mathcal{X}_i$ is finite. To define a distribution over $\mathcal{X}$, we will assume it equals some product of simpler functions (i.e., functions that have smaller domains than $\mathcal{X}$). To define these simpler functions, we use projections of the form $\pi_C: \mathcal{X} \to \mathcal{X}_C$, where $C \subseteq \{1,\ldots,n\}$ and $\mathcal{X}_C = \prod_{i \in C} \mathcal{X}_i$, and take elements in $\mathcal{X}$ to their components. Using these projections, we can define factors of the form $f_C: \mathcal{X}_C \to \mathbb{R}^+$, and 
a distribution $P$ factorizes over $\mathcal{C} \subseteq \mathbb{P}(\{1,\ldots,n\})$ if it can be written as:
\begin{align*}
P(X) & = \frac{1}{Z} \prod \limits_{C \in \mathcal{C} } f_C( \pi_C(X) ) \\
	& = \frac{1}{Z} \prod \limits_{C \in \mathcal{C} } f_C( X_C ),
\end{align*}
where $X_C \equiv \pi_C(X)$, and where $\mathbb{P}(\{1,\ldots,n\})$ denote the power set of $\{1,\ldots,n\}$. Structure can be specified in this model by the choice of factors. For a given model, complexity can be reduced through the removal of factors (i.e., removing elements from the set $\mathcal{C}$).


Now suppose we have a random graph $\bold{G}$ taking values in $\mathcal{G}$. As was done in the multivariate case, we define the factorization of distributions over this graph space using a projection family;  
a distribution can be defined as a product of factors of the form $f_V: \mathcal{G}_V \to \mathbb{R}^+$, where, recall $\mathcal{G}_V \equiv \pi_V(\mathcal{G})$ is a smaller graph space. A distribution $P$ factorizes over $\mathcal{V} \subseteq \mathbb{P}(\Lambda_V)$ if it can be written as:
\begin{align*}
P(G) & = \frac{1}{Z} \prod \limits_{V \in \mathcal{V}} f_V( \pi_V(G) ) \\
        & = \frac{1}{Z} \prod \limits_{\substack{V \in \mathcal{V} \\ V \subseteq V(G)} } f_V( G_V ),
\end{align*}
where $G_V \equiv \pi_V(G)$, and where we are assuming $f_{V}(G)=1$ if $V(G) \not \subseteq V$. As above, structure can be specified in this model through the choice of factors.  

\subsection{Cliques}


We now consider the representation of the factorizations in the previous section in terms of an undirected structure graph;
suppose we have a neighborhood function $\mathcal{N}: \mathcal{V}^2 \to \{0,1\}$ that is symmetric, where $\mathcal{V} = \mathcal{V}^{(1)} \cup \mathcal{V}^{(2)}$. In order for a neighborhood function to be valid (i.e., specify independence assumptions that are consistent in the sense that there exists a well-defined distribution that satisfies them), it must specify a direct dependency between any $V,V' \in \mathcal{V}$ such that one is a subset of the other. That is, for all $V,V' \in \mathcal{V}$, we require that
\begin{align*}
\mathcal{N}(V,V')=1 \text{ if } V \subset V' \text{ or } V' \subset V.
\end{align*}
A neighborhood function specifies the set of factors within a model based on its cliques, where cliques are defined as follows: 

\begin{definition}[Cliques]
For a neighborhood function $\mathcal{N}$, a collection of vertex sets $\tilde{\mathcal{V}} \subseteq \mathcal{V}$ is a \textit{clique} if:
\begin{enumerate}
\item 
$\mathcal{N}(V_0,V_1)=1 \text{ for all } V_0,V_1 \in \tilde{\mathcal{V}}$; or
\item 
$|\tilde{\mathcal{V}}|=1.$
\end{enumerate}
\end{definition}
Hence, by the second condition, we have that each vertex $\{v\}$ and each pair of vertices $\{v,v'\}$ are cliques. Let $\mathbb{V}_{\mathcal{N}}$ contain the vertex sets that correspond to cliques:
\begin{align*}
\mathbb{V}_{\mathcal{N}} = \left \{  \bigcup \limits_{ V \in \tilde{\mathcal{V}}} V \; \;  : \; \; \tilde{\mathcal{V}} \subseteq \mathcal{V}, \; \tilde{\mathcal{V}} \text{ is a clique} \right \}.
\end{align*}
This set represents the set of factors to be used in a distribution (i.e., for each $V \in \mathbb{V}_{\mathcal{N}}$, we will assume there is a factor over this set of vertices). Hence, a Gibbs distribution with respect to a neighborhood function can be defined as follows:

\begin{definition}[Gibbs Distribution]
Let $P$ be a pmf over $\mathcal{G}$. The distribution $P$ is a \textit{Gibbs distribution} with respect to the neighborhood function $\mathcal{N}$ if it can be written in the form:
\begin{align}
P(G) & = \frac{1}{Z} \prod \limits_{ \substack{V \in \mathbb{V}_{\mathcal{N}} \\ V \subseteq V(G) }} \phi_V(G_V), 
\label{gibbs_dist_c}
\end{align}
where $\phi_V: \mathcal{G}_V \rightarrow [0, \infty)$.
\end{definition}


Now that we have defined a Gibbs distribution with respect to a neighborhood function, let's consider its connections to Markov properties and Markov distributions.


\subsection{Markovity}

A distribution is Markov if, loosely speaking, conditional probabilities only depend on local parts of the random object. Let's consider Markovity for multivariate random variables. Suppose we have a random variable $\bold{X}$ taking values in $\mathcal{X} = \{0,1\}^n$, and a (symmetric) neighborhood function $\mathcal{N}: \{1,\ldots,n\}^2 \to \{0,1\}$. A distribution $P$ over $\mathcal{X}$ is Markov with respect to the neighborhood function $\mathcal{N}$ if, for all $X \in \mathcal{X}$ and for all $i$, we have that:
\begin{align}
P(X_i \; | \; X_j, \;  j \neq i ) = P(X_i \; | \; X_j , \; j \in J_i ) 
\end{align}
where $J_i = \{ j  \; | \; \mathcal{N}(i,j)=1 \}$, and where each $X_i$ denotes the $i^{\text{th}}$ component of $X$. 

Now consider random graphs; let $\Lambda_V$ and $\Lambda_E$ be a vertex and edge space, respectively, and let $\mathcal{G}$ be a graph space with respect to them. Further, let $\mathcal{V} = \mathcal{V}^{(1)} \cup \mathcal{V}^{(2)}$, and suppose we have a (symmetric) neighborhood function $\mathcal{N}: \mathcal{V}^2 \to \{0,1\}$. Then, a distribution $P$ is Markov with respect to the neighborhood function $\mathcal{N}$ if, for all $G \in \mathcal{G}$ and all $V \in \mathcal{V}$, we have that:
\begin{align}
\label{eqn:markovity_graphs}
P(G_V \: | \; G_{V'}, \; V' \in \mathcal{V} \setminus u(V)) & = P(\bold{G}_V = G_V \; | \; \bold{G}_{V'} = G_{V'}, \;  V' \in \mathcal{V} \setminus u(V) ) \nonumber \\
	& = P(\bold{G}_V = G_V \; | \; \bold{G}_{V'} = G_{V'}, \;  V' \in J_{V} )\nonumber \\
	& = P(G_V \: | \; G_{V'}, \; V' \in J_{V}),
\end{align}
where $J_{V} = \{ V' \in \mathcal{V} \setminus u(V) \; | \; \mathcal{N}(V,V')=1 \}$, and where $u(V) = \{V' \in \mathcal{V} \; | \; V \subseteq V' \}$. Thus, we define Markovity as follows: 

\begin{definition}[Markov Distribution]
Let $P$ be a pmf over $\mathcal{G}$. The distribution $P$ is a \textit{Markov distribution} with respect to neighborhood function $\mathcal{N}$ if, for all $G \in \mathcal{G}$ and $V \in \mathcal{V}$, equation (\ref{eqn:markovity_graphs}) holds.
\end{definition}

We have that if a distribution is Gibbs with respect to some neighborhood function, then it is Markov with respect to it as well:
 
 \begin{proposition}
Let $P$ be a distribution over $\mathcal{G}$ and let $\mathcal{N}$ be a neighborhood function. Then:
\begin{align*}
P \text{ is Gibbs w.r.t. } \mathcal{N} \implies P \text{ is Markov w.r.t. } \mathcal{N}. 
\end{align*}
\end{proposition}

The reverse implication in the above proposition is not true (i.e., the Hammersley-Clifford theorem (\citep{grimmett1973theorem}, \citep{besag1974spatial}) does not hold). A neighborhood function can specify more structure for a Markov distribution than for a Gibbs distribution; hence, one cannot specify (general) independence assumptions and then assume a Gibbs form. The reason is because the atomic variables have redundancy in them; a vertex variable $\bold{G}_{\{v\}}$ is a function of an edge variable of the form $\bold{G}_{\{v,v'\}}$. 
For a discussion on this issue, see Section \ref{sec:discussion_random_graphs}. To avoid this drawback, but maintain the advantages offered by undirected models (in particular, the ability to express the probability of a graph in terms of only its subgraphs), we now consider partially directed models.


\section{Partially Directed Models}
\label{sec:directed_models}


In this section, we briefly review chain graph models \citep{lauritzen2002chain}, which we will use in the modeling of random graphs. These models involve structure graphs that can have both directed and undirected edges, a generalization of Bayesian models and Markov random fields. The reason chain graph models are beneficial for random graphs is because they allow one to specify, loosely speaking, a Gibbs distribution over vertices, as well as a Gibbs distribution over edges, while avoiding the functional dependencies that are problematic. For these structure graphs, we will assume that all edges between vertex variables and edge variables are directed, and all other edges undirected. 


 

In these models, structure graphs are required to be acyclic, where cycles are now defined as follows: a \textit{partially directed cycle} is a sequence of $n \geq 3$ distinct vertices $v_1,\ldots,v_n$  in a graph, and a vertex $v_{n+1} = v_1$, such that:
\begin{enumerate}
\item for all $1 \leq i \leq n$, either $v_i - v_{i+1}$ or $v_i \leftarrow v_{i+1}$, and
\item there exists a $1 \leq j \leq n$ such that $v_j \leftarrow v_{j+1}$.
\end{enumerate}
A \textit{chain graph} is a graph in which there are no partially directed cycles. For a given chain graph, let the \textit{chain components} $\mathcal{K}$ be the partition of its vertices such that any two vertices $v$ and $u$ are in the same partition set if there exists a path between them that contains only undirected edges. In other words, $\mathcal{K}$ is the partition that corresponds to the connected components of the graph after the directed edges have been removed. 

A distribution $P$ over graph space $\mathcal{G}$ factorizes according to a chain graph $H$ if it can be written in the form: 
\begin{align*}
P(G) = \prod_{K \in \mathcal{K}}P(G_{K} \; | \; G_{\text{pa}(K)}),
\end{align*}
and further, we have that:
\begin{align*}
P(G_{K} \; | \; G_{\text{pa}(K)}) = \frac{1}{Z(G_{\text{pa}(K)})} \prod_{C \in \mathcal{C}(K)} \phi_C(G_C),
\end{align*}
where $\mathcal{C}(K)$ is the set of cliques in the moralization of the graph $H_{K \cup \text{pa}(K)}$, i.e., the undirected graph that results from adding edges between any unconnected vertices in $\text{pa}(K)$ and converting all directed edges into undirected edges, where
\begin{align*}
\text{pa}(K) = \bigcup_{v \in K} \text{pa}(v) \setminus K.
\end{align*}
The factor $Z$ normalizes the distribution:
\begin{align*}
Z(G_{\text{pa}(K)}) = \sum_{G \in \mathcal{G}_{K}} \prod_{C \in \mathcal{C}(K)} \phi_C(G_C).
\end{align*}


\begin{example}
\label{example:chain}
Suppose we have a vertex space $\Lambda_V = \{1,2,3\}$, an edge space $\Lambda_E = \{0,1\}$, and a graph space with respect to them. Then, the atomic variables for this graph space correspond to the set $\mathcal{V} = \mathcal{V}^{(1)} \cup \mathcal{V}^{(2)}$, where, in this case:
\begin{align*}
\mathcal{V}^{(1)} & = \{ \{1\}, \{2\}, \{3\} \} \\
\mathcal{V}^{(2)} & = \{ \{1,2\}, \{1,3\}, \{2,3\} \},  
\end{align*}
and so, the vertices in structure graphs correspond to them. Suppose the structure graph is a chain graph and is as shown in Figure \ref{fig:chain_fig}. Then the chain components $\mathcal{K}$ is the partition $\mathcal{K} = \{\mathcal{V}^{(1)}, \mathcal{V}^{(2)}\}$. The distribution takes the following form:
\begin{align*}
P(G) & = \prod_{K \in \mathcal{K}}P(G_{K} \; | \; G_{\text{pa}(K)}) \\
	& = P(G_{\mathcal{V}^{(1)}} ) P(G_{\mathcal{V}^{(2)}} \; | \; G_{ \mathcal{V}^{(1)} })
\end{align*}
where, recall $G_{\mathcal{V}^{(i)}} = \{G_{V}, \; V \in \mathcal{V}^{(i)}\}$. Further, we have that each component can be expressed in Gibbs form:
\begin{align*}
& P(G_{\mathcal{V}^{(1)}})  \propto  \prod \limits_{ \substack{ V \subseteq V(G) }} \sigma(V) \\
& P(G_{\mathcal{V}^{(2)}} \; | \; G_{ \mathcal{V}^{(1)} }) \propto  \prod \limits_{ \substack{ V \subseteq V(G) }} \phi_V(G_V), 
\end{align*}
where $\sigma: \mathbb{P}(V) \to [0, \infty)$ and $\phi_V: \mathcal{G}_V \rightarrow [0, \infty)$.
 \begin{figure}
  \centering
    \includegraphics[scale=0.6]{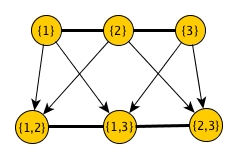}
  \caption{An example chain graph for example \ref{example:chain}. The thick edges represent that the vertices connected by them are fully connected.}
  \label{fig:chain_fig}
\end{figure}
\end{example}






\section{Discussion}
\label{sec:discussion_random_graphs}

We now take a step back and examine some of the design choices made in this section. Graphical models, from a high-level, may be thought of as a framework for modeling random objects based on the use of independence assumptions between the parts of the object. It is important that these independence assumptions be made, or can be made, between the smallest parts, those that cannot be decomposed into smaller ones. The reason, as we will discuss in this section, is that this makes the space of (possible) independence assumptions as large as possible, and hence allows the most structure to be specified within a graphical model. 

\subsection{Redundant Representations}
\label{sec:redundancy_reps}

The representation of a random object based on its atomic marginal variables can have redundancy in it; for example, a vertex variable $\bold{G}_{\{v\}}$ is a function of an edge variable of the form $\bold{G}_{\{v,v'\}}$. This redundancy may appear troublesome since, for example, it means the Hammersley-Clifford theorem cannot be used, preventing us from specifying independencies and then assuming a Gibbs form for distributions. We could remove the redundant variables (i.e., variables that are functions of other variables), and represent the random graph $\bold{G}$ by only the random variables $\{\bold{G}_V, \; V \in \mathcal{V}^{(2)} \}$, a subset of the atomic variables. However, this approach is problematic since it also diminishes our ability to specify structure. 
Representations with redundancy have the advantage, compared to representations without redundancy, of providing a larger space of possible independence assumptions. 
We illustrate the concept with some examples: 

\begin{example}[Context-Sensitive Independence]
Suppose we have a multivariate random variable $\bold{X}$ taking values in $\mathcal{X} = \mathcal{X}_1 \times \ldots \times \mathcal{X}_n$, and suppose we specify (within a Bayesian network) that the distribution over $\mathcal{X}$ has the following conditional independence:
\begin{align*}
P(\bold{X}_1 \; | \; \bold{X}_i, i \neq 1) = P(\bold{X}_1 \; | \; \bold{X}_2, \bold{X}_3).
\end{align*}
Now suppose that we want to specify the additional invariance that
\begin{align*}
P(\bold{X}_1 \; | \; \bold{X}_2, \bold{X}_3=X) = P(\bold{X}_1 \; | \; \bold{X}_3=X),
\end{align*}
where $X \in \mathcal{X}_3$; this type of invariance is sometimes referred to as context-sensitive independence (\citep{shimony1991explanation}, \citep{boutilier1996context}, \citep{chickering1997bayesian}). A simple way to incorporate it within a Bayesian network is by the addition of a redundant random variable of the form $\bold{Y}=f(\bold{X}_2,\bold{X}_3)$, taking values in a partition of the space of values of the input variables. In particular, define the function $f: \mathcal{X}_2 \times \mathcal{X}_3 \to \mathbb{P}(\mathcal{X}_2 \times \mathcal{X}_3)$ to be:
\begin{align*}
f(X_2,X_3) = \begin{cases}
    B, & \text{if $X_3=X$}\\
    \{(X_2,X_3)\}, & \text{otherwise}
  \end{cases},
\end{align*}
 where $B = \{(X_2,X_3) \; | \; X_3 = X \}$. 
Then, by including the variable $\bold{Y}$ in the network and letting it be the sole parent of $\bold{X}_1$, we have that
\begin{align*}
P(\bold{X}_1 \; | \; \bold{Y}, \bold{X}_i, i \neq 1) = P(\bold{X}_1 \; | \; \bold{Y}).
\end{align*}
Hence, using the redundant variable, we were able to specify this additional invariance and reduce the number of parameters in the model. This method for specifying context-sensitive invariance differs from the one taken in \citep{boutilier1996context}, where the focus is on the representation of dependencies within conditional probability tables, using for example, tree structures.
\end{example}

\begin{example}[General Independence]
In the previous example, we encoded a context-sensitive independence within a Bayesian network, where the context was based on events of the form $\{\bold{X}_A = X\}$, where $X \in \mathcal{X}_A$ and $A \subseteq \{1,\ldots,n\}$. This was done by defining a partition over the space of values of the parents of a random variable (corresponding to these events), and then introducing a new variable that takes values in this partition. This same approach works for more general context-sensitive independence, i.e., contexts based on events of the form $\{\bold{X}_A \in \mathcal{D}\}$, where $\mathcal{D} \subset \mathcal{X}_A$. Hence, we can specify invariances of the form
\begin{align}
\label{eq:kdl}
P(\bold{X}_A,\bold{X}_B \; | \; f_3(\bold{X}_C)) = P(\bold{X}_A \; | \; f_3(\bold{X}_C)) \cdot P(\bold{X}_B \; | \; f_3(\bold{X}_C))
\end{align}
within a Bayesian network. Finally, notice that redundant variables also allow us to include invariances of the form:
\begin{align}
\label{eq:kdl2}
P(f_1(\bold{X}_A), f_2(\bold{X}_B) \; | \; f_3(\bold{X}_C)) = P(f_1(\bold{X}_A) \; | \; f_3(\bold{X}_C)) \cdot P(f_2(\bold{X}_B) \; | \; f_3(\bold{X}_C)),
\end{align}
the most general form that independence assumptions can take. The invariance in equation (\ref{eq:kdl}) implies that in equation (\ref{eq:kdl2}), but not the other way around. For an additional discussion about statistical invariances, and in particular, how different invariances relate to each other, see Appendix \ref{sec:stat_invariances}.
\end{example}

These examples illustrate that representing a random object with atomic variables, even if there is redundancy in them, allows more invariances to be specified by a graphical model than would be possible without all of them. Although having a larger space of independence assumptions is not always beneficial - a practitioner cannot specify invariances between variables so low-level that they are uninterpretable - the specification of invariances involving vertices is natural when modeling random graphs, and so vertex variables should generally be included in any graphical model for this problem. 

\subsection{Graph Variations}
\label{sec:graph_variations}



In this section, we briefly describe other mathematical objects - variations on the definition of a graph - that may be useful for some problems; the graphical model framework discussed in this section can accommodate these objects in a straightforward way. 

In the definition of graphs presented in Section \ref{sec:graph_def}, vertices were a subset of some vertex space $\Lambda_V$, and hence each vertex has a unique value in this space. In some applications, graphs have attributes associated with their vertices, in which case, the vertices need only be unique on some component, for example a location component, and may otherwise have common attribute values. These graphs are referred to as attributed in the literature (\citep{pfeiffer2014attributed}, \citep{jain2004central}). Suppose we have a finite vertex space $\Lambda_V$, an edge space $\Lambda_E$, and an attribute space $\mathcal{X}$. We define an attributed graph to be of the form $G = (V,X,E)$, where $V$ is a set of vertices, $X$ is a function assigning an attribute value to each vertex, and $E$ is a function assigning an edge value to every pair of vertices:
\begin{align*}
& V \subseteq \Lambda_V \\
& X: V \rightarrow \mathcal{X} \\
& E: V \times V \rightarrow \Lambda_E.
\end{align*}
Hence, every vertex in a graph has a unique value in $\Lambda_V$, and the vertices may be thought of as indices for the variables $X_v \equiv X(v)$. For example, if we let $\Lambda_V = \{1,\ldots,n\}$, then a graph may be thought of as some collection of variables of the form $\{X_i, \; i \in V\}$, where $V \subseteq \{1,\ldots,n\}$, as well as edges $E$ between them. The attribute space $\mathcal{X}$ could be, for example, a finite set of labels or a Euclidean space (for specifying positions). 


Graphs may be further generalized to allow higher-order edges, referred to as \textit{hypergraphs} \citep{berge1973graphs}. Suppose we have a finite vertex space $\Lambda_V$ and an edge space $\Lambda_{E}$. Then, we can define a generalized graph to be of the form $G = (V,E_1,\ldots,E_k)$, where $V$ is a set of vertices, and each $E_i$ is a function assigning an edge value to every group of $i$ vertices:
\begin{align*}
& V \subseteq \Lambda_V \\
& E_1: V \rightarrow \Lambda_{E} \\
& E_2: V^2 \rightarrow \Lambda_{E} \\
& \; \vdots \\
& E_k: V^k \rightarrow \Lambda_{E}.
\end{align*}
Graphs with higher-order edges may be useful in problems in which interactions can be between multiple objects, and these interactions are not a function of the pairwise interactions (\citep{zhou2006learning}, \citep{tian2009hypergraph}).

\subsection{Projections}

It is worth noting that if an attributed graph space is constrained to only graphs that: (a) contain the same set of vertices; and (b) have no edges, then the canonical graph projections (Definition \ref{def:canonical_graph_projections}), in essence, reduce to the component projections used with multivariate random variables. In this sense, the graph projections may be thought of as an extension of the component projections to graph spaces.

\chapter{Random Trees}



In this section, we consider the statistical modeling of trees; since trees are a type of graph, the random graph models described in Section 2 could be used. However, it is beneficial to instead use models that are tuned to the defining structure of trees. If the vertices in trees are assumed to take a certain form, then the edges in trees are deterministic, given the vertices in it; as a result, the tree space and its modeling are simplified. In particular, with these assumptions about the vertex space, the atomic variables correspond to individual vertices (in contrast to the atomic variables in random graphs). Hence, in basic models, the vertices in structure graphs correspond to the vertices in trees, and in more complex models (e.g., with context-sensitive dependencies), the vertices in structure graphs correspond to the vertices in the vertex space. 

We begin by considering Bayesian networks in which: (a) the directionality of edges (in the structure graph) are from root to leafs, which we refer to as branching models; and (b) models in which the directionality is the opposite, from leafs to root, which we refer to as merging models. The former is well-suited for problems in which there is a physical process such that, as time progresses, objects divide into more objects; most models in the literature are of this form. The latter model, in contrast, is well-suited for problems in which there is some initial set of objects, and as time progresses, these objects merge with each other. 

In these types of causal problems, it is generally accepted that the directionality of edges in Bayesian networks should, if possible, correspond to the causality. In some applications, however, trees are not formed by an obvious causal mechanism, and one need not limit themself to either a branching or merging model. For example, consider trees that describe the structure of objects in scenes, where vertices correspond to objects (e.g., cars, trucks, tires, doors, etc.), and edges encode when an object is a subpart of another object (\citep{jin2006context}, \citep{zhu2007stochastic}). These trees are representations of scenes, not formed by a clear time-dependent process. Hence, although distributions on these trees can be expressed using branching or merging models, they may not be expressible by them in a compact form, which is essential. In the last part of this section, we consider more general models that may be useful for these problems.

\section{Branching Models}

In this section, we consider directed and partially-directed models for random trees in which the directed edges are from root to leaf. We first consider trees without attributes, then proceed to trees with them. To demonstrate the value of the graphical model approach to random trees, we contrast it with approaches based on grammars.

\subsection{Trees}

A tree is a graph that is connected and acyclic. A rooted tree is a tree that has a partial ordering (over its vertices) defined by distance from some designated vertex referred to as the \textit{root} of the tree. Due to the structure of trees, if the vertices in them are given appropriate labels, then the edges are deterministic. For simplicity, let's consider binary trees; let the vertex space $\Lambda = \Lambda_V$ be
\begin{align*}
\Lambda =  \bigcup_{n=0}^{N}  \left[ \{v_{\text{root}} \} \times \{0,1\}^n  \right]. 
\end{align*}
where $N$ is some natural number. Thus, a vertex $v \in \Lambda$ has the form $v = (v_{\text{root}},v_1,\ldots,v_n)$, where each $v_i \in \{0,1\}$ and $v_{\text{root}}$ is some arbitrary element that denotes the root vertex (see Figure \ref{fig:tree}). Let $\pi_{1:k}$ be the projection of a vertex to its first $k$ components:
\begin{align*}
\pi_{1:k}(v) = \begin{cases}
(v_{\text{root}},v_1,\ldots,v_{k-1}), \text{ if } k \leq |v| \\
v, \text{ otherwise}
\end{cases}.
\end{align*}
Let a tree $T \subseteq \Lambda$ be a set of vertices such that, for each vertex in $T$, its ancestors are also in it:
\begin{align*}
v \in T  \implies \pi_{1:k}(v) \in T \text{ for all } k \leq |v|. 
\end{align*}
If $T = \emptyset$, we will refer to it as the empty tree. 
Given a tree $T$, define the parent, children, and siblings of a vertex $v \in T$ as:
\begin{align*}
\text{pa}(v) & = \pi_{1:|v|-1}(v) \\
\text{ch}(v) & = \{v' \in T \; | \; \text{pa}(v')=v \} \\
\text{sib}(v) & = \{v' \in T \; | \; \text{pa}(v')=\text{pa}(v) \}.
\end{align*}
 
 
 
 
  
   
    
    
    \begin{figure}
  \centering
    \includegraphics[scale=0.4]{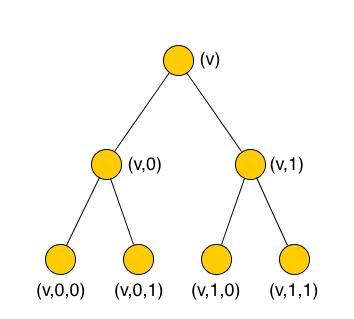}
  \caption{An example tree using the vertex labeling in the branching model (Section 3.1). }
  \label{fig:tree}
\end{figure}


\subsection{Basic Models}

In this section, we consider random tree models over a finite tree space $\mathcal{T}$ based on marginal random variables that take values in $\mathcal{T}$, i.e., are also random trees. In the next section, we expand the set of marginal variables to also include tree parts that do not take values in $\mathcal{T}$, but rather in substructures of this space. Let $\mathcal{T}$ be a set of trees that is projectable, i.e.:
\begin{align*}
T \in \mathcal{T} \implies T' \in \mathcal{T} \text{ for all } T' \in S(T),
\end{align*}
where $S(T)$ denotes the set of all subtrees of $T$. We can then define tree projections: 

\begin{definition}[Tree Projections]
Let $V \in \mathcal{T}$ be a set of vertices. Define the projection $\pi_{V}: \mathcal{T} \to \mathcal{T}_V$, where $\mathcal{T}_V = \{T \in \mathcal{T} \; | \; T \subseteq V\}$, as follows:
\begin{align*}
\pi_{V}(T) = T \cap V.
\end{align*}
Let $T_V \equiv \pi_V(T)$ denote the projection of a tree $T$ onto the vertices $V$.
\end{definition}

This projection is similar to the one used for (general) graphs, the main difference being that the set of vertices $V$ being projected onto cannot be an arbitrary subset of the vertex space, but must correspond to a tree (i.e. $V \in \mathcal{T}$). The reason is so that the projection of a tree is always a tree (in a projectable tree space $\mathcal{T}$). We consider projections onto substructures of the tree space in the next section. 

Suppose we have a distribution $P$ over the tree space $\mathcal{T}$. For each $V \in \mathcal{T}$, we can define a marginal random (tree) variable taking values in $\mathcal{T}_V$:
\begin{align*}
P_V^{\text{marg}}(T) = \sum \limits_{\substack{T' \in \mathcal{T} \\ \pi_V(T')=T}} P(T'), \quad T \in \mathcal{T}_V.
\end{align*}

For the projection family $\{\pi_V, V \in \mathcal{T}\}$, the atomic projections correspond to vertex sets $V$ that are trees with only one leaf, where a leaf is a vertex with no children (i.e., a vertex $v$ such that $\text{ch}(v)=\emptyset$); 
we will refer to a tree with a single leaf as a \textit{path-tree}. 
The reason the set of atomic projections corresponds to the set of path-trees is because: (1) any tree can be represented by a set of path-trees; and (2) no path-tree can be represented by a set of smaller path-trees. Let $\mathcal{V}_{\text{atom}}$ denote the set of path-trees:
\begin{align*}
\mathcal{V}_{\text{atom}} = \{ V \in \mathcal{T} \; | \; V \text{ is a path-tree} \}.
\end{align*} 


To define structure in distributions over the tree space $\mathcal{T}$, we can apply a graphical model. We use a Bayesian network here; let $H = (\mathcal{V}_{\text{atom}}, \mathcal{N})$ be a structure graph, where $\mathcal{N}:\mathcal{V}_{\text{atom}} \times \mathcal{V}_{\text{atom}} \to \{0,1\}$ is an edge function that is asymmetric (where the asymmetry is used in specifying edges that are directed\footnote{We consider there to be a directed edge from $V_0$ to $V_1$ if $\mathcal{N}(V_0,V_1)=1$ and $\mathcal{N}(V_1,V_0)=0$.}). To define valid distributions, the structure graph must be acyclic and must specify a dependency between any two path-trees in which one is a function of the other; thus, we will assume that there is: (1) a directed edge from every path-tree to its immediate successors (i.e., the path-trees that contain it and have one additional vertex); and (2) there is no directed edge from a path-tree to any path-tree that is a subtree of it. That is:
\begin{enumerate}
\item $\mathcal{N}(V_0,V_1)=1$ for all $V_0 \subset V_1$, $ |V_0| = |V_1|-1$.
\item $\mathcal{N}(V_1,V_0)=0$ for all $V_0 \subset V_1$.
\end{enumerate}
These requirements on the edge function $\mathcal{N}$ ensure it is consistent with the chain rule:
\begin{align}
\label{eq:basic_structure}
P(T) & = P(T_V, \; V \in \mathcal{V}_{\text{atom}})  \nonumber \\
	& = P(T_{\{v_{\text{root}} \}} ) \cdot \prod_{i=2}^{\infty} P(T_V, \; V \in \mathcal{V}_{\text{atom}}^{(i)} \; | \; T_V, \; V \in \mathcal{V}_{\text{atom}}^{(1)} \cup \ldots \cup\mathcal{V}_{\text{atom}}^{(i-1)}) \nonumber \\
	& = P(T_{\{v_{\text{root}} \}} ) \cdot \prod_{i=2}^{\infty} P(T_V, \; V \in \mathcal{V}_{\text{atom}}^{(i)} \; | \; T_V, \; V \in \mathcal{V}_{\text{atom}}^{(i-1)}) 
\end{align}
where $\mathcal{V}_{\text{atom}}^{(i)} = \{V \in \mathcal{V}_{\text{atom}} \; : \; |V|=i \}$ denotes the set containing path-trees of cardinality $i$. 





\subsection{Substructures}
\label{sec:substructures_trees}

Similar to multivariate random variables, the use of projections onto substructures (Section \ref{sec:substructure_sec}) is important when modeling random trees. These additional projections allow one to form additional marginal random variables, which in turn, allow statistical models to specify more structure in distributions. 

Let a \textit{shifted} tree be a pair $(T,v_0)$, where $v_0 \in \Lambda$ is a vertex and $T \subseteq \Lambda$ is a set of vertices such that:
\begin{enumerate}
\item $|v_0| \leq |v| \text{ for all } v \in T$;
\item $v \in T \implies \pi_{1:k}(v) \in T \text{ for all } |v_0| \leq k \leq |v|$.
\end{enumerate}
In other words, a shifted tree may be thought of as a tree in which $v_0$ serves as the root vertex. 
For the tree space $\mathcal{T}$, let $\mathcal{T}(v_0)$ denote the set of shifted trees with root vertex $v_0$, i.e.:
\begin{align*}
\mathcal{T}(v_0) = \{ T \cap \Lambda(v_0) \; | \; T \in \mathcal{T} \},
\end{align*}
where $\Lambda_V(v_0)$ denote the set of vertices that are descendants of $v_0$, i.e.,:
\begin{align*}
\Lambda(v_0) = \{v \in \Lambda \; | \; \pi_{1:|v_0|}(v)=v_0 \}. 
\end{align*}
For a vertex $v_0 \neq v_{\text{root}}$, the space $\mathcal{T}(v_0) \not \subseteq \mathcal{T}$ is not a subset of the tree space, but rather a substructure of the space $\mathcal{T}$. 
For a given vertex $v_0$, we define the projection taking trees in $\mathcal{T}$ to trees in $\mathcal{T}(v_0)$ as follows:

 

\begin{definition}[Substructure Projections]
Let $v_0 \in \Lambda$ be a vertex and let $V = \Lambda(v_0)$ be the set of all its descendants. Define the projection $\pi_{V}: \mathcal{T} \to \mathcal{T}(v_0)$ as:
\begin{align*}
\pi_{V}(T) = T \cap V.
\end{align*}
Let $T_V \equiv \pi_V(T)$ denote the projection of a tree $T$ onto the vertices $V$.
\end{definition}




For a random tree $\bold{T}$ taking values in a tree space $\mathcal{T}$, the substructure projections define marginal random (shifted tree) variables of the form $\bold{T}_V = \pi_V(\bold{T})$, where $V = \Lambda(v)$ and $v \in \Lambda$. 
Each substructure is itself equipped with tree projections. 
Hence, allowing for both projections to substructures and then projections to trees within this substructure, 
the set of all projections on $\mathcal{T}$ is the projection family $\{\pi_V, V \in \mathcal{T}(v), v \in \Lambda \}$, and the atomic projections are just the projections onto individual vertices, i.e., those in the set $\{\pi_{\{v\}}, v \in \Lambda\}$. 

To define structure in distributions over the tree space $\mathcal{T}$, we can apply a graphical model. Let $H = (\Lambda, \mathcal{N})$ be a structure graph, where $\mathcal{N}:\Lambda \times \Lambda \to \{0,1\}$ is an edge function. As before, let the structure graph be acyclic, and require it specify a dependency (either directly or indirectly) between any two vertices in which one is an ancestor of the other. 



Similar to general graphs, it will often be useful in the statistical modeling of trees to incorporate invariance assumptions about (shifted) trees that are isomorphic to each other. Recall, two graphs are said to be isomorphic if they share the same edge structure (see Section \ref{sec:chap2_isomorphisms}). Similarly, two rooted trees are said to be isomorphic if they share the same edge structure, as well as the same partial ordering structure: 

\begin{definition}[Tree Isomorphism]
A tree $(T,v_0)$ is \textit{isomorphic} to a tree $(T',v_0')$ if there exists a bijection $f: T \rightarrow T'$ such that:
\begin{enumerate}
\item $f(v_0) = v_0'$.
\item $v \in \text{ch}(u) \implies f(v) \in \text{ch}(f(u))$.
\end{enumerate}
Two trees that are isomorphic are denoted by $(T,v_0) \simeq (T',v_0')$. (The set of children $\text{ch}(u)$ for vertex $u$ is with respect to the tree $T$, and the set of children $\text{ch}(f(u))$ for vertex $f(u)$ is with respect to the tree $T'$.)
\end{definition}


\begin{example}[Galton-Watson Model]
The Galton-Watson model is a classic random tree model that makes two invariance assumptions. The first is that a vertex is conditionally independent of all other vertices except its parent and siblings:
\begin{align*}
\mathcal{N}(v_0,v_1) = \begin{cases}
1, \text{ if } v_0 \in \text{sib}(v_1) \text{ or } v_0 = \text{pa}(v_1) \\
0, \text{ otherwise}
\end{cases},
\end{align*}
where the children and sibling functions are with respect to the total tree $T = \Lambda$. 
The second invariance assumption is that, conditioned on their roots, shifted trees that are isomorphic to each other have the same probability. That is, for all shifted trees $(T,v_0)$ and $(T',v_0')$ such that $(T,v_0) \simeq (T',v_0')$, we assume that $P(T \; | \; T_{\{v_0\}} ) = P(T' \; | \; T_{\{v_0'\}}' )$. (We have simplified the notation from $P^{\text{marg}}_V$ to $P$; the distribution should be clear from the context.)
Thus, we have:
\begin{align*}
P(T) & = P(T_{\{v\}}, \; v \in \Lambda)  \\
	& = P(T_{\{v_{\text{root}} \}} ) \cdot \prod_{i=2}^{\infty} P(T_{\{v\}}, \; v \in \Lambda^{(i)} \; | \; T_{\{v\}}, \; v \in \Lambda^{(i-1)}) \\
	& = P(T_{\{v_{\text{root}} \}} ) \cdot \prod_{i = 2}^{\infty} \left[ \prod  \limits_{\substack{v \in \Lambda^{(i-1)} }} P(T_{\text{ch}(v)} \; | \; T_{\{v\}} ) \right] \\
	& = P(T_{\{v_{\text{root}} \}} ) \cdot \prod_{v \in T} \mu( |T_{\text{ch}(v)}| ),
\end{align*}
where $\Lambda^{(i)} = \{v \in \Lambda \; : \; |v|=i\}$ denotes the set containing vertices of cardinality $i$, and where $\mu$ is a distribution over the number of children (e.g., over $\{0,1,2\}$ in binary trees). The second to last equality follows from the independence assumptions for this model, and the last equality follows from the isomorphism invariance assumption.
\end{example}




\subsection{Attributed Trees}

In many real-world problems, the vertices in trees have attributes associated with them. In most of the literature on attributed trees, grammars are used to define the tree space (i.e., the set of trees the grammar can produce). These grammars produce trees by production rules; beginning with the empty tree, larger trees are incrementally built by the iterative application of these rules. 

For context-free grammars, distributions can be defined over trees by associating a probability with each production rule. However, it is well-known that this approach (associating probabilities to production rules) does not generalize to the case of context-sensitive grammars (i.e., does not produce well-defined distributions for this grammar). The reason is because, in context-sensitive grammars, the order in which production rules are applied now matters (in determining what trees can be produced), and hence this grammar must have an ordering policy that specifies the next production rule to apply, given the current tree; this policy is a function that generally depends on many of the vertices in the current tree. Hence, to define a distribution over this tree space, the conditional probability of the next tree in a sequence, given the previous one, would not (in general) be conditionally independent of vertices even far removed in tree distance from the vertices being used by the production rule itself. In other words, to make well-defined distributions for a context-sensitive grammar, very high-order models are required.



In this section, rather than trying to define distributions in terms of grammars, we use a graphical model approach; by using the marginal random variables in a random tree, it becomes tractable to specify dependencies and make well-defined tree distributions that are, loosely speaking, context-sensitive. Let an attributed tree be a pair $T=(V,X)$, where $V \in \mathcal{T}$ is a tree and $X: V \to \mathcal{X}$ is a function taking each vertex to some attribute value in an attribute space $\mathcal{X}$. For an attribute space $\mathcal{X}$, let $\mathcal{T}_{\mathcal{X}}$ denote the space of attributed trees:
\begin{align*}
\mathcal{T}_{\mathcal{X}} = \{(V,X) \; | \; V \in \mathcal{T}, \; X: V \to \mathcal{X} \}.
\end{align*}

Since $\mathcal{X}$ need not be finite, the space $\mathcal{T}_{\mathcal{X}}$ may not be finite either (we have only assumed the vertex space $\Lambda$ is finite, implying a finite number of projections). The definition of a projection on a tree can be extended to a projection on an attributed tree in a straightforward manner: for a tree $T=(V,X)$, let the projection $\pi_{V_0}(T) = (V \cap V_0,X \restriction_{V \cap V_0})$ be the intersection of the tree's vertices with $V_0 \subseteq \Lambda$ and the restriction of the attribute function $X$ to these vertices. We let $T_{V_0} \equiv \pi_{V_0}(T)$. 

Let an attributed shifted trees be a triple $T = (V,X,v_0)$, where $v_0$ is the designated root and $V \in \mathcal{T}(v_0)$, the space of shifted trees with respect to $v_0$. The definition of isomorphisms for attributed trees is the same as before, except with the additional requirement that the attribute values also match: the trees $T=(V,X,v_0)$ and $T'=(V',X',v_0')$ are isomorphic to each other if there exists a bijection $f:V \to V'$ such that:
\begin{enumerate}
\item $(V,v_0) \simeq (V',v_0')$ with respect to $f$.
\item $X(v) = X'(f(v)) \text{ for all } v \in T$.
\end{enumerate}
Two trees that are isomorphic are denoted by $T \simeq T'$. 


\begin{example}[Probabilistic Context-Free Grammar]
A probabilistic context-free grammar is a random tree model that may be thought of as an extension of the Galton-Watson model to attributed trees. The attribute space is assumed to have the form $\mathcal{X} = \mathcal{X}_{\text{leaf}} \cup \mathcal{X}_{\text{non-leaf}}$, where $\mathcal{X}_{\text{leaf}} \cap \mathcal{X}_{\text{non-leaf}} = \emptyset$, and the tree space is assumed to be restricted to trees such that leaf vertices only take attribute values in $\mathcal{X}_{\text{leaf}}$ and non-leaf vertices only take attribute values in $\mathcal{X}_{\text{non-leaf}}$. 

The model makes two invariance assumptions. The first is that it assumes a vertex is conditionally independent of all other vertices except its parent and siblings; in terms of its structure graph, the independence assumptions are:
\begin{align*}
\mathcal{N}(v_0,v_1) = \begin{cases}
1, \text{ if } v_0 \in \text{sib}(v_1) \text{ or } v_0 = \text{pa}(v_1) \\
0, \text{ otherwise}
\end{cases}.
\end{align*}
The second invariance assumption is that, conditioned on their roots, shifted trees that are isomorphic to each other have the same probability. That is, for all shifted trees $T=(V,X,v_0)$ and $T'=(V',X',v_0')$ such that $T \simeq T'$, we assume that $P(T \; | \; T_{\{v_0\}} ) = P(T' \; | \; T_{\{v_0'\}}' )$. 
Thus, we have:
\begin{align*}
P(T) & = P(T_{\{v\}}, \; v \in \Lambda)  \\
	& = P(T_{\{v_{\text{root}} \}} ) \cdot \prod_{i=2}^{\infty} P(T_{\{v\}}, \; v \in \Lambda^{(i)} \; | \; T_{\{v\}}, \; v \in \Lambda^{(i-1)}) \\
	& = P(T_{\{v_{\text{root}} \}} ) \cdot \prod_{i = 2}^{\infty} \left[ \prod  \limits_{\substack{v \in \Lambda^{(i-1)} }} P(T_{\text{ch}(v)} \; | \; T_{\{v\}} ) \right] \\
	& = P(T_{\{v_{\text{root}} \}} ) \cdot \prod_{v \in V(T)} \mu( \tilde{T}_{\text{ch}(v)} \; | \; \tilde{T}_{\{v\}} ),
\end{align*}
where $\mu$ is a distribution over the space
\begin{align*}
\mathcal{T}_{\mathcal{X}}^{(1)} = \{ T \in \mathcal{T}_{\mathcal{X}} \; : \; |v| \leq 2 \text{ for all } v \in V(T)\},
\end{align*}
the set of trees of length less than or equal to one, and $\tilde{T}$ denotes a tree such that $\tilde{T} \simeq T$ and $\tilde{T} \in \mathcal{T}_{\mathcal{X}}$, i.e., a non-shifted version of the shifted tree $T$. The second to last equality follows from the independence assumptions for this model, and the last equality follows from the isomorphism invariance assumption. 
The distribution $\mu$ is usually assumed to have zero probability over a portion of its input trees (or, equivalently, the tree space is assumed to be constrained).
\end{example}



\begin{example}[Context-Sensitive Random Tree]
We will refer to a random tree model as context-sensitive if, compared to the probabilistic context-free grammar in the previous example, it has the following additional dependencies. As before, a vertex depends on its siblings and parent, but now also depends on certain vertices that are adjacent to its parent as well. For each vertex $v$, define its adjacent vertices as the set
\begin{align*}
\text{adj}(v) = \{v' \in \Lambda \; : \; |v'| = |v| \text{ and } d(v,v') \leq 1\},
\end{align*}
where $d$ denotes some function between vertices of the same level in a tree. For example, if one visualizes a tree by depicting it as an image on the plane, as in Figure  \ref{fig:tree}, vertices on each level will have an ordering based on which vertices come before others from left to right. In linguistics, this ordering coincides with the order words occur in sentences, loosely speaking. More formally, we can assume there is an order relation $\leq$ on each set $\Lambda^{(i)} = \{v \in \Lambda \; : \; |v|=i\}$, and then define $d$ based on this ordering. Then, in terms of its structure graph, the independence assumptions are:
\begin{align*}
\mathcal{N}(v_0,v_1) = \begin{cases}
1, \text{ if } v_0 \in \text{sib}(v_1) \text{ or } v_0 = \text{pa}(v_1) \text{ or } v_0 \in \text{adj}( \text{pa}(v_1)) \\
0, \text{ otherwise}
\end{cases}.
\end{align*}
This structure graph could also have directed edges, not just between adjacent levels of the tree, but across multiple levels of the tree. Similar to probabilistic context-free grammars, this random tree model also makes isomorphism assumptions, except with respect to subsets of vertices that may not be trees. 
\end{example}

\section{Merging Models}


In the previous section, we used a vertex space in which the label of each vertex encoded its entire ancestry; hence, if we know a vertex is in a tree, then we also know its ancestors as well, and this limits one to branching models. In this section, we consider a vertex space in which the label of each vertex instead encodes its descendants, allowing merging models for random trees: beginning with some set of initial objects, trees can be formed by iteratively merging them. Examples include the modeling of cell fusion (i.e., cells that combine) and the modeling of mergers between industrial corporations (which, in the end, form monopolies). We present a simplified version of the vertex space here; it can be extended to more sophisticated forms. 
As before, due to the structure of trees, if the vertices in them are given appropriate labels, then the edges are deterministic. 

Suppose we have some set of vertices $\Lambda_{\text{leaf}}$ such that, for every tree, its leafs are in this set; beginning with some set of vertices $V_{\text{leaf}} \subseteq \Lambda_{\text{leaf}}$, trees will be constructed by merging them. Letting $N = |\Lambda_{\text{leaf}}|$, define the vertex space $\Lambda$ to be:
\begin{align*}
\Lambda =  \mathbb{P}(\{1,\ldots,N\}) \setminus \{\emptyset\}
\end{align*}
Thus, a vertex $v \in \Lambda$ has the form $v \subseteq \{1,\ldots,N\}$. As before, we assume binary trees for simplicity; a tree $T \subseteq \Lambda$ is a set of vertices such that:
\begin{enumerate}
\item There exists a vertex $v \in T$ such that $|v| > |v'|$ for all $v' \in T$, $v' \neq v$. This vertex corresponds to the root of the tree.
\item For each vertex $v \in T$, its cardinality is $|v| = 2^n$, for some $n \in \{0,1,2, \ldots\}$. The value $n$ for a vertex corresponds to its level, which we denote by $\text{level}(v)$.
\item For each vertex $v\in T$ such that $|v|>1$, there exists a binary partition $\{v',v''\}$ of this vertex (i.e., $v = v' \cup v''$ and $v' \cap v'' = \emptyset$), such that $v', v'' \in T$ and $|v'| = |v''|$.
\end{enumerate}
An example tree is shown in Figure \ref{fig:tree2}. If $T = \emptyset$, we will refer to it as the empty tree. In this tree definition, a vertex $v \in T$ is a leaf if and only if it has cardinality of one (i.e., $|v|=1$). Hence, the label of each individual vertex defines if it is a leaf or not (unlike in the previous section). For a tree $T$, let $\text{leaf}(T)$ denote the set of its vertices that are leafs:
\begin{align*}
\text{leaf}(T) = \{v \in T \; : \; |v|=1 \}.
\end{align*}
This distinction, in turn, means that for a subset $T' \subseteq T$ to be a tree (i.e., a subtree of $T$), its leafs must be a subset of the leafs of $T$ (i.e., $\text{leaf}(T') \subseteq \text{leaf}(T)$). This requirement is in contrast to the previous section, where trees and their subtrees had to have the root vertex be in common. 
 



Let $\mathcal{T}$ be a set of trees that is projectable, i.e.: $T \in \mathcal{T} \implies T' \in \mathcal{T} \text{ for all } T' \in S(T)$, where $S(T)$ denotes the set of all subtrees of $T$. As before, we can then define tree projections: 

\begin{definition}[Tree Projections]
Let $V \in \mathcal{T}$ be a set of vertices. Define the projection $\pi_{V}: \mathcal{T} \to \mathcal{T}_V$, where $\mathcal{T}_V = \{T \in \mathcal{T} \; | \; T \subseteq V\}$, as follows:
\begin{align*}
\pi_{V}(T) = T \cap V.
\end{align*}
Let $T_V \equiv \pi_V(T)$ denote the projection of a tree $T$ onto the vertices $V$.
\end{definition}



In the case of the projection family $\{\pi_V, V \in \mathcal{T}\}$, the atomic projections are not a subset, but rather coincide with the entire projection family. However, assuming projections to substructures as well, as was done in the branching models, we then arrive at the same set of atomic projections, the set of individual vertices $\{\pi_{\{v\}}, v \in \Lambda\}$. 

To define structure in distributions over the tree space $\mathcal{T}$, we can apply a graphical model. We use a Bayesian network here; let $H = (\Lambda, \mathcal{N})$ be a structure graph, where $\mathcal{N}: \Lambda^2 \to \{0,1\}$ is an edge function that is asymmetric (where the asymmetry is used in specifying edges that are directed). To define valid distributions, the structure graph must be acyclic; for merging models, we assume that edges are in the direction from leafs to root. We must specify a dependency between any two vertices in which one is a function of the other; thus, we assume:
\begin{enumerate}
\item $\mathcal{N}(v_0,v_1)=1$ for all $v_0 \subset v_1$ and $\text{level}(v_0) = \text{level}(v_1)-1$.
\item $\mathcal{N}(v_1,v_0)=0$ for all $v_0 \subset v_1$.
\end{enumerate}
These requirements on the edge function $\mathcal{N}$ ensure it is consistent with the chain rule:
\begin{align}
\label{eq:basic_structure}
P(T) & = P(T_V, \; V \in \Lambda)  \nonumber \\
	& = P(\text{leaf}(T) ) \cdot \prod_{i=1}^{\infty} P(T_{\{v\}}, \; v \in \Lambda^{(i)} \; | \; T_{\{v\}}, \; v \in \Lambda^{(1)} \cup \ldots \cup \Lambda^{(i-1)}) \nonumber \\
	& = P(\text{leaf}(T) ) \cdot \prod_{i=1}^{\infty} P(T_{\{v\}}, \; v \in \Lambda^{(i)} \; | \; T_{\{v\}}, \; v \in \Lambda^{(i-1)}) 
\end{align}
where $\Lambda^{(i)} = \{v \in \Lambda \; : \; \text{level}(v)=i \}$ denotes the set of vertices that are on level $i$. 

If one assumes that a vertex can only merge with one other vertex on a given layer, then complex dependencies are introduced in which a vertex depends on more than just its children; this situation is similar to that of context-sensitive grammars in branching models, except in the reverse direction. In this case, complex models can result. 

    \begin{figure}
  \centering
    \includegraphics[scale=0.4]{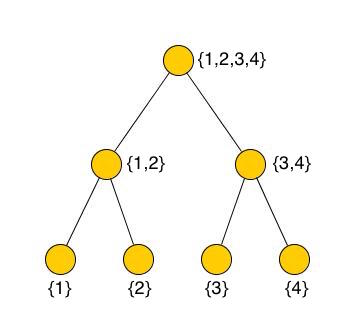}
  \caption{An example tree using the vertex labeling in the merging model (Section 3.2). }
  \label{fig:tree2}
\end{figure}

\section{General Models}




In the previous sections, we used specialized vertex spaces for defining trees; using vertices with labels that specify its set of possible children or possible parents (and assuming, in any valid tree, these sets are non-overlapping), then trees have deterministic edges, given the vertices. 
However, we could instead define trees in terms of an arbitrary vertex space and then define the tree space by restricting the corresponding graph space to only trees. This has the advantage of allowing one to employ any type of graphical model for random graphs (Section 2). In this more general formulation, distributions need not be defined in terms of how trees are incrementally constructed by a top-down or bottom-up process, but rather how they deconstruct (e.g., into subtrees). This allows, for some problems, a more natural method for defining distributions since it may allow a more compact representation of dependencies. 


We will assume the vertex space has some minimal structure, allowing us to define trees based on basic conditions on the vertices and edges. 
Suppose we have a vertex space $\Lambda$ of the form
\begin{align*}
\Lambda = \bigcup_{n=0}^{N} \Lambda^{(n)}, 
\end{align*}
where each space $\Lambda^{(n)}$ corresponds to the set of vertices that can occur on the $n$th level of the tree (i.e., the distance from a vertex in this set to the root is assumed to be $n$ in any tree). Further, we assume $\Lambda^{(n)} \cap \Lambda^{(m)} = \emptyset$ for every $n \neq m$. For example, for modeling real-world scenes, often one assumes some fixed hierarchy of objects (e.g., cars occur on the $k$th level and car tires occur on the $(k+1)$th level). Finally, suppose the edge space $\Lambda_E = \{0,1\}$ is binary. 

Let a tree be a graph $T=(V,E)$ with respect to this vertex space and edge space (i.e., where $V \subset \Lambda$ is a set of vertices and $E:V^2 \to \{0,1\}$ a binary edge function) 
such that the following conditions are satisfied: letting $V^{(n)} \equiv V \cap \Lambda^{(n)}$, we have that:
\begin{enumerate}
\item There is only a single root vertex: if $V \neq \emptyset$, then $|V^{(0)}|=1$.
\item Every (non-root) vertex has one and only one parent: for $n = 1,2,\ldots$, for all $v \in V^{(n)}$, we have:
	\begin{align*}
	\sum_{v' \in V^{(n-1)}} E(v',v) = 1.
	\end{align*}
\item There are only edges between adjacent layers: for all $n,m$ such that $|n-m| \neq 1$, we have that $E(v,v')=0$ for all $v \in V^{(m)}$ and $v' \in V^{(m)}$. 
\end{enumerate}
Let $\mathcal{T}$ be the space of all such trees. A distribution over this space can be defined using a random graph model; in particular, we may apply an undirected or partially directed model. As mentioned, this additional flexibility may be useful for the modeling of some problems in which there is no obvious causal mechanism.

\chapter{General Random Objects}
\label{sec:General_Random_Objects}


In this section, we consider a general formulation of graphical models on a sample space $\Omega$ based on a family of random variables with basic consistency and completeness properties. In the literature, the definition of 
consistency for random variables is stated in terms of distributions (\citep{chung2001course}, \citep{lamperti2012stochastic}). In this work however, we find it convenient to define consistency in terms of the functions themselves (rather than the distributions induced by them). This more elemental definition will be useful in modeling over more general spaces, where to make independence assumptions on distributions, a consistent projection family must first be specified. The projections from this family then define random variables that are consistent (referred to as marginal variables). 
We begin by considering the case in which projections are from a given sample space to subsets of it; the random graph model discussed in Section 2 uses projections of this form. Then, we consider more general projections, where for example, the random tree model discussed in Section 3, and the traditional formulation of graphical models for multivariate random variables are instances. For simplicity, we limit the formulation here to finite projection families.

\section{Projection Families}
\label{sec:proj_families}


Suppose we have a random object taking values in some space $\Omega$, and suppose we have a family $\Pi$ of projections 
where each projection has the form $\pi: \Omega \to \Omega' \subseteq \Omega$. Recall, a function $\pi$ is a projection if $\pi \circ \pi = \pi$, i.e., projecting an object more than once does not change its value. In order to produce random variables that are consistent with each other, the projections must be consistent with each other: 

\begin{definition}[Consistency]
The projections $\pi_1: \Omega \to \Omega_1$ and $\pi_2: \Omega \to \Omega_2$ are \textit{consistent} if:
\begin{align*}
& \Omega_1 \subseteq \Omega_2 \implies \pi_1 \circ \pi_2 = \pi_1 \\
& \Omega_2 \subseteq \Omega_1 \implies \pi_2 \circ \pi_1 = \pi_2.
\end{align*}
\end{definition}

In other words, two projections are consistent if: (a) one's image is not a subset of the other's; or (b) projecting an object onto the smaller space is the same as first projecting the object onto the larger space, and then projecting onto the smaller space. We say that a projection family is consistent if every pair of projections in it are consistent. A consistent family of projections defines a consistent family of random variables (referred to as \textit{marginal} variables).

Although this definition of consistent projections corresponds to the definition of consistent random variables, it will be useful when formulating graphical models to assume a stronger form:


\begin{definition}[Strong Consistency]
The projections $\pi_1: \Omega \to \Omega_1$ and $\pi_2: \Omega \to \Omega_2$ are \textit{strongly consistent} if $\Omega_1 \cap \Omega_2 \neq \emptyset$ implies that there exists a projection $\pi_3:\Omega \to \Omega_1 \cap \Omega_2$ consistent with $\pi_1$ and $\pi_2$. If this projection exists, then it is unique. 
\end{definition}



As before, we say that a projection family is strongly consistent if every pair of projections in it are strongly consistent. The canonical projection family for random graphs (Section \ref{sec:chap2_discussion}) is strongly consistent, and the canonical projection family for random vectors (i.e., the coordinate projections) is strongly consistent as well (see next section). We illustrate the importance of strong consistency in modeling with an example:

\begin{example}[Consistency and Conditional Independence]
Let $\Omega$ be a sample space with a distribution over it. In order to model this distribution using independence assumptions, we need to specify some random variables. Suppose we have two projections $\pi_0:\Omega \to \Omega_0$ and $\pi_1:\Omega \to \Omega_1$ such that $\Omega_0 \cap \Omega_1 \neq \emptyset$, $\Omega_0 \not \subseteq \Omega_1$, and $\Omega_1 \not \subseteq \Omega_0$. Since neither projection's image is a subset of the other's, these two projections are consistent. However, the question arises about the nature of their agreement when objects are projected to the intersection $\Omega_0 \cap \Omega_1$; there are two scenarios to consider.

First, suppose $\pi_0$ and $\pi_1$ are strongly consistent; then there exists a unique projection $\pi_2:\Omega \to  \Omega_0 \cap \Omega_1$ consistent with both $\pi_0$ and $\pi_1$ (and so the set $\{\pi_0,\pi_1,\pi_2\}$ is consistent). A standard assumption is that the random variables $\pi_0$ and $\pi_1$ are conditionally independent given $\pi_2$, which we denote by
\begin{align*}
(\pi_0 \perp \pi_1 \; | \; \pi_2).
\end{align*}

Now, suppose $\pi_0$ and $\pi_1$ are not strongly consistent; then there does not exist a projection $\pi_2:\Omega \to  \Omega_0 \cap \Omega_1$ consistent with both $\pi_0$ and $\pi_1$. In order to specify a conditional independence assumption analogous to the one above, define $\pi_{2}': \Omega \to \Omega_0 \cap \Omega_1$ as the projection consistent with $\pi_0$ and define $\pi_{2}'': \Omega \to \Omega_0 \cap \Omega_1$ as the projection consistent with $\pi_1$. 
Now, if we want to specify conditional independence between $\pi_0$ and $\pi_1$, we must condition on both $\pi_{2'}$ and $\pi_{2''}$, i.e.:
\begin{align*}
(\pi_0 \perp \pi_1 \; | \; \pi_{2}', \pi_{2}'').
\end{align*}
This illustrates that, to specify conditional independence between random variables that are not strongly consistent, the structure graph in a graphical model both needs to be larger (i.e., incorporate more variables) and needs more edges, than if they were.

\end{example}



This example motivates formulating graphical models in terms of strongly consistent projections. It will be convenient to index projection families so as to indicate which projection's images are subsets of other's. This can be done as follows. For a finite consistent projection family $\Pi$, there exists some finite set $B$ and $\mathcal{A} \subseteq \mathbb{P}(B)$, such that we can write: 
\begin{align*}
\Pi = \{\pi_A: \Omega \to \Omega_A, \; A \in \mathcal{A} \},
\end{align*}
and where:
\begin{enumerate}
\item $\Omega_B = \Omega$.
\item $A' \subseteq A \Longleftrightarrow \Omega_{A'} \subseteq \Omega_A$.
\item $A' \cap A = \emptyset \Longleftrightarrow \Omega_{A'} \cap \Omega_A = \emptyset$.
\end{enumerate}
Further, we will assume that $B$ is a minimal set for indexing $\Pi$ in this way (in the sense that there does not exist $B'$ such that $|B'| < |B|$ and the above holds). Thus, the indices in an index set $\mathcal{A}$ show when the images of projections intersect or are subsets. Henceforth, we assume projection families are indexed in this way. 

We now define a completeness condition for a projection family, which will also be useful in modeling: 
\begin{definition}[Completeness]
A projection family $\Pi$ is \textit{complete} if its index set $\mathcal{A}$ is closed under intersections:
\begin{align*}
A, A' \in \mathcal{A} \implies A \cap A' \in \mathcal{A}
\end{align*}
\end{definition}
In other words, for any two projections $\pi_A:\Omega \to \Omega_A$ and $\pi_{A'}:\Omega \to \Omega_{A'}$ in $\Pi$, a projection of the form $\pi_{A \cap A'}: \Omega \to \Omega_{A \cap A'}$ also exists in it. Notice that if a projection family $\Pi$ is consistent and complete, then it is also strongly consistent. Conversely, if a projection family is strongly consistent, then it can be made complete by augmenting it with additional projections. For modeling purposes, the value of completing a projection family in this sense is that it provides a larger space of possible independence assumptions. Since the traditional formulation of graphical models is in terms of a consistent, complete system of projections, we will define the extended formulation likewise. We now define the notion of atomic projections.

Loosely speaking, for a projection family $\Pi$, a projection in it is atomic if: (1) there does not exist a projection in this family that projects to a subset of its image; or (2) if there are projections in this family that project to a subset of its image, then this set loses information. The second condition ensures that any object projected by a set of atomic projections can be reconstructed. To define this more formally, we introduce some notation. For a projection family $\Pi$ indexed by $\mathcal{A}$, let $\Pi_{\mathcal{B}} \subseteq \Pi$ denote the subset of projections indexed by $\mathcal{B} \subseteq \mathcal{A}$, i.e.,:
\begin{align*}
\Pi_{\mathcal{B}} = \{ \pi_A,  \; A \in \mathcal{B}\}.
\end{align*}
We say that a set of projections $\Pi_{\mathcal{B}}$ is \textit{invertible} over a set $\Omega' \subseteq \Omega$ if there exists a function $\Pi_{\mathcal{B}}^{-1}$ such that
\begin{align*}
	\Pi_{\mathcal{B}}^{-1} (\Pi_{\mathcal{B}}(w)) = w, \; \forall w \in \Omega'.
\end{align*}
We define atomic projections as follows:

\begin{definition}[Atomic Projections]
For a finite projection family $\Pi$ indexed by $\mathcal{A}$, a projection $\pi_A$ in this family is \textit{atomic} if:
\begin{enumerate}
\item $\mathcal{B}(A) = \emptyset$; or
\item If $\mathcal{B}(A) \neq \emptyset$, then the projection set $\Pi_{\mathcal{B}(A)}$ is not invertible over $A$. 
\end{enumerate}
where $\mathcal{B}(A) = \{A' \in \mathcal{A} \; | \; A' \subset A\}$.
\end{definition}
In other words, for a projection family, the atomic projections are those with either the smallest images, or if there are projections with smaller images, they cannot be reconstructed from them. We will call a random variable atomic if it corresponds to an atomic projection. Finally, to be used in modeling, we need to assume that a projection family has enough coverage over a space $\Omega$ so that it can be used for representing objects in it:

\begin{definition}[Atomic Representation]
For a finite projection family $\Pi$ over $\Omega$, a set of atomic projections $\Pi_{\text{atom}} \subseteq \Pi$ is an \textit{atomic representation} of the space $\Omega$ if it is invertible over $\Omega$. 
\end{definition}
If a finite projection family $\Pi$ over $\Omega$ contains the identity projection $\text{I}_{\Omega}$, then there exists an atomic representation of $\Omega$ within $\Pi$. If a finite projection family $\Pi$ over $\Omega$ contains the identity projection $\text{I}_{\Omega}$ and is consistent and complete, then it has a unique atomic representation. For defining a graphical model over $\Omega$ with respect to a projection family $\Pi$, we will let its structure graph correspond to an atomic representation of $\Omega$ within $\Pi$ (i.e., the vertices in the structure graph will correspond to the projections in the atomic representation); if the atomic representation is unique, then so will be the vertex set in the structure graph. We can now express, for graphical models, the requirements on projection families. 




Suppose we have a consistent, complete system of projections $\Pi$ over an object space $\Omega$. Further, assume $\Pi$ is finite, non-empty, and contains the identity projection $\text{I}_{\Omega}$. With only these assumptions on the projection family, we can model distributions over $\Omega$ using independence and factorization, the invariances used in graphical models. 
The projections, since they are consistent, define a set of marginal random variables, as well as a unique set of atomic random variables, and so we can encode independence assumptions in a compact form using them. 
A projection family on the space $\Omega$ also gives rise to a Gibbs representation for distributions over it:
\begin{align}
\label{eq:factorization_general}
P(w) = \frac{1}{Z} \exp \left [  \sum_{A \in \mathcal{A}} \psi_A(w_A) \right ],
\end{align}
where $\psi_A: \Omega_A \to \mathbb{R} \cup \{-\infty\}$, and where $w_A \equiv \pi_A(w)$, facilitating factorization and the use of undirected models. If there are functional dependencies in the atomic projections\footnote{There are functional dependencies in the atomic representation whenever there exists a projection $\pi_A$ in it such that the set $\mathcal{B}(A) = \{A' \in \mathcal{A} \; | \; A' \subset A\} \neq \emptyset$ is non-empty.}, then the objects in $\Omega$ have structural constraints, and the structure graph must respect them. If there are no functional dependencies, then the Hammersley-Clifford theorem may be directly applied; otherwise, a partially directed network may be necessary.

The formulation of graphical models given here encompasses the random graph models from Section 2. However, notice that in graphical models for multivariate random variables, the projections are not to subsets of the sample space, but rather to substructures. We now turn to this topic, and extend the formulation to include these more general projections. 

\section{Substructures} 
\label{sec:substructure_sec}


In the previous section, we discussed projection families in which each projection's image was a subset of its domain (i.e., for each projection $\pi: \Omega \to \Omega'$, we have $\Omega' \subseteq  \Omega$). In this section, we now consider functions on $\Omega$ in which their images may not be a subset (i.e., $\Omega' \not \subseteq  \Omega$); for convenience, we will view these functions as mapping into substructures, and refer to them as projections. This extension is important because it allows projection families to be larger, which in turn, allows additional structure to be incorporated within models. An example of a projection from a space $\Omega$ to a substructure is the the projection of a vector to one of its components. For simplicity, the development here is not given in terms of substructures; the ideas can be stated in terms of the projections, without making more explicit assumptions about the structure of the space $\Omega$.   

Suppose we want to define a distribution over a space $\Omega$ using graphical models, and further, suppose we have a set of projections of the form $\pi: \Omega \to \Omega'$; since the images of these projections are not necessarily subsets of $\Omega$, the composition of these projections is no longer well-defined (i.e., the image of one projection is not necessarily a subset of the domain of another projection). In order to define a notion of consistency for this projection family, there must exist projections between these spaces. For a projection family 
\begin{align*}
\Pi = \{\pi_A: \Omega \to \Omega_A, \; A \in \mathcal{A} \},
\end{align*}
where $\mathcal{A}$ is an index set, suppose that the following conditions hold:
\begin{enumerate}
\item (Completeness) For all $A,A' \in \mathcal{A}$ such that $A \cap A' \neq \emptyset$, we have that $A \cap A' \in \mathcal{A}$.
\item (Consistency) For all $A,A' \in \mathcal{A}$ such that $A' \subseteq A$, there exists a projection of the form
\begin{align*}
\pi_{A \to A'}: \Omega_A \to \Omega_{A'},
\end{align*}
and this projection is defined by
\begin{align*}
\pi_{A \to A'} \circ \pi_A = \pi_{A'}.
\end{align*}
\end{enumerate}
If these conditions hold, then we say the projection family $\Pi$ is consistent and complete (this is a natural extension of the definitions in the previous section). Incorporating these projections into the above projection family, we have the following set of projections:
\begin{align*}
\tilde{\Pi} = \{\pi_{A \to A'}: \Omega_A \to \Omega_{A'} \},
\end{align*} 
where $A,A' \in \mathcal{A}$, and $A' \subseteq A$. 

\begin{example}[Vectors]
A simple example of projections to substructures is the familiar coordinate projections used in modeling multivariate random variables. Let $\Omega$ be a product space of the form $\Omega = \mathcal{X} = \mathcal{X}_1 \times \cdots \times \mathcal{X}_n$, and let the coordinate projection $\pi_A:\mathcal{X} \to \mathcal{X}_A$ be
\begin{align*}
\pi_A(X)=X_A,
\end{align*}
where $\mathcal{X}_A = \prod_{i \in A} \mathcal{X}_i$ and $A \subseteq \{1,\ldots,n\}$. Since $\mathcal{X}_A \not \subseteq \mathcal{X}$ for $A \neq \{1,\ldots,n\}$, these projections are to substructures of the original space. Let $\Pi$ be the projection family:
\begin{align*}
\Pi = \{\pi_A, \; A \in \mathcal{A} \}
\end{align*}
where $\mathcal{A} = \mathbb{P}(\{1,\ldots,n\})$. Then the projection family $\Pi$ is complete (since the index set $\mathcal{A} = \mathbb{P}(\{1,\ldots,n\})$ is closed under intersection), and is also consistent (since for any $A,A' \in \mathcal{A}$ such that $A' \subseteq A$, there exists a projection between substructures of the form $\pi_{A \to A'}: \mathcal{X}_A \to \mathcal{X}_{A'}$ such that $\pi_{A \to A'} \circ \pi_A = \pi_{A'}$).
\end{example}

Another example is the projection to substructures used in modeling random trees (Section \ref{sec:substructures_trees}). For a sample space $\Omega$ with a distribution $P$ over it, if we have a finite, consistent, complete system of substructure projections $\Pi$ on $\Omega$, then we can define a marginal random variable for each index $A \in \mathcal{A}$ in this family, and similarly, we may also define a Gibbs form (equation \ref{eq:factorization_general}). Hence, we have arrived at a general framework, based on general projections. 

%

\section{Compositional Systems}

A projection family on an object may be viewed as defining a compositional system. Compositionality refers to the phenomena in which objects are composed of parts, which in turn, are themselves composed of parts, etc., and that the same part can occur in multiple larger parts. For a given set of objects $\Omega$, a projection family on it defines the decomposition of objects into a hierarchy of parts, and this may be viewed as a top-down approach to defining a compositional system. 
This approach for defining these systems differs from that taken in \citep{geman2002composition}; in that work, given a set of primitive parts $T \subseteq \Omega$, a set of composition rules are used to define the allowable groupings of parts into larger parts, and may be viewed as a bottom-up approach to defining compositional systems. The alternative perspective offered here on these systems is very different than that taken in the literature; our intention is only to provide a context to how graphical models, as formulated above, fits among other general frameworks, and this is only one possible interpretation of the relationship. 


To illustrate modeling a compositional system, consider character recognition, a classic problem in the field of computer vision. The goal is to design a computer vision system that takes images of handwritten characters and determines the character being displayed. Let the object space $\Omega$ be the space of possible binary (i.e., black and white) images with a label (i.e., we assume every image has a label attached to it from some label set). Following along the lines of the example given in \citep{geman2002composition}, the most primitive parts may be images with only a single black point, having the label `point'; the next simplest parts might be images with only two points within close proximity, having the label `linelet'; these objects can be combined to form objects with the label `line', which in turn can be combined to form objects with the label 'L-junction', and so on, until finally objects with the label `character' are formed.

If instead of defining a part $t \in \Omega$ as a single element, we let it be a random variable $\bold{t}$ that takes values in a subset of this space, 
then we may associate them to projections (taking images and their labels to a subset of the images and their labels). These projections are consistent with each other, and so define a consistent set of marginal distributions over these parts. In turn, these marginal variables allow a graphical model approach to be applied to the problem, allowing the efficient estimation of distributions over the object space $\Omega$. 

We note that in the approach to compositional systems based on projection families, the composition rules (describing how to combine parts into larger parts) are provided, in a sense, by the (marginal) probabilities of parts having some value, conditioned on the value of its constituent parts (the value of its projections). When this probability is nonzero, one may interpret that a composition rule is dictating that these constituent parts are combinable. 

\chapter{Examples}
\label{sec:examples_sec}


In this section, we consider the practical application of the models described in the previous sections. Since the random graph models are more general than the random tree ones, and because they differ more from the models in the literature, we focus our attention on them here. 
We will use factorization to specify structure in distributions (which, for graphs, differs from specifying independence assumptions, see Section \ref{sec:compact_dists}); the reason is, for the examples considered here, this invariance is more straightforward to specify and operate on. We also discuss invariances on distributions based on graph isomorphisms, an assumption used in many random graph models. The use of these invariances on unattributed graphs, however, causes models to be susceptible to degeneracy problems. To avoid this issue, it is important for models that employ these invariances to assign latent variables to the vertices, or equivalently, to use attributed graphs. We will assume models that take a simple exponential form based on the use of template graphs. We illustrate the ideas with several examples. 

\section{Compact Distributions}
\label{sec:chap2_apps_compact_dists}

Although a distribution over a finite graph space $\mathcal{G}$ can always be specified by directly assigning a probability to each graph in it, in practise we need to make assumptions about the distribution. In Section \ref{sec:compact_dists}, we discussed Gibbs form and the specification of structure based on factorization, where a Gibbs distribution has the form:
 \begin{align*}
P(G) = \exp \left[ \psi_0 + \sum \limits_{G' \in S_1(G)} \psi_1(G') +  \sum \limits_{G' \in S_2(G)} \psi_2(G') + \ldots \right].
\end{align*}
In the examples considered here, we find it natural to allow slightly more structure than can be obtained only through the specification of factors; we also want to be able to assign individual graphs to have a factor value of zero. In other words, we are interested in defining structure through the specification of a small subset $\mathcal{G}_{\text{basis}} \subset \mathcal{G}$ such that by assigning a potential value to each graph in $\mathcal{G}_{\text{basis}}$, the probability of every graph in $\mathcal{G}$ can be determined. Hence, given a basis, we assume the potential of any graph $G \notin \mathcal{G}_{\text{basis}}$ is zero, and 
define the probability of a graph as
\begin{align}
P(G) & = \frac{1}{Z} \exp \left [ \sum \limits_{G' \in \mathcal{C}(G)} \psi(G') \right ],
\label{eqn:model}
\end{align}
where $\psi: \mathcal{G}_{\text{basis}} \rightarrow \mathbb{R} \cup \{-\infty\}$ and $\mathcal{C}(G) \equiv S(G) \cap \mathcal{G}_{\text{basis}}$. 

\section{Additional Structure}

The model given in equation \ref{eqn:model} can be further simplified by assuming the function $\psi$ has some structure. This can be done in many ways; the simplest is to assign the same function value to graphs that are similar in some sense. For example, we might want graphs that are isomorphic to each other to have equal values (i.e., setting $\psi(G_1) = \psi(G_2)$ for all $G_1,G_2 \in \mathcal{G}_{\text{basis}}$ that are isomorphic). 
More generally, we can specify structure in $\psi$ by assuming an additive relationship of the form:
\begin{align*}
\psi(G) = \sum_{k=1}^K \lambda_k I_{\{ G \in \mathcal{D}_k \} },
\end{align*}
where each $\mathcal{D}_k \subset \mathcal{G}_{\text{basis}}$ is a subset of the basis and each $\lambda_k$ a real number. Then the model in equation \ref{eqn:model} simplifies to:
\begin{align}
P(G) & = \frac{1}{Z} \exp \left [ \sum \limits_{k=1}^K \lambda_k U_k(G) \right ],
\label{eqn:template_model}
\end{align}
where $U_k(G) = \# \{ G' \in S(G) \; : \; G' \in \mathcal{D}_k \}$ is the number of subgraphs of type $k$ in the graph $G$. We will find it convenient to reformulate each set $\mathcal{D}_k$ as a binary function: define a function $R_k: \mathcal{G} \to \{0,1\}$ such that
\begin{align*}
R_k(G)=1 \iff G \in \mathcal{D}_k.
\end{align*}
Then, equivalently, we have that $U_k(G) = \# \{ G' \in S(G) \; : \; R_k(G') = 1 \}$. We refer to the binary functions $R_k$ as \textit{compatibility maps}. We now consider methods for specifying these maps.




\section{Graph Isomorphisms}
\label{sec:chap2_isomorphisms}


An important way to compare two graphs is based on how their parts compare. In this section, we consider isomorphisms, a comparison method based on second-order subgraphs; two graphs are said to be isomorphic if they share the same edge structure:

\begin{definition}[Graph Isomorphism]
\label{def:graph_isomorphism_1}
A graph $G = (V,E)$ is \textit{isomorphic} to a graph $G' = (V',E')$ if there exists a bijection $f: V \rightarrow V'$ such that:
\begin{align*}
E(v,v') = E'(f(v), f(v')) \text{ for all } v,v' \in V.
\end{align*}
Two graphs that are isomorphic is denoted by $G \simeq G'$.
\end{definition}

A distribution $P$ over a graph space $\mathcal{G}$ is said to be invariant to isomorphisms if any two graphs that are isomorphic have the same probability, i.e.:
\begin{align*}
G \simeq G' \implies P(G) = P(G').
\end{align*}
where $G,G' \in \mathcal{G}$. We now consider some isomorphism variations that will be useful for attributed graphs.

\section{Attributed Graph Isomorphisms}
\label{sec:attributed_graph_isos}

When modeling unattributed graphs, often it is important to associate attributes to vertices in these graphs. The attributes, in this case, may be thought of as latent variables, which can simplify the order of models. 
Suppose we have a finite vertex space $\Lambda_V$, an edge space $\Lambda_E$, and an attribute space $\mathcal{X} $. Recall from Section \ref{sec:graph_variations}, an attributed graph has the form $G = (V,X,E)$, where:
\begin{align*}
& V \subseteq \Lambda_V \\
& X: V \rightarrow \mathcal{X} \\
& E: V \times V \rightarrow \Lambda_E.
\end{align*}
The simplest isomorphism for attributed graphs is based on the edge structure and attributes on individual vertices:





\begin{definition}[First-order Isomorphism]
\label{def:graph_isomorphism_2}
A graph $G=(V,X,E)$ is \textit{first-order isomorphic} to a graph $G'=(V',X',E')$ if there exists a bijection $f: V \rightarrow V'$ such that:
\begin{enumerate}
\item $X(v) = X'(f(v)) \text{ for all } v \in V.$
\item $E(v,v') = E'(f(v),f(v')) \text{ for all } v,v' \in V$ 
\end{enumerate}
Two attributed graphs that are first-order isomorphic is denoted by $G \simeq_1 G'$.
\end{definition}


This definition is a natural extension of Definition \ref{def:graph_isomorphism_1} to attributed graphs. As an example, suppose the attribute space $\mathcal{X} = \{c_1,\ldots,c_k\}$ is some finite set of labels or colors; then for graphs to be isomorphic by this definition, the coloring of vertices must be respected in addition to the edge structure. The next simplest isomorphism for attributed graphs is based on the attributes on pairs of vertices. Suppose we have a distance function $d$ over the attribute space $\mathcal{X}$.

\begin{definition}[Second-order Isomorphism]
\label{def:graph_isomorphism_3}
A graph $G=(V,X,E)$ is \textit{second-order isomorphic} to a graph $G'=(V',X',E')$ if there exists a bijection $f: V \rightarrow V'$ such that:
\begin{enumerate}
\item $d(X(v),X(v')) = d(X'(f(v)),X'(f(v'))) \text{ for all } v,v' \in V.$
\item $E(v,v') = E'(f(v),f(v')) \text{ for all } v,v' \in V$ 
\end{enumerate}
Two attributed graphs that are second-order isomorphic is denoted by $G \simeq_2 G'$.
\end{definition}



This second-order isomorphism is used in many latent position models (\citep{hoff2002latent}), where $\mathcal{X}=\mathbb{R}^d$ is a Euclidean space; for models using this isomorphism invariance, the probability of a graph depends on the distances between vertices in it, not on their particular locations. These definitions can be extended to higher-orders in a straightforward manner. To summarize, we presented some isomorphisms that can be used in specifying when graphs are similar to each other. We will make use of them to specify compatibility maps in the examples presented later in this section.

 \section{Master Interaction Function}
 
 In defining distributions over a graph spaces, often it will be useful to reduce the size of the graph space, removing graphs that have zero probability. One way to do this, assuming that the edge space $\Lambda_E$ has a partial ordering $\leq$, is to define a function that restricts the edge configurations allowed in graphs: 
 
 \begin{definition}[Master Interaction Functions] $\;$
 \begin{enumerate}
 \item A \textit{master interactions function} over vertices is a function of the form $F_V: \Lambda_V^2 \rightarrow \Lambda_E$. A graph $G=(V,X,E)$ is said to \textit{respect} a master interactions function $F_V$ if, for all $v,v' \in V$, we have $E(v,v') \leq F_V(v,v')$.
 \item A master interactions function over attributes is a function of the form $F_{X}:\mathcal{X}^2 \to \Lambda_E$. A graph $G=(V,X,E)$ is said to \textit{respect} a master interactions function $F_X$ if, for all $v,v' \in V$, we have $E(v,v') \leq F_X(X(v),X(v'))$.
 \end{enumerate}
 \end{definition}
 
We use master interactions functions to restrict graph spaces to only those graphs that respect them. That is, for a graph space $\mathcal{G}$ and some master interaction functions $F_V$ and $F_X$, we can restrict the graphs to the set:
  \begin{align*}
\mathcal{G}' = 
 \left\{ (V,X,E) \in \mathcal{G} \; : \; 
 \begin{array}{l}
     E(v,v') \leq F_V(v,v') \text{ for all } v,v' \in V \\
     E(v,v') \leq F_X(X(v),X(v')) \text{ for all } v,v' \in V
 \end{array} 
  \right\}. 
 \end{align*}
 
 \begin{example}
Suppose the vertex space $\Lambda_V = \{1,\ldots,p\}$ and the edge space $\Lambda_E = \{0,1\}$. We can define a master interactions function $F_V$ that ensures there is no edge between vertices that are farther apart than $t \in \mathbb{R}^+$as follows:
\begin{align*}
F_V(v,v') =
\left \{
\begin{array}{l}
     0, \quad \text{if } |v-v'| > t \\
     1, \quad \text{otherwise}
 \end{array}
 \right .
\end{align*}
\end{example}

\begin{example}
Suppose the attribute space is $\mathcal{X} = \{c_1,\ldots,c_k\}$, where each $c_i$ represents a color and edge space is $\Lambda_E = \{0,1 \}$. We can define a master interactions function $F_X$ that ensures vertices with the same color attribute cannot have an edge:
\begin{align*}
F_X(c_i,c_j) = 
\left \{
\begin{array}{l}
     0, \quad \text{if } c_i=c_j \\
     1, \quad \text{otherwise}
 \end{array}
 \right .
\end{align*}
\end{example}

 \section{Examples}
 
 In this section, we illustrate the above ideas with some examples. In each example, the model takes the form of equation \ref{eqn:template_model}, and uses some set of templates $\{T_1, \ldots, T_K\}$. For each template $T_k$, the compatibility map $R_k:\mathcal{G} \to \{0,1\}$ is based on if a graph $G \in \mathcal{G}$ is isomorphic to it, i.e.: 
\begin{align}
\label{eq:compatibility_map_iso_def}
R_k(G) = 
\left \{
\begin{array}{l}
     1, \quad \text{if } G \simeq T_k \\
     0, \quad \text{otherwise}
 \end{array}.
 \right .
\end{align}
In all the examples except the first one, we assume the isomorphism used is the first-order isomorphism. The sampling and learning algorithms are discussed in Section \ref{sec:inference_and_learning}. 


 \subsection{Example 1: Grid Graphs}
 
 We consider unattributed grid-like graphs such as the one shown in Figure \ref{fig:grider}. Let the vertex space be $\Lambda_V = \{1,\ldots,p\}^2$ be a grid of size $p$, and let $\Lambda_E = \{0,1 \}$, specifying the absence of an edge or the presence of an edge, respectively. We can specify the master interactions function $F_V$ to take pairs of vertices that cannot have an edge to the value $0$, and pairs that can have an edge to the value 1. Define $F_V$ as follows:
\begin{align*}
F_V(v,v') =
\left \{
\begin{array}{l}
     1, \quad \text{if } |v_1-v_1'| \leq 1 \text{ and } |v_2-v_2'| \leq 1 \\
     0, \quad \text{otherwise}
 \end{array}
 \right .,
\end{align*}
where $v=(v_1,v_2) \in \Lambda_V$ and $v'=(v_1',v_2') \in \Lambda_V$. Hence, this master interactions function $F_V$ ensures the graph space $\mathcal{G}$ only contains grid-like graphs.

\begin{figure}[h!]
  \centering
    \includegraphics[scale=0.35]{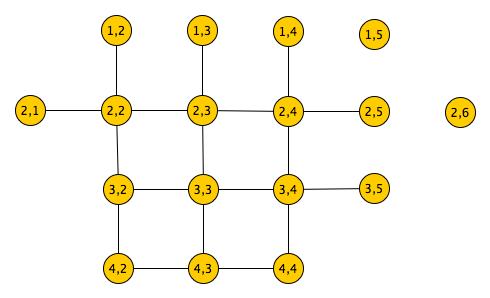}
  \caption{An example of a grid-like graph.}
  \label{fig:grider}
\end{figure}


A possible set of templates is shown in Table \ref{tab:gtk}. Each template $T_k$ in these tables specifies a compatibility map based on graphs that are isomorphic to $T_k$. Here, we made the following design choices. First, we have limited the order of the template graphs to fourth order and lower (i.e. graphs such that $|V(G)| \leq 4$). Secondly, to make computation feasible, we apply a `locality' principle in which only connected graphs are used as templates. 
Since unconnected graphs constitute the vast majority of the subgraphs in $S(G)$ for any given graph $G \in \mathcal{G}$, the restriction to only these is necessary for computational reasons. For example, consider the second-order subgraphs in the graph in Figure \ref{fig:grider}; there are $|S_2(G)| = {17 \choose 2} = 136$ subgraphs of this order, but only $18$ of them are connected. If we consider higher-order subgraphs, this gap widens.




Given these templates, the number of subgraphs that correspond to a given pattern can be calculated for any graph $G \in \mathcal{G}$, and hence its probability can be calculated. For example, for the graph $G$ in Figure \ref{fig:grider}, the probability is expressed as follows:
\begin{align*}
P(G) & = \frac{1}{Z} \exp \left [ \sum \limits_{k=1}^K \lambda_k U_k(G) \right ] \\
	& = \frac{1}{Z} \exp \left [ 17\lambda_1 + 18 \lambda_2 + 26 \lambda_3 + 4\lambda_4 + 45 \lambda_5 + 20 \lambda_6 \right ].
\end{align*}

\begin{table}[H]
\centering
\begin{tabular}{*{2}{m{0.2\textwidth}}}
\hline
\includegraphics[scale=0.27]{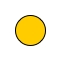} &  $ \lambda_1$ \\
\hline
\includegraphics[scale=0.27]{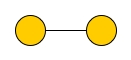} & $ \lambda_2$ \\
\hline
\includegraphics[scale=0.27]{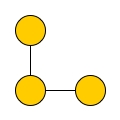} & $ \lambda_3$  \\
\hline
\includegraphics[scale=0.27]{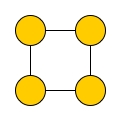} & $ \lambda_4$  \\
\hline
\includegraphics[scale=0.27]{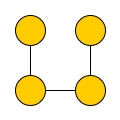} & $ \lambda_5$  \\
\hline
\includegraphics[scale=0.27]{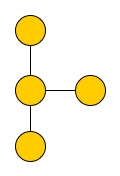} & $ \lambda_6$  \\
\hline
\end{tabular}
\caption{The set of connected graphs that are used as templates; the compatibility maps are based on graphs that are isomorphic to these templates.}
\label{tab:gtk}
\end{table}


\newpage
 \subsection{Example 2: `Molecule' Graphs}
 \label{sec:example2_molecule_graphs}
 
 We consider an example in which the graph space $\mathcal{G}$ is composed of graphs that loosely resemble molecules in appearance. An example is shown in Figure \ref{fig:mol}.

\begin{figure}[H]
  \centering
    \includegraphics[scale=0.27]{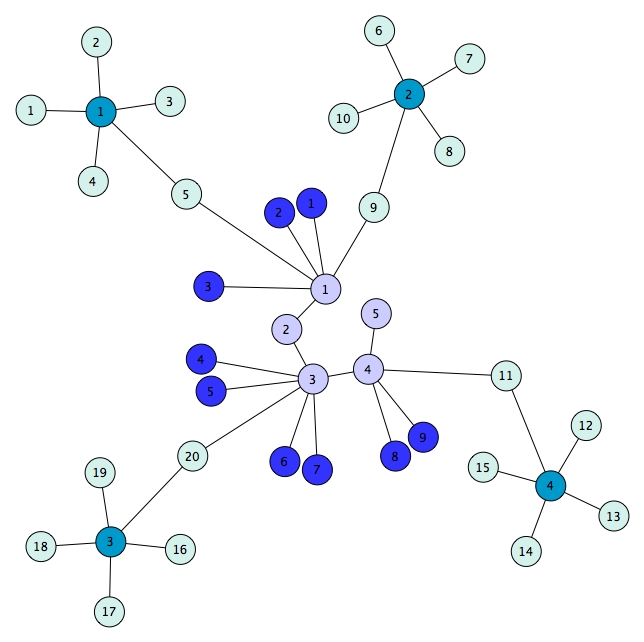}
  \caption{An example of a `molecule' graph. This is an artificial graph, made only for illustration.}
  \label{fig:mol}
\end{figure}

In this example, we will use attributed graphs of the form $G = (V,X,E)$. Let $\Lambda_V = \{1,\ldots,p\}$ be the vertex space, $\mathcal{X} = \{c_1,c_2,c_3,c_4\}$ the attribute space, where each $c_i$ represents a color, and $\Lambda_E = \{0,1 \}$ the edge space. 
We can specify the master interactions function $F_X:\mathcal{X}^2 \to \Lambda_E$ to specify that vertices with the same color cannot have an edge between them (e.g., set $F_X(c_i,c_j) = 0$ if $c_i=c_j$). Similarly, we might want to specify that vertices with certain different colors can have an edge between them (e.g. set $F_X(c_i,c_j) = 1$ for some $c_i \neq c_j$). 

A possible set of templates and their corresponding parameters is shown in Table \ref{tab:gt3}. For each template graph $T_k$, we define a compatibility map $R_k: \mathcal{G} \rightarrow \{0,1\}$ based on graphs that are second-order isomorphic to it (equation \ref{eq:compatibility_map_iso_def}). Given these templates, the number of subgraphs that correspond to a given pattern can be calculated for any graph $G \in \mathcal{G}$, and hence its (unnormalized) probability can be calculated. For the graph $G$ in Figure \ref{fig:mol}, the probability can be expressed as follows:
\begin{align*}
P(G) & = \frac{1}{Z} \exp \left [ \sum \limits_{k=1}^K \lambda_k U_k(G) \right ] \\
	& = \frac{1}{Z} \exp \left [ 20 \lambda_{1} + 4 \lambda_{2} + 9 \lambda_{3} + 5 \lambda_{4} + 4 \lambda_{5} + 20 \lambda_{6} +  4 \lambda_{7} + 9 \lambda_{8} +  6 \lambda_{9} + 4 \lambda_{10} \right ].
\end{align*}
Notice that in this example, the attributes (i.e., the colors associated with vertices) allow distributions in which, loosely speaking, typical samples have complex structure even despite the fact that the basis does not contain high-order graphs. For example, the edge structure in these graphs are very unlikely to have been generated by independent coin flips as in an Erd\H os-R\'enyi model. 
If the vertices did not have these attributes and we wanted to define a distribution that has equivalent probabilities as in this example, (e.g. assign the same probability to the unattributed version\footnote{That is, for an attributed graph $G=(V,X,E)$, removing the attribute function $X:V^2 \to \mathcal{X}$ from it to form the unattributed graph $G=(V,E)$.} of the graph in Figure \ref{fig:mol}), it would be necessary for any basis to contain graphs of much higher orders than those in the basis used in this example. Hence, we see that attributes are important latent variables even if one only wants to define distributions over unattributed graph spaces. Thus, ideas contained in latent position models (\citep{hoff2002latent}) and latent stochastic blockmodels (\citep{airoldi2009mixed}, \citep{latouche2011overlapping}) can be incorporated within the framework here.


\begin{table}[H]
\centering
\begin{tabular}{*{2}{m{0.2\textwidth}}}
\hline
\includegraphics[scale=0.27]{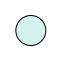} &  $\lambda_{1}$ \\
\hline
\includegraphics[scale=0.27]{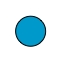} & $ \lambda_{2}$ \\
\hline
\includegraphics[scale=0.27]{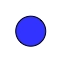} & $ \lambda_{3}$  \\
\hline
\includegraphics[scale=0.27]{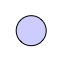} & $ \lambda_{4}$  \\
\hline
\includegraphics[scale=0.27]{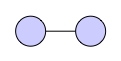} & $ \lambda_{5}$  \\
\hline
\includegraphics[scale=0.27]{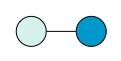} & $ \lambda_{6}$  \\
\hline
\includegraphics[scale=0.27]{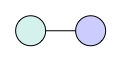} & $ \lambda_{7}$  \\
\hline
\includegraphics[scale=0.27]{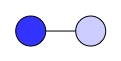} & $ \lambda_{8}$  \\
\hline
\includegraphics[scale=0.27]{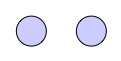} & $ \lambda_{9}$  \\
\hline
\includegraphics[scale=0.27]{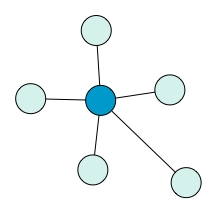} & $ \lambda_{10}$ \\
\hline
\end{tabular}
\caption{The set of graphs that are used as templates. These are used to specify the compatibility maps based on graphs that are isomorphic to them.}
\label{tab:gt3}
\end{table}




\newpage
\subsection{Example 3: Mouse Visual Cortex}


A graph $G_0$ that corresponds to the visual cortex of a mouse and is shown in Figure \ref{fig:mouse}. 
 
\begin{figure}[H]
  \centering
    \includegraphics[scale=0.24]{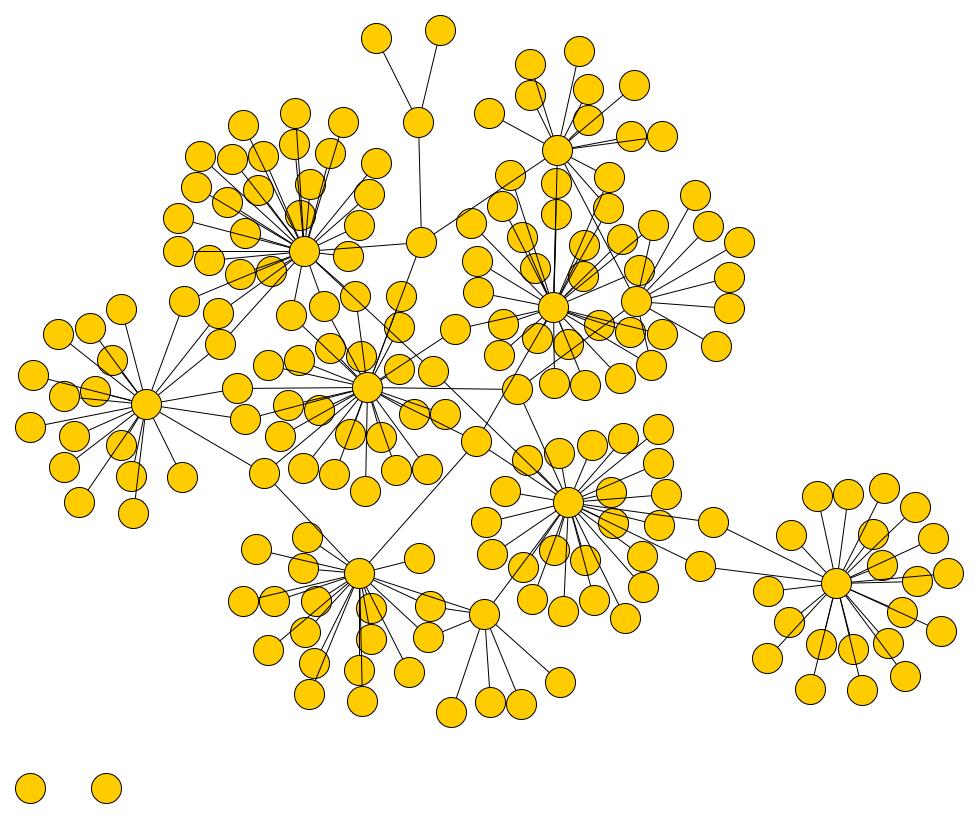}
  \caption{A mouse visual cortex, which we label as $G_0$ \citep{bock2011network}. }
  \label{fig:mouse}
\end{figure}


In order to model mouse visual cortexes, we consider a hierarchical model; we begin by modeling parts (i.e., interesting subgraphs) that compose these visual cortexes. From the graph $G_0$, we extract subgraphs of $G_0$ that exist in the following graph space:
\begin{align*}
\mathcal{G} = \left \{ G \in S(G_0) \; : \; 
\begin{array}{l}
     |V| = 25 \\ 
     G \text{ is connected} \\
     \text{diameter}(G) \leq 5
\end{array}
\right \},
\end{align*}
where $S(G_0)$ denotes the set of all subgraphs of $G_0$. Some examples of graph in $\mathcal{G}$ are shown in Figure \ref{fig:mouse2}.


\begin{figure}[H] 
  \begin{minipage}[b]{0.5\linewidth}
    \includegraphics[scale=0.23]{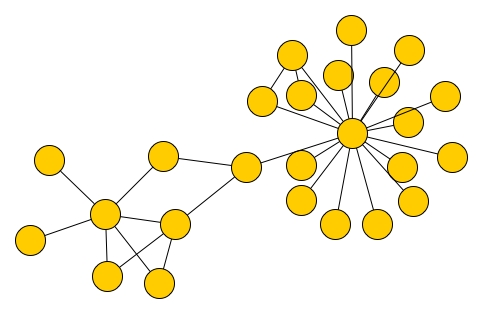} 
  \end{minipage} 
  \begin{minipage}[b]{0.5\linewidth}
    \includegraphics[scale=0.23]{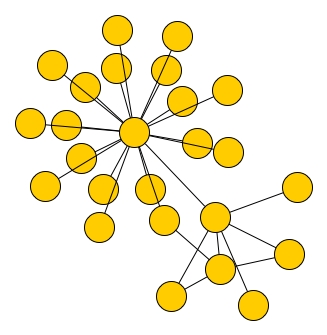} 
  \end{minipage} 
    \begin{minipage}[b]{0.5\linewidth}
    \includegraphics[scale=0.23]{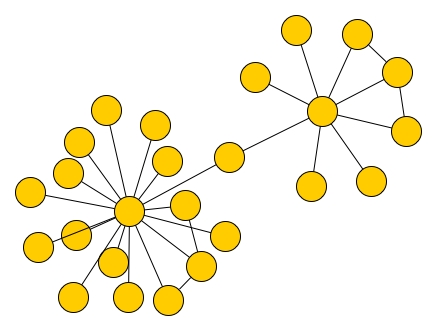} 
  \end{minipage} 
  \begin{minipage}[b]{0.5\linewidth}
    \includegraphics[scale=0.23]{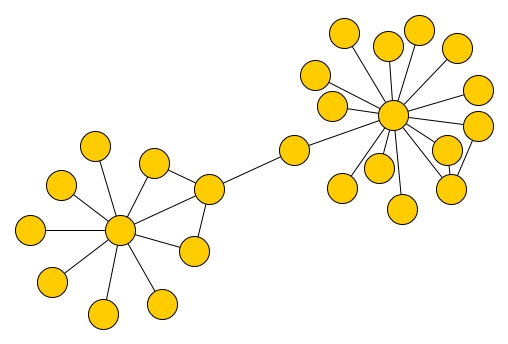} 
  \end{minipage} 
      \begin{minipage}[b]{0.5\linewidth}
    \includegraphics[scale=0.23]{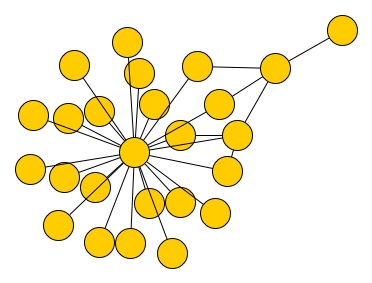} 
  \end{minipage} 
  \begin{minipage}[b]{0.5\linewidth}
    \includegraphics[scale=0.23]{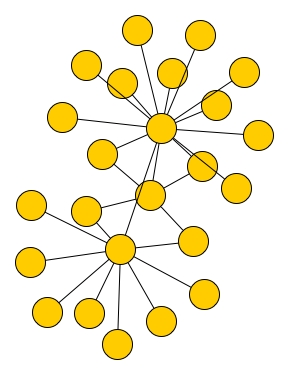} 
  \end{minipage} 
  \caption{Some examples of subgraphs of $G_0$ used for learning our subgraph model.}
  \label{fig:mouse2}
\end{figure}


Notice that the graphs in $\mathcal{G}$ do not have attributes, i.e. the vertex and edge spaces have the form:
\begin{align*}
 \Lambda_V & = \{1,\ldots,25\} \\
 \Lambda_E & = \{0,1 \},
\end{align*}
For modeling purposes, we introduce attributes for the vertices, which serve as latent variables; without them, very high-order templates would be necessary. 
We assign attributes to the vertices based on their edge counts, which can be used to define different types of vertices. Let $d_G(v)$ denote the number of edges incident on vertex $v$ in the graph $G$. We will assign each vertex in a graph an attribute value in $\mathcal{X} = \{c_1, c_2, c_3\} $ based on its edge count as follows. For a graph $G=(V,E)$, define the attribute function $X:V \to \mathcal{X}$ by:
\begin{align*}
X(v) = \left \{
\begin{array}{l}
    c_1, \text{ if } 0 \leq d_G(v) \leq 1 \\ 
    c_2, \text{ if } 2 \leq d_G(v) \leq 4 \\
     c_3, \text{ if } 5 \leq d_G(v)
\end{array},
\right.
\end{align*}
and augment the graph $G$ with it to form the attributed graph $G'=(V,X,E)$. For an example of this augmentation, see Figure \ref{fig:mouse3}. We note a couple of things: (1) more complicated attributes can be used for the vertices, for example using a generalized notion of edge counts; (2) a clustering algorithm can be used here, for example clustering similar subgraphs.

\begin{figure}[H]
\begin{minipage}[c]{0.45\linewidth}
    \includegraphics[scale=0.25]{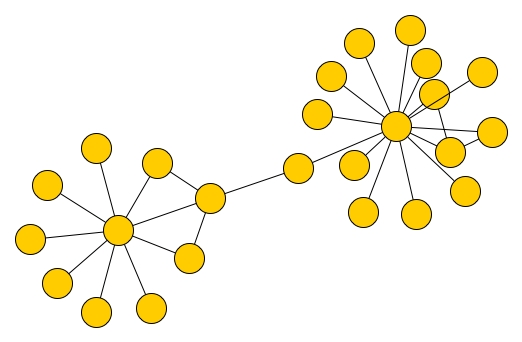} 
  \end{minipage} 
    $\Longrightarrow \; $ 
  \begin{minipage}[c]{0.45\linewidth}
    \includegraphics[scale=0.25]{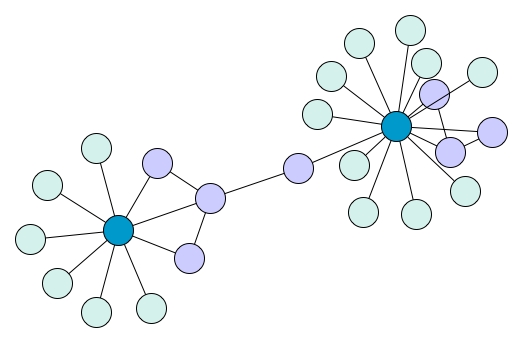} 
  \end{minipage} 
  \caption{An example of a graph being assigned color attributes.} 
  \label{fig:mouse3}
\end{figure}


The templates used for our model are shown in Tables \ref{tab:gt4},  \ref{tab:gt5}, and  \ref{tab:gt6}. We learn model parameters using the algorithm in Section \ref{sec:learning}, and a small training set of $31$ graphs (examples shown in Figure \ref{fig:mouse2}). Some samples from the model are shown in Figure \ref{fig:mouse4}. As mentioned, these are only samples of subgraphs. With only one full graph $G_0$, we cannot develop a hierarchical model, so instead, we combine these subgraphs just using a few simple rules (e.g., combining them by randomly placing edges between them). An example is shown in Figure \ref{fig:mouse5}. 

\begin{table}[p]
\centering
\begin{tabular}{*{2}{m{0.2\textwidth}}}
\hline
\includegraphics[scale=0.25]{1_1.jpg} &  $\lambda_{1}$ \\ 
\hline
\includegraphics[scale=0.25]{1_4.jpg} & $ \lambda_{2}$ \\ 
\hline
\includegraphics[scale=0.25]{1_2.jpg} & $ \lambda_{3} $ \\ 
\end{tabular}
\caption{The set of 1st-order graphs that are used as templates.}
\label{tab:gt4}
\end{table}

\begin{table}[p]
\centering
\begin{tabular}{*{2}{m{0.2\textwidth}}}
\hline
\includegraphics[scale=0.25]{2_5.jpg} & $ \lambda_{4}$ \\ 
\hline
\includegraphics[scale=0.25]{2_1.jpg} & $ \lambda_{5} $ \\ 
\hline
\includegraphics[scale=0.25]{2_7.jpg} & $ \lambda_{6} $ \\ 
\hline
\includegraphics[scale=0.25]{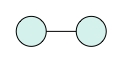} &  $\lambda_{7}$ \\ 
\hline
\includegraphics[scale=0.25]{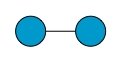} & $\lambda_{8} $ \\ 
\hline
\includegraphics[scale=0.25]{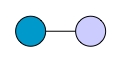} & $\lambda_{9} $ \\ 
\hline
\end{tabular}
\caption{The set of 2nd-order graphs that are used as templates.}
\label{tab:gt5}
\end{table}

\begin{table}[p]
\centering
\begin{tabular}{*{2}{m{0.2\textwidth}}}
\hline
\includegraphics[scale=0.25]{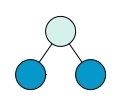} &  $\lambda_{10}  $ \\ 
\hline
\includegraphics[scale=0.25]{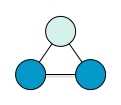} &  $\lambda_{11} $ \\ 
\hline
\end{tabular}
\caption{The set of 3rd-order graphs that are used as templates.}
\label{tab:gt6}
\end{table}

\begin{figure}[H] 
  \begin{minipage}[b]{0.5\linewidth}
    \includegraphics[scale=0.21]{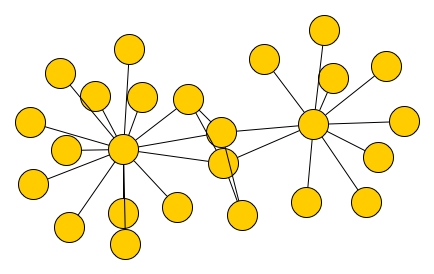} 
  \end{minipage} 
  \begin{minipage}[b]{0.5\linewidth}
    \includegraphics[scale=0.21]{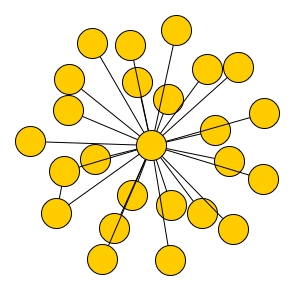} 
  \end{minipage} 
    \begin{minipage}[b]{0.5\linewidth}
    \includegraphics[scale=0.21]{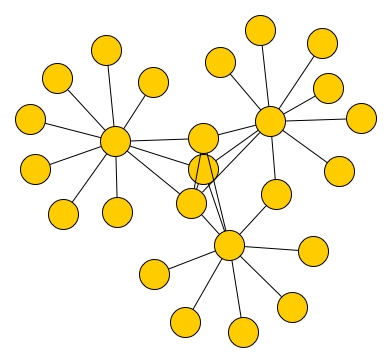} 
  \end{minipage} 
  \begin{minipage}[b]{0.5\linewidth}
    \includegraphics[scale=0.21]{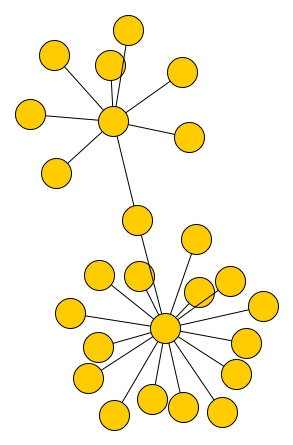} 
  \end{minipage} 
    \begin{minipage}[b]{0.5\linewidth}
    \includegraphics[scale=0.21]{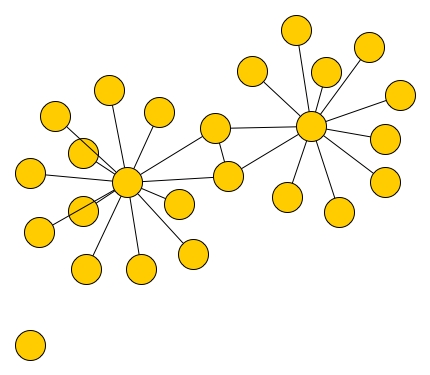} 
  \end{minipage} 
  \begin{minipage}[b]{0.5\linewidth}
    \includegraphics[scale=0.21]{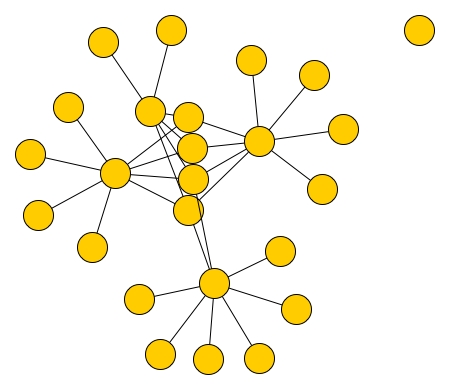} 
  \end{minipage} 
  \caption{Model samples.}
  \label{fig:mouse4}
\end{figure}

\begin{figure}[H]
  \centering
    \includegraphics[scale=0.21]{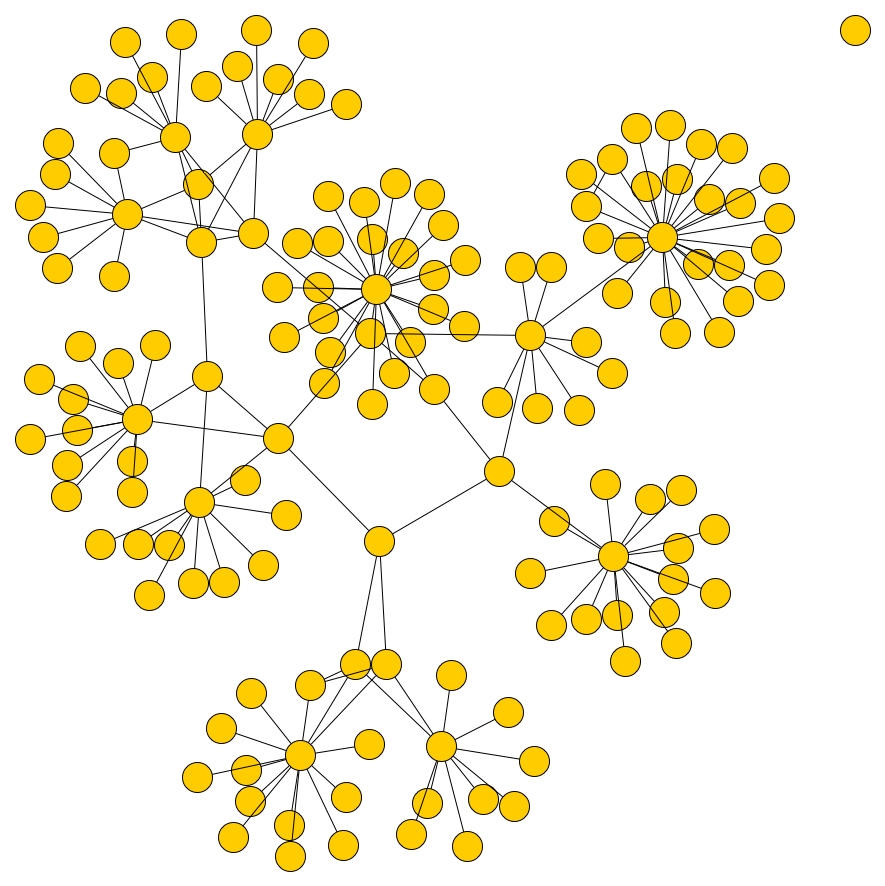}
  \caption{Mouse visual cortex sample from our model.}
  \label{fig:mouse5}
\end{figure}

\subsection{Example 4: Chemistry Data}


In the chemoinformatics dataset MUTAG \citep{shervashidze2011weisfeiler}, there are 188 mutagenic aromatic and heteroaromatic
nitro compounds. Examples are shown in Figure \ref{fig:chemo}. 

We will form a simple hierarchical model using deterministic subgraphs. Define a set of subgraphs as shown in Table \ref{tab:gt7}. These subgraphs will correspond to vertices in the second level of the model. There will be an edge between two vertices in the second level if there is an intersection between two subgraphs. For example, let $G$ be a molecule graph and suppose $G_1,G_2 \subset G$ are two subgraphs of $G$ and are in Table \ref{tab:gt7} below. If $G_1,G_2$ have a common subgraph (i.e. have at least one common vertex in $G$), then in the second level of the model, there will be an edge between the corresponding vertices. This edge will have an attribute specifying the degree of intersection, i.e. how many common vertices the two subgraphs share. See Figure \ref{fig:chemo2} for some examples. Thus, all the randomness in the problem is at this second level in the hierarchy and a model can be applied over it.


\begin{figure}[H]
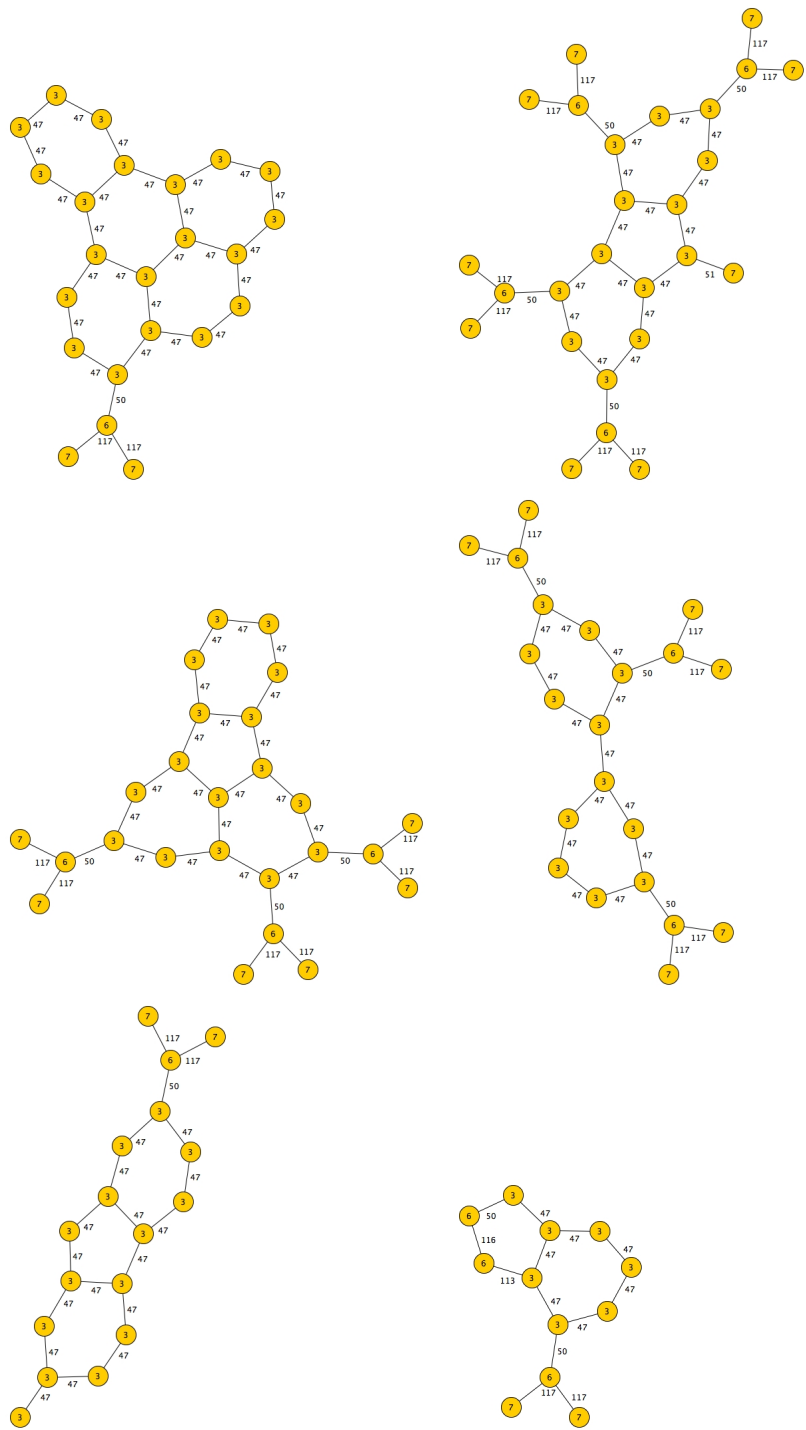
 
  \begin{minipage}[b]{0.5\linewidth}
    \includegraphics[scale=0.25]{g1.jpg} 
  \end{minipage} 
  \begin{minipage}[b]{0.5\linewidth}
    \includegraphics[scale=0.25]{g2.jpg} 
  \end{minipage} 
    \begin{minipage}[b]{0.5\linewidth}
    \includegraphics[scale=0.25]{g7.jpg} 
  \end{minipage} 
  \begin{minipage}[b]{0.5\linewidth}
    \includegraphics[scale=0.25]{g13.jpg} 
  \end{minipage} 
      \begin{minipage}[b]{0.5\linewidth}
    \includegraphics[scale=0.25]{g10.jpg} 
  \end{minipage} 
  \begin{minipage}[b]{0.5\linewidth}
    \includegraphics[scale=0.25]{g32.jpg} 
  \end{minipage} 
  \caption{Examples of molecule graphs in the MUTAG dataset.}
  \label{fig:chemo}
\end{figure}


\begin{table}[H]
\centering
\begin{tabular}{*{2}{m{0.2\textwidth}}}
\hline
\includegraphics[scale=0.25]{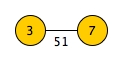} &  $\Longleftrightarrow$ \includegraphics[scale=0.25]{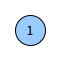} \\
\hline
\includegraphics[scale=0.25]{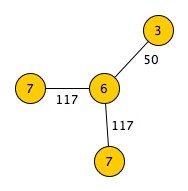} & $\Longleftrightarrow$ \includegraphics[scale=0.25]{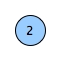}  \\
\hline
\includegraphics[scale=0.25]{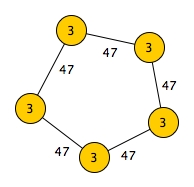} & $\Longleftrightarrow$ \includegraphics[scale=0.25]{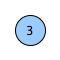} \\
\hline
\includegraphics[scale=0.25]{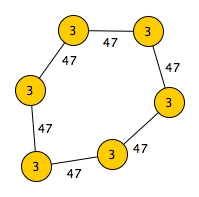} & $\Longleftrightarrow$ \includegraphics[scale=0.25]{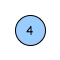} \\
\hline
\includegraphics[scale=0.25]{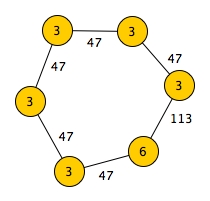} & $\Longleftrightarrow$ \includegraphics[scale=0.25]{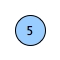} \\
\hline
\includegraphics[scale=0.25]{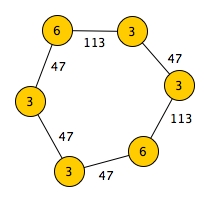} & $\Longleftrightarrow$ \includegraphics[scale=0.25]{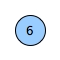} \\
\hline
\includegraphics[scale=0.25]{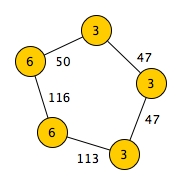} & $\Longleftrightarrow$ \includegraphics[scale=0.25]{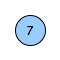} \\
\hline
\includegraphics[scale=0.25]{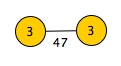} & $\Longleftrightarrow$ \includegraphics[scale=0.25]{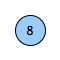} \\
\hline
\includegraphics[scale=0.25]{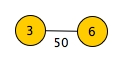} & $\Longleftrightarrow$ \includegraphics[scale=0.25]{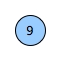} \\
\end{tabular}
\caption{Parts in the model. The blue vertices on the right-hand side represent the corresponding graph on the left-hand side. If one of the graphs on the left-hand side is a subgraph in a larger graph, then we may simplify the description of that larger graph through the use of these parts.}
\label{tab:gt7}
\end{table}

\begin{figure}[H]
\begin{minipage}[c]{0.5\linewidth}
    \includegraphics[scale=0.25]{g1.jpg} 
  \end{minipage} 
    $\Longrightarrow \; $ 
  \begin{minipage}[c]{0.4\linewidth}
    \includegraphics[scale=0.25]{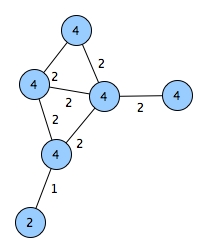} 
  \end{minipage} 
  \begin{minipage}[c]{0.5\linewidth}
    \includegraphics[scale=0.25]{g2.jpg} 
  \end{minipage} 
    $\Longrightarrow \; $ 
  \begin{minipage}[c]{0.4\linewidth}
    \includegraphics[scale=0.25]{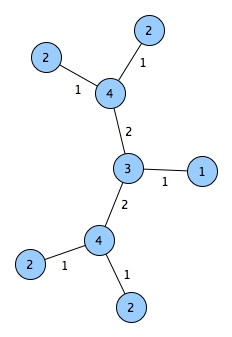} 
  \end{minipage} 
    \begin{minipage}[c]{0.5\linewidth}
    \includegraphics[scale=0.25]{g7.jpg} 
  \end{minipage} 
    $\Longrightarrow \; $ 
  \begin{minipage}[c]{0.4\linewidth}
    \includegraphics[scale=0.25]{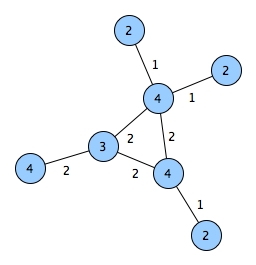} 
    \end{minipage}
  \caption{Examples of molecule graphs (from the MUTAG dataset) depicted by higher-level parts (subgraphs) rather than their lowest level parts. It should be easier to learn a distribution over this higher-level description of these graphs.} 
  \label{fig:chemo2}
\end{figure}



\subsection{Example 5: Vertices with Color and Location}

We consider an example of attributed graphs in which the vertex space $\Lambda_V = \{1,\ldots,p\}^2$ is a two dimensional grid of size $p$, the attribute space $\mathcal{X} = \{c_1,c_2,c_3,c_4\}$ is a set of colors, and a binary edge space $\Lambda_E = \{0,1 \}$. We will define a master interactions function $F_V: \Lambda_V^2 \rightarrow \Lambda_E$ that assigns the value $0$ to every pair of vertices that cannot have an edge, and assigns a value $1$ to every pair of vertices that can have an edge. Define $F_V$ as follows:
\begin{align*}
U(v,v') =  \begin{cases} 
	      			0, & \text{if } d(v,v') > t  \\ 
	      			1, & \text{otherwise} 
			\end{cases}
\end{align*}
where $d(v,v')$ is some distance function that assigns a distance between vertices based on their location attributes. In other words, this master interactions function can be used to ensure there is no edge between vertices that are farther apart than $t \in \mathbb{R}$. Let $N$ be the maximum order of graphs in the graph space. Let the graph space be:
 \begin{align*}
\mathcal{G} = 
  \left\{ (V,E) \; : \; 
 \begin{array}{l}
     V \subseteq \Lambda_V, \; |V| \leq N \\ 
     X: V \to \mathcal{X} \\
     E: V \times V \rightarrow \Lambda_E \\
     E(v,v') \leq F_V(v,v') \text{ for all } v,v' \in V
 \end{array} 
  \right\}.
 \end{align*}
 
The templates and the parameters used are shown in Figures \ref{tab:gt8} and \ref{tab:gt11}. For a template graph $T_k \in \mathcal{G}$, we define the compatibility map $R_k: \mathcal{G} \rightarrow \{0,1\}$ based on graphs that are second-order isomorphic to it (equation \ref{eq:compatibility_map_iso_def}). 
Some samples are shown in Figures \ref{fig:loc}. These were generated using the sampling algorithm in Section \ref{sec:inference_and_learning}.
	






\begin{table}[H]
\centering
\begin{tabular}{c c}
template $T_k$ & parameter value $\lambda_k$  \\
\hline
\includegraphics[scale=0.25]{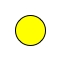}   &  $\lambda_{1} = 0.5$  \\
\hline
\includegraphics[scale=0.25]{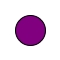}  & $ \lambda_{2} = 0.4$ \\
\hline
\includegraphics[scale=0.25]{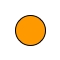}  & $ \lambda_{3} = 0.5$  \\
\hline
\includegraphics[scale=0.25]{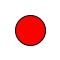}  & $ \lambda_{4} = 0.4$  \\
\hline
\end{tabular}
\caption{The set of 1st-order graphs that are used as templates. Parameters were hand-tuned here.}
\label{tab:gt8}
\end{table}

\begin{table}[H]
\centering
\begin{tabular}{*{4}{m{0.19\textwidth}}}
template $T_k$ & parameter value $\lambda_k$ & template $T_k$ & parameter value $\lambda_k$  \\

\hline
\includegraphics[scale=0.25]{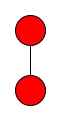} & $ \lambda_{5} = 0.5$ & \includegraphics[scale=0.25]{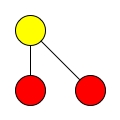} &  $\lambda_{15}  = -5$ \\
\hline
\includegraphics[scale=0.25]{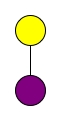} & $ \lambda_{6} = 1.5$ & \includegraphics[scale=0.25]{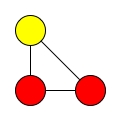} &  $\lambda_{16}  = -5$ \\
\hline
\includegraphics[scale=0.25]{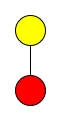} & $ \lambda_{7} = 0.4$ & \includegraphics[scale=0.25]{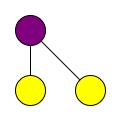} &  $\lambda_{17}  = 0.75$ \\
\hline
\includegraphics[scale=0.25]{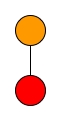} & $ \lambda_{8} = 1.5$ & \includegraphics[scale=0.25]{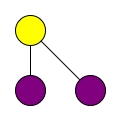} &  $\lambda_{18}  = -5$ \\
\hline
\includegraphics[scale=0.25]{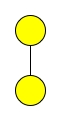} &  $\lambda_{9} = -\infty$ & \includegraphics[scale=0.25]{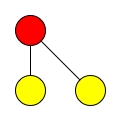} &  $\lambda_{19}  = -\infty$ \\
\hline
\includegraphics[scale=0.25]{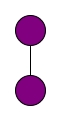} & $\lambda_{10} = -\infty$ & \includegraphics[scale=0.25]{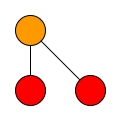} &  $\lambda_{20}  = -\infty$ \\
\hline
\includegraphics[scale=0.25]{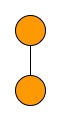} & $\lambda_{11} = -\infty$ & \includegraphics[scale=0.25]{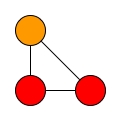} &  $\lambda_{21}  = -\infty$ \\
\hline
\includegraphics[scale=0.25]{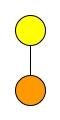} & $\lambda_{12} = -\infty$ &  & \\
\hline
\includegraphics[scale=0.25]{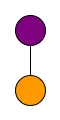} & $\lambda_{13} = -\infty$ & &  \\
\hline
\includegraphics[scale=0.25]{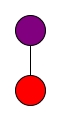} & $\lambda_{14} = -\infty$ & & \\
\hline
\end{tabular}
\caption{The set of 2nd and 3rd order graphs that are used as templates. Parameters were hand-tuned here.}
\label{tab:gt11}
\end{table}

\begin{figure}[H]
\begin{minipage}[c]{0.5\linewidth}
    \includegraphics[scale=0.23]{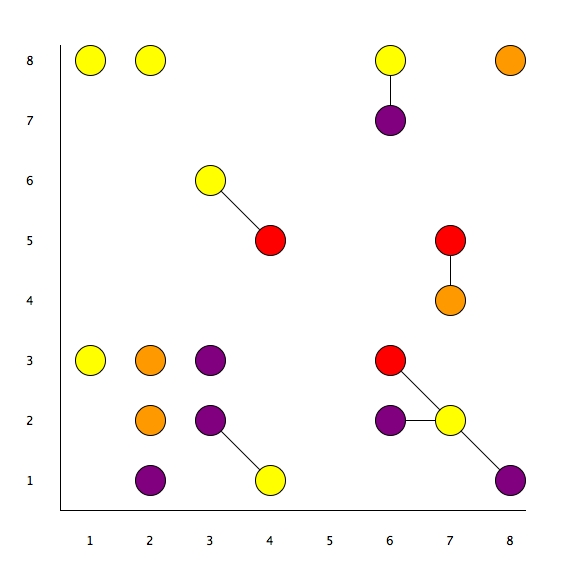}
  \end{minipage} 
  \begin{minipage}[c]{0.5\linewidth}
    \includegraphics[scale=0.23]{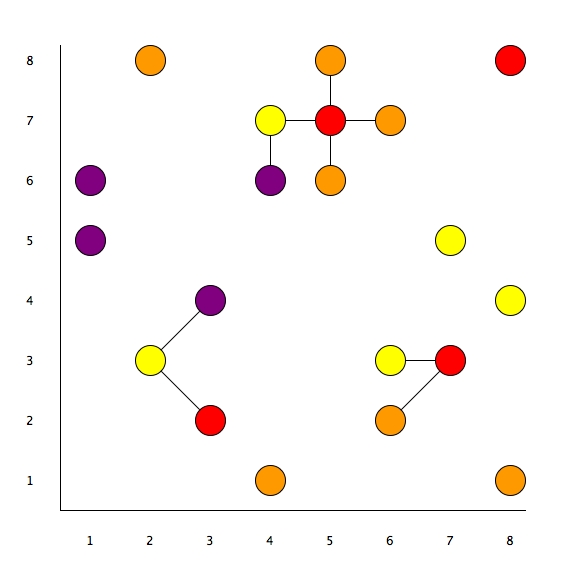}
  \end{minipage} 
  \begin{minipage}[c]{0.5\linewidth}
     \includegraphics[scale=0.23]{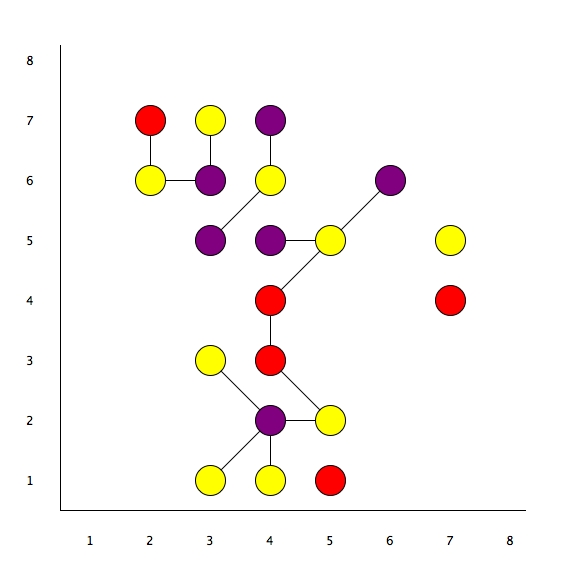}
  \end{minipage} 
  \begin{minipage}[c]{0.5\linewidth}
    \includegraphics[scale=0.23]{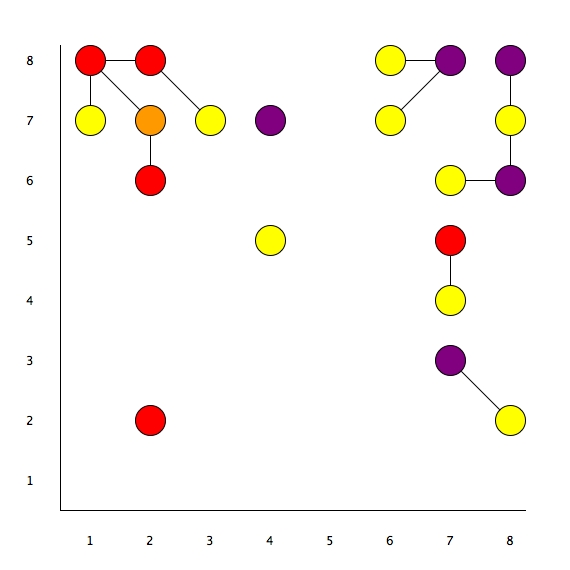}
  \end{minipage} 
    \caption{Samples from the model.} 
    \label{fig:loc}
\end{figure}

\section{Inference and Learning}
\label{sec:inference_and_learning}

In this section, we discuss inference for random graphs; for a given probability distribution, inference refers to the calculation (or estimation) of probabilities in that distribution, or more generally, of functions of probabilities in that distribution. Inference can be performed by sampling from distributions; a standard Metropolis-Hastings algorithm is presented here for this purpose. 

Next, we present a learning algorithm for random graph models; for a given model, learning refers to the selection of a particular distribution in it. A stochastic learning algorithm is presented here.
 

\subsection{Sampling}
\label{sec:sampling_1}

Suppose we have a vertex space $\Lambda_V$, an edge space $\Lambda_E$, and an attribute space $\mathcal{X}$, and let $\mathcal{G}$ be a finite graph space with respect to them. Further, suppose we have a distribution $P$ over $\mathcal{G}$ that we want to sample from. We will use Markov Chain Monte Carlo (MCMC), and in particular, the Metropolis-Hastings algorithm: \\


\begin{algorithm}[H]
\caption{Metropolis-Hastings}
Given a \textit{transition kernel} $q(G' | G)$ and starting from an initial state $G_1$, repeat the following steps from $t = 1$ to $T$:
\begin{enumerate}
\item Generate a candidate $G' \sim q(G' | G_t)$.
\item Generate $U \sim \mathcal{U}(0,1)$ and set
	\begin{align*}
		G_{t+1} = \begin{cases} 
	      			G', & \text{if } U \leq \alpha(G_t, G')  \\ 
	      			G_t, & \text{otherwise} 
			\end{cases}
	\end{align*}
	where $\alpha(G,G')$ is the \textit{acceptance probability}, given by:
	\begin{align*}
		\alpha(G,G') = \min \left \{ \frac{ P(G') \; q( G \; | G') } {P(G) \; q( G' \; | G) }, 1 \right \}.
	\end{align*}
\end{enumerate}
\end{algorithm}

This algorithm will generate a sequence $G_1, G_2, \ldots$ of dependent random graphs, and for large $t$, the graph $G_t$ will be approximately distributed according to $P$ (assuming an appropriate transition kernel). Given a graph $G = (V,X,E)$, the transition kernel will generate a proposal graph $G'$ based on simple moves. There are many possibilities, but at least for simple problems, we found the following moves suffice:
\begin{enumerate}
\item Adding a vertex $v \in \Lambda_V \setminus V$ to the vertex set and adding an attribute value $X(v)$ to the attributes (and not adding any edges, i.e., $E(v,\cdot)=0$.)
\item Deleting a vertex $v \in V$ from the vertex set and restricting it from $X$ (from among vertices with no incident edges).
\item Changing the value of an edge in $E$ (i.e., changing the value of $E(v,v') \in \Lambda_E$ for some $v,v' \in V$).
\end{enumerate}
The probability of each type of move can be uniform, although sometimes non-uniform probabilities may be preferable (e.g., assigning a greater probability to edge moves). Each move in this set has an inverse move allowing a chain to return to the previous state; hence these moves satisfy the weak symmetry property. 

\subsection{Computation}

We now consider computational efficiencies for this sampling algorithm. Suppose we have a distribution $P$ in the form of equation \ref{eqn:template_model}, i.e.:
\begin{align*}
P(G) & = \frac{1}{Z} \exp \left [ \sum \limits_{k=1}^K \lambda_k U_k(G) \right ],
\end{align*}
 and 
define $H$ as the exponent in this distribution:
\begin{align*}
H(G) =  \sum \limits_{k=1}^K \lambda_k U_k(G) .
\end{align*}
In this section, we consider the calculation of differences of the form
\begin{align*}
\bigtriangleup H \equiv H(G_{\text{new}}) - H(G_{\text{old}}),
\end{align*}
where $G_{\text{new}}$ and $G_{\text{old}}$ are graphs. 
In the naive approach, each exponent is computed separately, and the difference taken. However, the difference $\bigtriangleup H$ can be calculated efficiently by ignoring subgraphs that are shared between the two graphs. Define $J_0 \subseteq S(G_{\text{new}})$ as the set of subgraphs of $G_{\text{new}}$ that are not subgraphs of $G_{\text{old}}$; similarly, define $J_1 \subseteq S(G_{\text{old}})$ as the set of subgraphs of $G_{\text{old}}$ that are not subgraphs of $G_{\text{new}}$. That is:
\begin{align*}
J_0 & \equiv \{ G' \in S(G_{\text{new}}) \; | \; G' \notin S(G_{\text{old}}) \} \cap \mathcal{G}_{\text{basis}} \\
J_1 & \equiv \{ G' \in S(G_{\text{old}}) \; | \; G' \notin S(G_{\text{new}}) \} \cap \mathcal{G}_{\text{basis}}.
\end{align*}
Hence, we have that
\begin{align*}
\bigtriangleup H & \equiv H(G_{\text{new}}) - H(G_{\text{old}}) \\
 & = \sum \limits_{k=1}^K \lambda_k \left [ U_k(G_{\text{new}}) - U_k(G_{\text{old}}) \right ] \\
 & = \sum \limits_{k=1}^K \lambda_k \left [ \tilde{U}_k(G_{\text{new}}) - \tilde{U}_k(G_{\text{old}}) \right ] 
\end{align*}
where
\begin{align*}
\tilde{U}_k(G_{\text{new}}) & = \# \{ G' \in J_0 \; : \; R_k(G')=1 \} \\
\tilde{U}_k(G_{\text{old}})  & = \# \{ G' \in J_1 \; : \; R_k(G')=1 \}.
\end{align*}
Let's consider some examples. \\


\noindent \textbf{New Node}: Let $G_{\text{old}}$ be any graph and suppose we formed $G_{\text{new}}$ by adding a vertex $u$ to it (and possibly edges). In this case, we have that $G_{\text{old}} \subset G_{\text{new}}$ (i.e. is an induced subgraph) and hence:
\begin{align*}
J_0 & = \{ G' \in S(G_{\text{new}}) \; | \;  u \in V(G')  \} \cap \mathcal{G}_{\text{basis}} \\
J_1 & = \emptyset.
\end{align*}


\noindent \textbf{Deleted Node}: Let $G_{\text{old}}$ be a graph (with at least one vertex) and suppose we formed $G_{\text{new}}$ by deleting a vertex $u \in V(G_{\text{old}})$. In this case, we have that $G_{\text{new}} \subset G_{\text{old}}$ and hence:
\begin{align*}
J_0 & = \emptyset \\ 
J_1 & = \{ G' \in S(G_{\text{old}}) \; | \;  u \in V(G') \} \cap \mathcal{G}_{\text{basis}}.
\end{align*}


\noindent \textbf{New Edge}: Let $G_{\text{old}}$ be a graph (with at least two vertices) and suppose we formed $G_{\text{new}}$ by changing the value of an edge $E(v,v')$ for some $v,v' \in V(G_{\text{old}})$. In this case, we have that:
\begin{align*}
J_0 & = \{ G' \in S(G_{\text{new}}) \; | \;  v,v' \in V(G')  \} \cap \mathcal{G}_{\text{basis}} \\ 
J_1 & = \{ G' \in S(G_{\text{old}}) \; | \;  v,v' \in V(G')  \} \cap \mathcal{G}_{\text{basis}}.
\end{align*}

\subsection{Learning}
\label{sec:learning}

To estimate the parameters $\lambda = \{\lambda_1,\ldots,\lambda_K\}$ in the model in equation \ref{eqn:template_model}, we use the maximum likelihood estimate (MLE). Suppose we have a set of graphs $\{G_1,G_2,\ldots,G_N\}$ sampled according to $P(G ; \lambda^*)$, where $\lambda^* \in \Lambda$ and where $\Lambda \subset \mathds{R}^K$ is a compact set. The MLE is the solution to the following optimization problem:
\begin{align*}
\hat{\lambda} = \arg \max_{\lambda \in \Lambda} \prod_{i = 1}^N P(G_i; \lambda).
\end{align*}
It can be shown that the MLE is also a solution to the equations:
  \begin{align*}
  	\mathds{E}_{\lambda}[U_k(G)] = \hat{U}_k, \;  \; k=1,\ldots,K,
  \end{align*}
  where $\hat{U}$ are the empirical statistics of the features:
  \begin{align*}
   \hat{U}_k = \frac{1}{N} \sum_{i=1}^{N} U_k(G_i), \quad k=1,\ldots,K.
  \end{align*}
We run a stochastic approximation algorithm (\citep{younes1988estimation}, \citep{salakhutdinov2009learning}) for estimating the solution (see Algorithm 2 below). This stochastic algorithm uses multiple Markov chains in parallel; we find this beneficial in practice since individual chains can sometimes become stuck in certain regions for long periods of time. \\ 


\begin{algorithm}[H]
\caption{Stochastic Approximation Procedure}
\textbf{Input:} \\
Empirical statistics $\hat{U}$; \\
Initial parameters $\lambda^1$; \\
Initial set of $M$ particles $\{G^{1,1},\ldots,G^{1,M}\}$; \\


\begin{algorithmic} 
\FOR{$t = 1:T \; \;  \text{(number of iterations)}$}
	\FOR{$m = 1:M \; \;  \text{(number of parallel Markov chains)}$}
		\STATE $\text{Sample } G^{t+1,m} \text{ given } G^{t,m} \text{ using the transition operator } T_{\lambda^t}( G^{t+1,m}, G^{t,m})$
	\ENDFOR
\ENDFOR
\STATE $\text{Update: } \lambda^{t+1}=\lambda^t + \alpha_t \left [\hat{U} -  \frac{1}{M} \sum_{m=1}^{M} U(G^{t+1,m}) \right ]$
\STATE $\text{Decrease } \alpha_t$
\end{algorithmic}
\end{algorithm}

\chapter{Summary and Discussion}

In this work, we considered the statistical modeling of real-world problems that involve objects that are naturally represented by graphs of varying orders. In general, a distribution can be specified over a finite space of objects (e.g., a finite graph space), by directly assigning a probability to each object in it. Of course, for all but the smallest spaces, this is impractical, and it is essential to use invariance assumptions. 
To address this, we considered independence and factorization, the invariances used in graphical models, and observed that they can be defined in terms of projections, allowing their use in the modeling of random graphs. These invariances, for a given family of projections, can be described using only a small subset of projections, those that are atomic (with respect to this family), allowing their compact representation by a (structure) graph. We found factorization to be an easier invariance to use for graphs, at least in some problems, and illustrated their modeling with some examples.

One critique of this work may be that the formulation of graphical models discussed here is not actually an extension, and that only the original formulation of graphical models was presented, applied to a particular multivariate random variable. That is, since a multivariate random variable is a general object, almost any other object can be represented by it (e.g., for a random graph, since each of its marginal random graphs is a random variable, is representable by a multivariate random variable). We mention that the reverse is also true: a random graph is also a general object, and any multivariate random variable can also be represented by it (e.g., for a multivariate random variable, each of its individual marginal variables can be represented by an attributed vertex, and then graphs assumed to have no edges). This equivalence is examined in \citep{grenander2007pattern}, chapter 6, where it was proven more formally (for the random graph model considered in that work). In this work, we are not concerned with how a random variable is interpreted; for the purposes of statistical modeling, it is irrelevant if a variable is considered a graph or vector, it is only the structure within the object that matters. Thus, we view a graphical model as a framework for modeling any object with an appropriate structure, where this structure is defined by a projection family. The value of this more abstract viewpoint is that it focuses attention on the essence of these models, the relationship between object structure and invariance, and hence clarifies the modeling of complicated objects such as graphs and trees. 

Lastly, we note that although we limited our attention here, for the most part, to the invariances used in graphical models, that there may exist others that result in similar modeling frameworks.  
A model is a set of distributions, and a modeling framework a set of models, usually defined in terms of an invariance. To design a (useful) modeling framework, an invariance must be identified such that: (a) it is applicable to many problems; and (b) it can be made at varying degrees, creating a range of model complexities, and allowing practitioners to adjust models to a given problem. 
It would be interesting to investigate other invariances, besides independence and factorization, that have hierarchical structures. 
One could imagine, if such invariances could be identified, the development of frameworks similar to graphical models. Perhaps, given an appropriate hierarchical structure, these invariances would be also representable by graphs, in which case, the graphical model framework itself would be expanded. 

\section{Extended Discussion}

We address questions we have received and elaborate on some possible points of confusion in this section. 

\subsection{Framework Merit}

  \textbf{Q}: Since some random objects can be converted to random vectors, the extended formulation of graphical models is equivalent, in some instances, to the original. What is the value of using the more abstract formulation? \\
  
\noindent \textbf{A}: This question, in essence, is asking about how we should define graphical models and why is this abstract formulation even necessary. There are a few ways in which the more abstract formulation might be of value. 
If we take graphical models to be, at its core, a modeling framework based on invariances between random variables in some family, then the question of how to define these models reduces to the question of what properties this family should satisfy. Given a distribution over a countable space $\Omega$, any function on this space defines a random variable, and any family of functions defines a family of random variables. In this work, for a family of functions to produce an appropriate family of random variables for graphical models, we considered the fundamental properties to be that it is finite, consistent, and complete (Section \ref{sec:General_Random_Objects}). 
These latter properties are also pertinent 
to spaces that have infinite projections, rather than just a finite set. This is the case, for example, in stochastic processes where $\Omega$ is an uncountable space of functions, and has an infinite set of consistent projections on it, where each restricts functions in this space to a finite set of points in their domain (i.e., the projections that define the finite-dimensional distributions.) 
Since the concept of consistent projections occurs throughout probability theory, there appears to be conceptual value in aligning the definition of graphical models with it, rather than confining them to coordinate projections on product spaces. 

From a more applied viewpoint, the value of a general framework is partly related to the degree that it compresses knowledge, which can for example, illuminate similarities and differences in somewhat disconnected models. 
In the formulation of graphical models given here, many classic random graph models are covered; more importantly though, it provides insight into modeling random graphs in which the graphs vary in order, an important problem that has received less research attention (we defer discussion of this topic to the next section). This framework for random graphs also includes random trees as a particular instance, a desirable property since a tree is a type of graph, and it covers some classic random tree models (e.g., probabilistic context-free grammars). The framework is based on fundamental invariances for structured objects, where the necessary structure has been defined. 
Finally, due to their importance, many inference algorithms are specifically designed with these invariances in mind. 

On a more technical note, even if a random object can be mapped to a random vector, this does not imply that graphical models for the random object are equivalent to graphical models for the corresponding random vector. A basic property of graphical models is that they specify sufficient conditions for the invariances they use to be consistent (i.e., consistent in the sense that there exists a distribution that satisfies them). For example, in Bayesian networks for random vectors, the structure graph encodes a set of consistent independence assumptions if the structure graph is acyclic. For other random objects, however, this statement does not necessarily hold. 
This can be seen, for example, in random graphs; since a graph has structural constraints imposed by the dependence of edges on vertices, for structure graphs to be valid, the invariances they specify cannot violate these structural constraints (Section \ref{sec:Bayesian_Networks}). Similarly, in random trees, there are structural constraints imposed by the dependence of a vertex on its ancestors. Since structural constraints of objects in a space $\Omega$ can be described by a projection family on it, the more general formulation of graphical models can ensure, since they are defined in terms of these projection families, the consistency of their invariances on more general random objects. 

\subsection{Random Vertices}

  \textbf{Q}: In the random graph models discussed in this work, the vertex set can be random. Why is this of interest? \\

\noindent \textbf{A}: There are two problem paradigms in which random graphs are applicable, the traditional one in which graphs have a fixed order and randomness only on the edges, and the less established one in which graphs can have variable order. We will compare these two paradigms in more detail in the next section. For the moment though, we note that this work will not be useful to those interested in the former problem. Rather, the focus here is on the latter, and those interested in that problem are the intended audience.

Applications of random graphs with random vertices are studied, for example, within the field of statistical relational learning \citep{getoor2007introduction}. One application is the statistical modeling of a set $\Omega$ of real-world scenes. Suppose scenes are composed of objects with attributes (e.g., child wearing a hat, blue sedan, etc.), and relationships between objects (e.g., holding hands, driving, etc.), and further, scenes vary in the objects in them (e.g., a scene may be empty or may have numerous objects). These scenes can be represented by graphs, where vertices represent objects, attributes of vertices represent attributes of objects (including their type), and edges represent relationships between objects. Thus, modeling these scenes corresponds to modeling random graphs in which graphs vary in their order. In the literature, many approaches to this problem are based on modeling a selected set of conditional distributions, where each is conditioned on the objects in the scene. For example, in probabilistic relational models \citep{getoor2001learning}, these conditional distributions are specified using templates and assuming repeated structure. In this work, we considered a formulation of graphical models that would allow us to model full distributions over this type of space.

It is worth mentioning that this formulation of graphical models may be relevant to applied research in \textit{probabilistic logic}, a field that couples probability and logic (an overview is given in \citep{russell2015unifying}). 
Notice that if we have a probability distribution over a countable space $\Omega$, then sentences of a logical language about this space (where sentences correspond to binary functions of the form $f:\Omega \to \{0,1\}$), can be assigned probabilities. (That is, a sentence assigned probability equal to the probability of the subset of $\Omega$ in which it is true.) 
If, on the other hand, the distribution over a space $\Omega$ is unknown and we want to learn it, logical expressions can be used to express invariances 
(i.e., constraints, structure) in the distribution. For example, invariances can be defined on distributions by constraining the distributions to those that assign certain probabilities to certain sentences. More generally, invariances can be defined on these distributions in terms of logical expressions about the distribution itself (\citep{fagin1990logic}, \citep{halpern1990analysis}), referred to as \textit{probability expressions} (these correspond to functionals of the form $f(P) \mapsto \{0,1\}$, where $P$ is a distribution over $\Omega$). In other words, distributions are constrained to only those in which some set of probability expressions are true. Thus, probabilistic logic can be viewed as a modeling framework based on general invariances, as expressed by logical expressions about the distribution. (In contrast, graphical models is a framework based on the less expressive invariances of independence and factorization.)

This level of expressiveness in invariances, however, can result in the specification of a set of invariances that is inconsistent in the sense that there does not exist a well-defined distribution that satisfies it. This problem has led researchers to consider forgoing some of the expressive power in these logics to ensure consistency, and since graphical models provide consistency guarantees for their invariances (see previous section), to research extensions of graphical models to more general spaces. Of particular interest are extensions to spaces $\Omega$ containing structured objects not of a fixed size (such as, for example, the real-world scenes described above), and an example is the template-based graphical models mentioned above. 

In \citep{milch20071}, an extension of Bayesian networks is proposed for modeling full distributions over such spaces, referred to as Blog, where probabilities are placed on objects in $ \Omega$ based on how they are incrementally constructed from some generative process. 
In this work, an extension was also proposed, and it is instructive to examine the differences between these two general frameworks. 
Firstly, the formulation here is based on finite families of functions on the space $\Omega$, whereas theirs extends to infinite families. (It would be interesting to expand the formulation here in this regard.) 
Secondly, in Blog, the function families are not required to be consistent (i.e., the formulation is not in terms of families of marginal random variables, but rather general random variables). A family of random variables that is inconsistent (among the subset that is atomic) generally produces an inefficient representation of invariances in a distribution. This is not necessarily a problem, but since the traditional formulation of graphical models is in terms of marginal variables, we defined our extension in terms of them as well. Finally, the formulation here yields other forms of factorization (i.e., undirected models where probabilities are placed on objects in $\Omega$ based on how they deconstruct (factorize) into a set of parts). 





\subsection{Consistent Distributions}



\textbf{Q}: In \citep{shalizi2013consistency}, it was observed that for random graphs, if a set of distributions from the same exponential family all have the same parameter values, then these distributions are inconsistent, except under very special circumstances. How does the proposed framework address this consistency problem? \\ 

\noindent \textbf{A}: To answer this question about the consistency of distributions, it is constructive to consider it from two perspectives. As mentioned above, there are two paradigms for random graphs, the traditional one in which graphs have a fixed order and randomness only on the edges, and the less established one in which graphs have randomness on both vertices and edges. Let's consider consistency in each. For simplicity, suppose we have a vertex space $\Lambda_V$ and edge space $\Lambda_E$ that are both finite. 


In the random vertices setting, a random graph has a distribution $P$ over a graph space
  \begin{align*}
\mathcal{G} = 
 \left\{ G=(V,E) \; | \; 
 \begin{array}{l}
     V \subseteq \Lambda_V \\ 
     E: V \times V \rightarrow \Lambda_E 
 \end{array} 
  \right\},
 \end{align*}
containing graphs of varying order. In this case, we can form conditional distributions of the form $P_V(E) \equiv P(E | V)$ by conditioning on the vertices $V \subseteq \Lambda_V$ in a graph. A set of conditional distributions $\{P_V, \; V \subseteq \Lambda_V\}$ is consistent if there exist a full distribution $P$ that produces it. Since the random graph models in this work define full distributions over $\mathcal{G}$, any conditional distributions induced from them will be consistent, and hence this consistency problem is not an issue. (Recall, one of our motivations for modeling full distributions was to avoid these difficult consistency problems.)

In the traditional setting, a random graph has a distribution $P$ over a graph space 
\begin{align*}
\mathcal{G}_{\Lambda} = \{G \in \mathcal{G} \; | \; V(G) = \Lambda_V\},
\end{align*}
containing only graphs with $n = |\Lambda_V|$ vertices, where $n$ is large and possibly infinite (for simplicity, we assume finiteness here). For modeling purposes, we can define projections to substructures of this graph space as follows. For a set of vertices $V \subseteq \Lambda_V$, let the substructure projection $\pi_V:\mathcal{G}_{\Lambda} \to \mathcal{G}_V$, where $\mathcal{G}_V = \{G \in \mathcal{G} \; | \; V(G) = V\} \not \subseteq \mathcal{G}_{\Lambda}$, be defined as $\pi_V(G) = G(V)$ (i.e., the subgraph of $G$ induced by the vertices $V$). If we have a distribution $P$ over $\mathcal{G}_{\Lambda}$, then the projections $\pi_V$, $V \subseteq \Lambda_V$ define a set of marginal distributions $P_V^{\text{marg}}$, $V \subseteq \Lambda_V$. On the other hand, if we do not have the distribution $P$ and we want to determine it through the specification of a set of distributions $P_V^{\text{marg}}$, $V \subseteq \Lambda_V$ (and using an extension theorem in the infinite case), we must take care to ensure we specify a set that is consistent. In \citep{shalizi2013consistency}, it was shown that if a set of distributions from the same exponential family all use the same parameter values for each distribution $P_V^{\text{marg}}$ in the set, then it is not consistent except under special circumstances. This result shows more about the difficulty of directly specifying a set of consistent distributions (especially when these distributions are assumed invariant to isomorphisms, see next section) than any limitations of exponential models. 

It is worth mentioning that this set of projections to substructures (i.e., the set $\{\pi_V, V \subseteq \Lambda_V\}$) is consistent and complete, and so graphical models, as formulated in Section \ref{sec:General_Random_Objects}, is applicable to them when this projection family is finite. However, the formulation of random graphs presented here based on variable-order graph spaces and subset projections (rather than fixed-order graph spaces and substructure projections) may provide a more interesting vantage point since it includes random trees as an instance of it.

\subsection{Degeneracy}

 \textbf{Q}: In the literature, many of the proposed exponential random graph models suffer from a `degeneracy' problem. Since the random graph models in this work are exponential models, why are they not degenerate?\\

\noindent \textbf{A}: It is important to stress upfront that degeneracy has nothing to do with exponential models, but rather is an issue most acutely affecting unattributed random graph models that use the isomorphism invariance (i.e., the assumption that all isomorphic graphs have the same probability). The formulation of graphical models for random graphs (Section 2) does not use this invariance, and so is not degenerate. Further, if the graphs are attributed, then even if models use isomorphism invariances (Section \ref{sec:attributed_graph_isos}), 
the probability of a graph will still depend on its vertex set beyond its cardinality; latent variables are now associated with the vertices, providing a means for models to differentiate vertices from each other. In the literature, many models use attributes (latent variables) and are not considered degenerate (although, most of these are not high order models). 
In Section 5, we gave five random graph examples; all but the first example, which was a toy problem for illustration, used attributed graphs (as well as constrained graph spaces). We now let the discussion proceed to the setting in which the degeneracy problem appears, at least in an acute form.

To begin our discussion, it's useful to illustrate the degeneracy problem with an example. Suppose we have a vertex space $\Lambda_V = \mathbb{N}$ and edge space $\Lambda_E = \{0,1\}$, and let $\mathcal{G}$ be the graph space containing all graphs with respect to them. For a finite set of vertices $V \subset \Lambda_V$, let $P_V$ denote the following conditional distribution:
\begin{align*}
P_V(E) \equiv P(E|V),
\end{align*}
where $E$ can be any function of the form $E: V \times V \to \Lambda_E$. Assume that these conditional distributions are invariant to isomorphisms:
\begin{align*}
G_0 \simeq G_1 \implies P_{V_0}(E_0)=P_{V_1}(E_1),
\end{align*}
where $G_0 = (V_0,E_0)$ and $G_1 = (V_1,E_1)$. For all vertex sets of a given cardinality $n \in \mathbb{N}$, since these distributions only depend on the edge structure, and not on the particular vertices, we can assume the vertex set is $V=\{1,\ldots,n\}$ and index the conditional distributions by $P_n \equiv P_V$. Now consider the exponential model studied in \citep{handcock2003assessing}; suppose the model has two sufficient statistics, the number of edges and the number of `2-stars' in a graph:
\begin{align*}
P_{n}(E) & = \frac{1}{Z(n,\lambda)} \exp \left [ \lambda_1 f_1(E) + \lambda_2 f_2(E) \right ] \\
f_1(E) & = \sum \limits_{i < j} E(i,j) \\
f_2(E) & = \sum \limits_{i<j<k} E(i,j) E(i,k)
\end{align*}
where $f_1$ is the number of edges and $f_2$ is the number of 2-stars (i.e., edges sharing a common vertex), and where $\lambda = (\lambda_1,\lambda_2)$ is a set of real parameters. This model was analyzed when the number of vertices $|V|=7$, and even at this small size, the following problems were observed. First, only a small subset of the parameter space for this model corresponds to distributions that are not degenerate, where loosely speaking, a distribution is considered degenerate if it places almost all of its mass on either the empty graph or the complete graph; these distributions are uninteresting from a modeling perspective. Second, small changes in the parameters can result in distributions that are dramatically different and learning algorithms often do not converge or take extremely long to do so. 

Since the practical use of models that suffer from these degeneracy issues is severely limited, developing an understanding of it is critical. A formal definition concerning this degeneracy appears in the works \citep{strauss1986general} and \citep{schweinberger2011instability}. We first present this definition, before making comments. 
Suppose we have a set of distributions $P_{n}$, $n \in \mathbb{N}$, and let
\begin{align*}
P_{n}^* = \max_{E} P_{n}(E)
\end{align*}
denote its maximum value. Further, let
\begin{align*}
\mathcal{M}_n = \{ E \; : \; P_{n}(E) = P_{n}^* \}
\end{align*}
be the subset of modes, and for any $0 < \epsilon < 1$, let 
\begin{align*}
\mathcal{M}_{\epsilon,n} = \{ E \; : \; P_{n}(E) > (1-\epsilon) P_{n}^* \}
\end{align*}
be the subset of $\epsilon$-modes for $P_{n}$. The following definition is based on the observation that degenerate exponential-family distributions tend to concentrate almost all of their mass on the distributions modes: 

\begin{definition}[\citep{strauss1986general}, \citep{schweinberger2011instability}]
A set of distributions $P_{n}$ is degenerate if, for any $0 < \epsilon < 1$, we have
\begin{align*}
P_{n}( \mathcal{M}_{\epsilon,n}) \longrightarrow 1 \text{ as } n \longrightarrow \infty.
\end{align*}
\end{definition}

We now make some comments. First, notice that by this definition, every distribution in the Erd\H os-R\'enyi model is degenerate. (This follows from well-known results about typical sets for sequences of independent and identically distributed random variables \citep{cover2012elements}, chapter 3.) This suggests that the characterization of degeneracy in this definition differs from our intuition: there appears to be a fundamental difference between the Erd\H os-R\'enyi model and the 2-star model described above, and this should be captured in any definition. 

Second, this definition only concerns degeneracy of a distribution, whereas it appears that degeneracy should be a property of a model (i.e., a set of distributions). For example, in the empirical studies performed in \citep{handcock2003assessing}, it was observed that nearly all distributions in the model placed the majority of their mass on either the empty graph or the complete graph. Hence, it is not the fact that the distributions in this model are `degenerate' in the sense of the above definition that is interesting, but rather that (1) most distributions in this model have the exact same modes, i.e., the empty or complete graph; and (2) the transition between these two disparate distributions can occur suddenly in the parameter space. This suggest that a more suitable definition would concern phase-transitions, i.e., loosely, the existence of singularities in the parameter space. This might be defined as the existence of a point in the parameter space such that arbitrary small balls around it contain points corresponding to dramatically different distributions, taking limits appropriately.

A phase-transition definition appears to capture our intuitive idea of what constitutes degenerate and non-degenerate models. For example, the Erd\H os-R\'enyi model would not be classified as degenerate by this definition. Consider another example; suppose we want to extend the second-order Erd\H os-R\'enyi model to a third-order one. There are four relevant statistics to consider: (a) the number (of third-order subgraphs) with three edges; (b) the number with two edges; (c) the number with one edge; and (d) the number with no edges. Adding any three of these four statistics to the Erd\H os-R\'enyi model results in a set of sufficient statistics for the extended model. The important thing to notice here is that these sufficient statistics counterbalance each other: an increase in the parameter corresponding to the number of triangles can be offset, loosely speaking, by an increase in the parameter corresponding to the number of empty (third-order) graphs. This balance in the set of sufficient statistics suggests that distributions in this model change smoothly in the parameter space (when parameters are non-zero), and thus this model does not contain singularities, given an appropriate definition. In turn, this suggests that learning (e.g., maximum likelihood estimation) is feasible and that this model may be useful in practice. 

Lastly, let's consider model complexity; in the example in Section \ref{sec:example2_molecule_graphs}, it was observed that adding attributes to unattributed graphs, loosely speaking, allows them to be modeled using lower order models. Since, for any model using latent variables, there exists a more complex model without latent variables that is equivalent, this suggests that models using unattributed graphs may require, to produce equivalent distributions, extremely high orders.

\begin{acknowledgements}
\addcontentsline{toc}{chapter}{Acknowledgements} 
I am indebted to Professors Laurent Younes and Donald Geman, for looking at very preliminary and vague versions of this work, 
and their helpful discussions and insights.
\end{acknowledgements}

\appendix
\chapter{Statistical Invariances}
\label{sec:stat_invariances}



Graphical models are a tool that use conditional independencies to produce distributions with compact representations, promoting learning and inference. 
More generally, however, any statistical invariance - not just conditional independence - may be used to compress representations and ease learning. 
One example is in random graphs, where a common assumption is that the probability of a graph only depends on its edge structure, and not on the particular labeling of the vertices; distributions are assumed invariant to graph isomorphisms (see Section \ref{sec:chap2_isomorphisms}). Another example is in the modeling of the spatial configuration of objects in street-scene imagery, where a simple assumption is that there's a symmetry to the horizontal location of objects, e.g. the probability of a person on the left half of an image is equal to the probability of a person on the right half \citep{geman2015visual}. Naturally, it is beneficial for practitioners to incorporate as many valid invariances as possible into their models. 

In this appendix, the goal is to simply provide a formulation of invariance that encompasses these ideas. We begin by considering some basic definitions of invariance involving transforms and then proceed to more complicated definitions of invariance involving functions and distributions. In Section \ref{sec:d}, we show that independence and conditional independence are particular instances of this general formulation. In Section \ref{sec:moment_invariance}, we consider invariances that are functions of the probability space, which we refer to as moment invariances.


\section{Invariant Transformations}

For a given probability space, a statistical invariance is some property of objects in it such that, loosely speaking, objects that share the same property have the same probability. A property of a space $\mathcal{X}$ is said to be invariant under a transformation if the transformation preserves that property. In general, a property may be considered as an equivalence class in which $x_1,x_2 \in \mathcal{X}$ share the property if and only if $x_1$ and $x_2$ are in the same equivalence class. In other words, define a property on $\mathcal{X}$ as an equivalence class, or likewise, as an equivalence relation: 

\begin{definition}[Equivalence Relation]
An \textit{equivalence relation} on $\mathcal{X}$ is a symmetric, reflexive, and transitive binary relation $\sim$ on $\mathcal{X}$. For any $x_1,x_2 \in \mathcal{X}$, we say $x_1$ and $x_2$ have the same property if $x_1 \sim x_2$.
\end{definition}

Now, we can define an invariant transform: 

\begin{definition}[Invariant Transform]
A transform $T: \mathcal{X} \rightarrow \mathcal{X}$ is \text{invariant} to an equivalence relation $\sim$ if for all $x_1,x_2 \in \mathcal{X}$, we have
\begin{align*}
x_1 \sim x_2 \implies T(x_1) \sim T(x_2). 
\end{align*}
\end{definition}




\section{Invariant Functions} 

We defined invariances for functions of the form $T: \mathcal{X} \rightarrow \mathcal{X}$. Now, let's extend the notion of invariance for a general function $f: \mathcal{X} \rightarrow \Psi$ in which the domain and codomain are not necessarily the same space. To define an invariance on a general function, we'll need two equivalence relations, one on the domain $\mathcal{X}$ and one on the codomain $\Psi$: denote these two equivalence relations as $\sim_{\mathcal{X}}$ and $\sim_{\Psi}$, respectively. Now, let's define an invariant function: 

\begin{definition}[Invariant Function]
A function $f: \mathcal{X} \rightarrow \Psi$ is \text{invariant} to equivalence relations $\sim_{\mathcal{X}}$ and $\sim_{\Psi}$ if for all $x_1,x_2 \in \mathcal{X}$, we have
\begin{align*}
x_1 \sim_{\mathcal{X}} x_2 \implies f(x_1) \sim_{\Psi} f(x_2).
\end{align*}
\end{definition}
We'll find this definition useful in defining invariant distributions.

\subsection{Invariant Probability Mass Functions}


In this section, we consider the invariance of probability mass functions (pmf); for ease of exposition, we skip the details about more general distributions. 
Let $(\mathcal{X},\Sigma_{\mathcal{X}},P)$ be a probability space in which $\mathcal{X}$ is countably infinite and $\Sigma_{\mathcal{X}}$ is the power set of $\mathcal{X}$.
As above, to define the invariance of a function $P:\mathcal{X} \rightarrow [0,1]$, we will need to specify equivalence relations on both $\mathcal{X}$ and $[0,1] \subset \mathbb{R}$. Suppose we have some equivalence relation $\sim_{\mathcal{X}}$ on $\mathcal{X}$, and for the equivalence relation on $\mathbb{R}$, let $\sim_{ \mathbb{R}}$ be the equality relation (i.e. for any $r_1, r_2 \in \mathbb{R}$, let $r_1 \sim_{ \mathbb{R}} r_2$ if and only if $r_1 = r_2$). For simplicity, we will always use the equality relation as the relation on the real numbers $ \mathbb{R}$. Thus, we have: 

\begin{definition}[Invariant pmf's]
\label{def:regular_def}
A pmf $P$ is \textit{invariant} to $\sim_{\mathcal{X}}$ if for all $x_1,x_2 \in \mathcal{X}$, we have:
\begin{align*}
x_1 \sim_{\mathcal{X}} x_2 \implies P(x_1) = P(x_2).
\end{align*}
\end{definition}

%

Let's consider examples. 

\begin{example}[Symmetry]
Let $\mathcal{X}= \mathbb{Z}$ be the integers, and define the following equivalence relation: for every $x_1,x_2 \in \mathcal{X}$:
\begin{align*}
x_1 \sim_{\mathcal{X}} x_2 \Longleftrightarrow |x_1| = |x_2|.
\end{align*}
Thus, a pmf $P$ that is invariant to this equivalence relation will have a symmetry in which $P(x)=P(-x)$ for all $x \in \mathcal{X}$. If one is trying to estimate $P$ from data, this symmetry can be utilized to improve the estimate, as is the case for any invariance.
\end{example}

\begin{example}[Graph Isomorphisms]
Let $\mathcal{X}= \mathcal{G}$ be a graph space, and define the following equivalence relation: for every $G_1,G_2 \in \mathcal{G}$:
\begin{align*}
G_1 \sim_{\mathcal{X}} G_2 \Longleftrightarrow G_1 \text{ is isomorphic to } G_2.
\end{align*}
Thus, a pmf $P$ that is invariant to this equivalence relation will have an invariance in which all graphs that are isomorphic to each other have the same probability. Many random graph models use this invariance (e.g., the Erd\H os-R\'enyi model). 
\end{example}





\section{Conditional Invariance} 
\label{sec:d}

In this section, we consider conditional and marginal invariances. Then we define independence and conditional independence, two special cases. 
Let $(\Omega,\Sigma_{\Omega},\mathbb{P})$ be a probability space and let $\mathcal{X}$ be a countably infinite space. Further, let $\bold{X}:\Omega \rightarrow \mathcal{X}$ be a measurable function (i.e. a discrete random variable) with a pmf $P_{\mathcal{X}}: \mathcal{X} \rightarrow [0,1]$, where
\begin{align*}
P_{\mathcal{X}}(x) = \mathbb{P}(\{w \in \Omega \; : \; \bold{X}(w)=x\}), \; x \in \mathcal{X}.
\end{align*}
Thus, we may form the probability space $(\mathcal{X},\Sigma_{\mathcal{X}},P_{\mathcal{X}})$, where $\Sigma_{\mathcal{X}}$ is the power set of $\mathcal{X}$. Furthermore, for any $A \in \Sigma_{\Omega}$ such that $\mathbb{P}(A)>0$, we may form the conditional probability:
\begin{align*}
P_{\mathcal{X}}(x|A) & = \mathbb{P}(\{w \in \Omega \; : \; \bold{X}(w)=x\} \; | \; \{w \in A\} ) \\
 & = \frac{\mathbb{P}(\{w \in A \; : \; \bold{X}(w)=x\}) }{ \mathbb{P}(\{w \in A\})  } 
\end{align*}
where $x \in \mathcal{X}$. Finally, suppose we have an equivalence relation $\sim_{\Sigma_{\Omega}}$ on $\Sigma_{\Omega}$. Then, we may define conditional invariance as follows. 

\begin{definition}[Conditional Invariance]
\label{def:conditional_def}
A distribution $P_{\mathcal{X}}$ is \textit{conditionally invariant} to $\sim_{\Sigma_{\Omega}}$ if, for all $A_1,A_2 \in \Sigma_{\Omega}$, we have:
\begin{align*}
A \sim_{\Sigma_{\Omega}} A_2 \implies P_{\mathcal{X}}(\cdot|A_1) = P_{\mathcal{X}}(\cdot|A_2).
\end{align*}
\end{definition}

Now, we consider distributions that have both regular invariances and conditional invariances.

\subsection{General Invariance}

Thus far, we have described a regular invariance (Definition \ref{def:regular_def}) and a conditional invariance (Definition \ref{def:conditional_def}); now, suppose we want to combine these.
As above, suppose we have the two probability spaces:
\begin{align*}
& (\Omega,\Sigma_{\Omega},\mathbb{P}) \\
& (\mathcal{X},\Sigma_{\mathcal{X}},P)
\end{align*}
where $P$ was the distribution induced by $\mathbb{P}$. Now, we will define an invariance in terms of an equivalence relation $\sim_{\Sigma}$ on the product space $\Sigma = \Sigma_{\Omega} \times \Sigma_{\mathcal{X}}$. The general definition of an invariance is as follows. 

\begin{definition}[General Invariance]
A distribution $P$ is \textit{invariant} to $\sim_{\Sigma}$ if for all $(A_1,B_1),(A_2,B_2) \in  \Sigma_{\Omega} \times \Sigma_{\mathcal{X}}$ we have:
\begin{align*}
(A_1,B_1) \sim_{\Sigma} (A_2,B_2) \implies P(B_1|A_1) = P(B_2|A_2). 
\end{align*}
\end{definition}





\begin{example}[Independence]
Suppose we have an equivalence relation $\sim_{\Sigma}$ on the product space $\Sigma = \Sigma_{\Omega} \times \Sigma_{\mathcal{X}}$. Let $A_1,A_2 \in \Sigma_{\Omega}$ and $B_1,B_2 \in \Sigma_{\mathcal{X}}$, and suppose that
\begin{align*}
A_2 = \Omega \text{ and } B_1 = B_2.
\end{align*}
Then, the event $B_1$ is said to be \textit{independent} of the event $A_1$ if $(A_1,B_1) \sim_{\Sigma} (A_2,B_2)$. That is, if $B_1$ is independent of $A_1$, then:
\begin{align*}
P(B_1|A_1) = P(B_1).
\end{align*}
\end{example}

\begin{example}[Conditional Independence]
Suppose we have an equivalence relation $\sim_{\Sigma}$ on the product space $\Sigma = \Sigma_{\Omega} \times \Sigma_{\mathcal{X}}$. Let $A_1,A_2 \in \Sigma_{\Omega}$ and $B_1,B_2 \in \Sigma_{\mathcal{X}}$, and suppose that
\begin{align*}
A_2 \subset A_1 \text{ and } B_1 = B_2.
\end{align*}
Then, the event $B_1$ is said to be \textit{conditionally independent} of the event $A_1 \setminus A_2$, given the event $A_2$ if $(A_1,B_1) \sim_{\Sigma} (A_2,B_2)$. That is, if $B_1$ is independent of $A_1 \setminus A_2$ given $A_2$, then:
\begin{align*}
P(B_1|A_1) = P(B_1|A_2).
\end{align*}
\end{example}

\section{An Alternative Formulation}

From the perspective of estimation, a pmf $P$ that is invariant to an equivalence relation $\sim_{\mathcal{X}}$ on $\mathcal{X}$ allows the production of additional data from our limited samples. That is, suppose we have a sample $x \in \mathcal{X}$; then, there exists a replication procedure on this example $x$ such that the equivalence class is preserved in each replicated version of the data. That is, we could define a transformation that takes any $x \in \mathcal{X}$ to the set $\mathcal{X}' = \{x' \in \mathcal{X} \; : \; x' \sim x \}$. In this section, we present an alternative representation of an equivalence relation $\sim_{\mathcal{X}}$; this representation will not be useful in and of itself, but does motivate more elaborate invariances for modeling purposes, which we consider in the next section.

As above, let $\mathcal{X}$ be a countably infinite space, and now define $\mathcal{T}$ to be the space of bijective transformations of the form $T:\mathcal{X} \rightarrow \mathcal{X}$. Notice that an equivalence relation $\sim_{\mathcal{X}}$ on $\mathcal{X}$ corresponds to a subset $\mathcal{T}_0 \subset \mathcal{T}$ where:
\begin{align*}
\mathcal{T}_0 = \{ T \in \mathcal{T} \; | \; x \sim_{\mathcal{X}} T(x) \text{ for all } x \in \mathcal{X} \}.
\end{align*}
In other words, $\mathcal{T}_0$ contains all bijective transformations that are consistent with the equivalence relation. Thus, the equivalence relation $\sim_{\mathcal{X}}$ induces a binary partition of $\mathcal{T}$. Furthermore, this binary partition may be represented by an equivalence relation $\sim_{\mathcal{T}}$ on $\mathcal{T}$. That is, define $\sim_{\mathcal{T}}$ as follows: for all $T_1,T_2 \in \mathcal{T}$, let
\begin{align*}
T_1 \sim_{\mathcal{T}} T_2 \; \Longleftrightarrow \; T_1, T_2 \in \mathcal{T}_0 \;  \text{ or } \; T_1=T_2.
\end{align*}
Using this equivalence relation on transforms, we can make an alternative definition of an invariant pmf. 

\begin{definition}[Invariant pmf's]
Suppose we have a space $\mathcal{T}$ composed of bijective transformations of the form $T: \mathcal{X} \rightarrow \mathcal{X}$ and suppose that we have an equivalence relation $\sim_{\mathcal{T}}$ on $\mathcal{T}$. A pmf $P:\mathcal{X} \rightarrow [0,1]$ is \textit{invariant} to $\sim_{\mathcal{T}}$ if for all $T_1,T_2 \in \mathcal{T}$, we have: 
\begin{align*}
T_1 \sim_{\mathcal{T}} T_2 \implies P(T_1(x)) = P(T_2(x)) \text{ for all } x \in \mathcal{X}. 
\end{align*}
\end{definition}

Although this alternative definition of an invariance is not particularly interesting in its own right\footnote{Suppose an equivalence relation $\sim_{\mathcal{T}}$ is constructed as described in this section. Then the partition of $\mathcal{T}$ that corresponds to the equivalence relation $\sim_{\mathcal{T}}$ can only have one set of size greater than one (the set in which the identity transform $T(x)=x$ belongs). Hence, this definition is not particularly useful.}, it motivates the concept of invariances of moments, which we now consider. 

\section{Moment Invariance}
\label{sec:moment_invariance}

In this section, we consider an important type of invariance of a distribution, that of moment invariance. As above, suppose we have a countably infinite space $\mathcal{X}$ and a pmf $P_{\mathcal{X}}$ over it. Now, suppose we have: (1) a space $\mathcal{F}$ composed of functions of the form $f: \mathcal{X} \rightarrow \mathbb{R}$; and (2) an equivalence class $\sim_{\mathcal{F}}$ on $\mathcal{F}$. Define moment invariance as follows: 

\begin{definition}[Moment Invariance]
\label{def:moment_invariance}
Suppose we have a set of real functions $\mathcal{F}$ and an equivalence class $\sim_{\mathcal{F}}$ on it. A pmf $P$ is \textit{moment invariant} to $\sim_{\mathcal{F}}$ if for all $f_1,f_2 \in \mathcal{F}$, we have:
\begin{align*}
f_1 \sim_{\mathcal{F}} f_2 \implies E_P(f_1(\bold{X})) = E_P(f_2(\bold{X})).
\end{align*}
The functions in $\mathcal{F}$ will be referred to as \textit{features}.
\end{definition}

If we compare this definition of invariance to the one in the previous section, we see that it follows naturally. Let's now consider some examples.


\begin{example}[Symmetric Expectations]
Let $\mathcal{X}=\mathbb{Z}$ be the integers. Define
\begin{align*}
f_1(x) & = |x| \cdot I_{\{x < 0\}} \\
f_2(x) & =  |x| \cdot I_{\{x \geq 0\}},
\end{align*}
and let $\mathcal{F} = \{f_1,f_2\}$ and let $f_1 \sim f_2$. If a pmf $P$ is moment invariant to $\sim$, then we have that
\begin{align*}
E_P(|\bold{X}| \cdot I_{\{\bold{X} < 0\}}) = E_P(|\bold{X}| \cdot I_{\{\bold{X} \geq 0\}}).
\end{align*}
Notice that this is a less stringent invariance compared to the previous example where $P$ was required to be symmetric around the axis; here we only require $P$ to be symmetric around the axis in the sense that the expected value over the negative numbers has the same magnitude as the expected value over the positive numbers. 
\end{example}

\begin{example}[Marginal Expectations]
Suppose $\mathcal{X} =\mathbb{Z}^n$, and let $\mathcal{F} = \{f_1,\ldots,f_n\}$ be a set of real function over $\mathcal{X}$ such that each $f_i(x) = x_i$ is the projection of $x \in \mathcal{X}$ onto its $i$th factor. Further, suppose we define the trivial equivalence class on $\mathcal{F}$ in which all functions are in the same equivalence class (i.e. $f_i \sim f_j$ for all $i,j$). Then, if $P$ is moment invariant to this equivalence class, we have
\begin{align*}
E_P(\bold{X}_1) = \ldots = E_P(\bold{X}_n).
\end{align*}
That is, the expected values of all marginal random variables are the same. 
\end{example}

\begin{example}[Marginal Distributions]
Suppose we want a distribution $P$ over the space $\mathcal{X} =\mathbb{Z}^n$ to have the invariance property that all marginal distributions are equal to each other, i.e., that
\begin{align*}
P(\bold{X}_1) = \ldots = P(\bold{X}_n).
\end{align*}
To specify this invariance, let 
\begin{align*}
\mathcal{F} = \{f_{i,j} \; | \; i = 1,\ldots,n \text{ and } j \in \mathbb{Z}\}
\end{align*}
be the set of real functions on $\mathcal{X}$ in which each $f_{i,j}(x) = I_{\{x_i = j\}}$ is an indicator function. Further, define the equivalence class on $\mathcal{F}$ such that $f_{i,j} \sim f_{k,j}$ for all $i,k \in \{1,\ldots,n\}$. Then, if $P$ is moment invariant to this equivalence class, we have that all marginal random variables have the same distribution. Naturally, this invariance is useful in estimations involving the distribution $P$, which we consider in the next example. 
\end{example}

\begin{example}[Estimations using Moment Invariances]
Suppose we have a finite set of real functions $\mathcal{F}$ and it has an equivalence relation $\sim$ on it. For a function $f_0 \in \mathcal{F}$, denote its equivalence class by
\begin{align*}
	 \mathcal{F}_0 = \{f \in \mathcal{F} \; | \; f \sim f_0 \}.
\end{align*}
If a pmf $P$ is moment invariant to $\sim$, then:
\begin{align*}
E_P(f_0(\bold{X})) &= \frac{1}{|\mathcal{F}_0|} \sum \limits_{f \in \mathcal{F}_0 } E_P(f(\bold{X})) \\
		& = \frac{1}{|\mathcal{F}_0|} E_P\left[ \sum \limits_{f \in \mathcal{F}_0 } f(\bold{X}) \right].
\end{align*}
Thus, given data samples $x^{(1)},\ldots,x^{(m)}$, we may utilize this invariance and estimate $E_P(f_0(\bold{X}))$ as follows:
\begin{align*}
\hat{E}_P(f_0(\bold{X})) =  \frac{1}{|\mathcal{F}_0|} \frac{1}{m} \sum \limits_{i=1}^m \sum \limits_{f \in \mathcal{F}_0 } f(x^{(i)}).
\end{align*}
Utilizing these invariances will be beneficial in many applications. 
\end{example}

\section{Conditional Moment Invariance}
\label{sec:e}

In Section \ref{sec:d}, we defined independence and conditional independence, two special cases of invariances on distributions. Similarly, in this section, we define moment independence and conditional moment independence. We begin by defining conditional and marginal moment invariances.

Let $(\Omega,\Sigma_{\Omega},\mathbb{P})$ be a probability space and let $\mathcal{X}$ be a countably infinite space. Further, let $\bold{X}:\Omega \rightarrow \mathcal{X}$ be a measurable function (i.e. a discrete random variable) with a pmf $P: \mathcal{X} \rightarrow [0,1]$, where
\begin{align*}
P(x) = \mathbb{P}(\{w \in \Omega \; : \; \bold{X}(w)=x\}), \; x \in \mathcal{X}.
\end{align*}
Thus, we may form the probability space $(\mathcal{X},\Sigma_{\mathcal{X}},P)$, where $\Sigma_{\mathcal{X}}$ is the power set of $\mathcal{X}$. 
Furthermore, for any $A \in \Sigma_{\Omega}$ such that $\mathbb{P}(A)>0$, we may form the conditional probability:
\begin{align*}
P_{\mathcal{X}}(x|A) & = \mathbb{P}(\{w \in \Omega \; : \; \bold{X}(w)=x\} \; | \; \{w \in A\} ) \\
 & = \frac{\mathbb{P}(\{w \in A \; : \; \bold{X}(w)=x\}) }{ \mathbb{P}(\{w \in A\})  } 
\end{align*}
where $x \in \mathcal{X}$. Suppose we have a set of functions $\mathcal{F}$ composed of functions of the form $f: \mathcal{X} \rightarrow \mathbb{R}$.
Now, we will define an invariance in terms of an equivalence relation $\sim_{\Gamma}$ on the product space $\Gamma = \Sigma_{\Omega} \times \mathcal{F}$. The general definition of an invariance is as follows. 

\begin{definition}[Conditionally Moment Invariance]
A distribution $P$ is \textit{conditionally moment invariant} to $\sim_{\Gamma}$ if for all $(A_1,f_1),(A_2,f_2) \in  \Sigma_{\Omega} \times \mathcal{F}$ we have:
\begin{align*}
(A_1,f_1) \sim_{\Gamma} (A_2,f_2) \implies E_P[f_1(\bold{X})\; | \; A_1] = E_P[f_2(\bold{X}) \; | \; A_2],
\end{align*}
where
\begin{align*}
E_P[f(\bold{X})\; | \; A ] \equiv \sum_{x \in \mathcal{X}} f(x) \; P(x | A).
\end{align*}
\end{definition}

\begin{example}[Moment Independence]
Suppose we have an equivalence relation $\sim_{\Gamma}$ on the product space $\Gamma = \Sigma_{\Omega} \times \mathcal{F}$. Let $A_1,A_2 \in \Sigma_{\Omega}$ and $f_1,f_2 \in \mathcal{F}$, and suppose that
\begin{align*}
A_2 = \Omega \text{ and } f_1 = f_2.
\end{align*}
Then, the feature $f_1$ is said to be \textit{moment independent} of the event $A_1$ if $(A_1,f_1) \sim_{\Gamma} (A_2,f_2)$. That is, if $f_1$ is moment independent of $A_1$, then:
\begin{align*}
E_P(f_1(\bold{X}) \; | \; A_1) = E_P(f_1(\bold{X})).
\end{align*}
\end{example}

\begin{example}[Conditional Moment Independence]
Suppose we have an equivalence relation $\sim_{\Gamma}$ on the product space $\Gamma = \Sigma_{\Omega} \times \mathcal{F}$. Let $A_1,A_2 \in \Sigma_{\Omega}$ and $f_1,f_2 \in \mathcal{F}$, and suppose that
\begin{align*}
A_2 \subset A_1 \text{ and } f_1 = f_2.
\end{align*}
Then, the feature $f_1$ is said to be \textit{conditionally moment independent} of the event $A_1 \setminus A_2$, given the event $A_2$ if $(A_1,f_1) \sim_{\Gamma} (A_2,f_2)$. That is, if $f_1$ is moment independent of $A_1 \setminus A_2$ given $A_2$, then:
\begin{align*}
E_P(f_1(\bold{X}) \; | \; A_1) = E_P(f_1(\bold{X}) \; | \; A_2).
\end{align*}
\end{example}

\backmatter  

\bibliographystyle{plainnat}
\bibliography{sample}

\end{document}